\documentclass[twoside]{article}

\usepackage[accepted]{aistats2021}

\usepackage{tikz} 
\usepackage{amsmath}
\usepackage{amsfonts}
\usepackage{amsthm}
\usepackage{hyperref}
\usepackage{algpseudocode}
\usepackage{algorithm}
\usepackage{mathtools, bbm}
\usepackage{color}
\usepackage{caption}
\usepackage{subcaption}
\usepackage{comment}
\usepackage{natbib}

\newcommand{\rev}[1]{#1}
\newcommand{\del}[1]{}
\DeclareMathOperator*{\argmin}{arg\,min}
\DeclareMathOperator*{\diag}{diag}

\newtheorem{lemma}{Lemma}
\newtheorem{theorem}{Theorem}
\newtheorem{proposition}{Proposition}

\newtheorem{definition}{Definition}

\def\calE{\mathcal{E}}

\def\calS{\mathcal{S}}

\def\calR{\mathcal{R}}

\def\A{\mathcal{A}}
\def\X{\mathcal{X}}

\def\Z{\mathcal{Z}}
\def\Y{\mathcal{Y}}
\def\B{\mathcal{B}}
\def\C{\mathcal{C}}
\def\Y{\mathcal{Y}}
\def\L{\mathcal{L}}
\def\J{\mathcal{J}}

\def\E{\mathbb{E}}

\def\Z{\mathcal{Z}}
\def\1{\mathbf{1}}
\def\P{\mathbb{P}}
\def\R{\mathbb{R}}
\def\N{\mathbb{N}}

\def\one{\mathbbm{1}}
\def\thetast{\theta_*}
\def\gamworst{\bar{\gamma}}

\def\Delmax{\Delta_{\max}}
\def\gamworstae{\bar{\gamma}_{\mathrm{ae}}}

\def\xst{x_*}

\def\cN{\mathcal{N}}

\def\mwtest{\widehat{OPT}}
\def\tol{\textsc{tol}}
\def\opt{\textsc{opt}}
\def\optt{\wt{\textsc{opt}}}

\def\oca{\A}
\def\ocb{\B}
\def\occ{\C}

\def\sg{\mathcal{SG}}

\newcommand{\mc}[1]{\mathcal{#1}}

\def\t{\top}
\DeclareMathOperator*{\trianglem}{\tilde{\triangle}}
\DeclareMathOperator*{\tr}{\text{Trace}}
\DeclareMathOperator*{\ucb}{\text{CB}}
\DeclareMathOperator*{\simp}{\triangle}
\DeclareMathOperator*{\kl}{\mathbf{KL}}

\newcommand{\wt}[1]{\widetilde{#1}}

\newcommand{\jks}[1]{}
\newcommand{\aw}[1]{}

\newcommand{\xmax}{x_{\mathrm{max}}}
\newcommand{\ellmax}{\ell_{\mathrm{max}}}
\newcommand{\norm}[1]{\left\lVert#1\right\rVert}
 \newcommand{\Asb}{A_{\mathrm{semi}}}
\newcommand{\Ab}{A_{\mathrm{band}}}
\newcommand{\cV}{\mathcal{V}}
\newcommand{\poly}{\mathrm{poly}}
\newcommand{\tilO}{\wt{\mathcal{O}}}
\newcommand{\calO}{\mathcal{O}}
\DeclarePairedDelimiter\ceil{\lceil}{\rceil}
\newcommand{\Exp}{\mathbb{E}}
\newcommand{\Roptim}{\calR^{\mathrm{optimism}}}
\newcommand{\Delmin}{\Delta_{\mathrm{min}}}

\newcommand{\Af}{A_{\mathfrak{f}}}
\newcommand{\nf}{n_{\mathfrak{f}}}
\newcommand{\frakf}{\mathfrak{f}}
\newcommand{\band}{\mathrm{band}}
\newcommand{\sband}{\mathrm{semi}}

\newcommand{\Rsfw}{R_{\mathrm{SFW}}}

\DeclareMathOperator*{\argmax}{arg\,max}

\begin{document}

\twocolumn[

\aistatstitle{Experimental Design for Regret Minimization in Linear Bandits}

\aistatsauthor{ Andrew Wagenmaker$^*$ \And Julian Katz-Samuels$^*$ \And  Kevin Jamieson }

\aistatsaddress{ University of Washington \\ \texttt{ajwagen@cs.washington.edu} \And University of Washington \\ \texttt{jkatzsam@cs.washington.edu} \And University of Washington \\ \texttt{jamieson@cs.washington.edu} } ]

\begin{abstract}
In this paper we propose a novel experimental design-based algorithm to minimize regret in online stochastic linear and combinatorial bandits. While existing literature tends to focus on optimism-based algorithms--which have been shown to be suboptimal in many cases--our approach carefully plans which action to take by balancing the tradeoff between information gain and reward, overcoming the failures of optimism. In addition, we leverage tools from the theory of suprema of empirical processes to obtain regret guarantees that scale with the Gaussian width of the action set, avoiding wasteful union bounds. We provide state-of-the-art finite time regret guarantees and show that our algorithm can be applied in both the bandit and semi-bandit feedback regime. In the combinatorial semi-bandit setting, we show that our algorithm is computationally efficient and relies only on calls to a linear maximization oracle. In addition, we show that with slight modification our algorithm can be used for pure exploration, obtaining state-of-the-art pure exploration guarantees in the semi-bandit setting. Finally, we provide, to the best of our knowledge, the first example where optimism fails in the semi-bandit regime, and show that in this setting our algorithm succeeds. 
\end{abstract}

\section{INTRODUCTION}
Multi-armed bandits have received much attention in recent years as they serve as an excellent model for developing algorithms that adeptly deal with the exploration-exploitation tradeoff. In this paper, we consider the stochastic linear bandit problem in which there is a set of arms $\X \subset \R^d$ and an unknown parameter $\thetast \in \R^d$. An agent plays a sequential game where at each round $t$ she chooses an arm $x_t \in \X$ and receives a noisy reward whose mean is $x^\t_t \thetast$. The goal is to maximize the reward over a given time horizon $T$. An important special case of stochastic linear bandits is the combinatorial setting where $\X \subset \{0,1\}^d$, which can be used to model problems such as finding a shortest path in a graph or the best weighted matching in a bipartite graph. We consider both the \emph{bandit feedback} setting, where the agent receives a noisy observation of $x_t^\top \thetast$, and the \emph{semi-bandit feedback} setting,  where the agent receives a noisy observation of $\theta_{*,i}$ for each $i$ with $x_{t,i} = 1$.

Existing regret minimization algorithms for linear bandits suffer from several important shortcomings. First, they typically rely on naive union bounds, which yield regret guarantees scaling as either $\calO(d\sqrt{T})$ or $\calO(\sqrt{d \log(|\X|) T})$. Such union bounds ignore the geometry present in the problem and, as such, can be very wasteful. As the union bound often appears in the confidence interval within the algorithm, this is not simply an analysis issue---it can also affect real performance. Second, in the moderate, non-asymptotic time regime, existing algorithms tend to rely on the principle of \emph{optimism}---pulling only the arms they believe may be optimal. Algorithms relying on this principle are very myopic, foregoing initial exploration which could lead to better long-term reward and instead focusing on obtaining short-term reward, leading to suboptimal long-term performance. This is a well-known effect in the bandit setting but, as we show, is also present in the semi-bandit setting.

In this paper, we develop an algorithm overcoming both of these shortcomings. Rather than employing a naive union bound, we appeal to tools from empirical process theory for controlling the suprema of a Gaussian process, allowing us to obtain confidence bounds that are geometry-dependent and potentially much tighter. In addition, our algorithm relies on careful planning to balance the exploration-exploitation tradeoff, taking into account both the potential information gain as well as the reward obtained when pulling an arm. This planning allows us to collect sufficient information for good long-term performance without incurring too much initial regret and, to the best of our knowledge, is the first planning-based algorithm in the linear bandit setting that provides finite-time guarantees.

We emphasize that we are interested in the non-asymptotic regime and aim to optimize the whole regret bound, including lower-order terms. While several recent works achieve instance-optimal regret, they suffer from loose lower-order terms which dominate the regret for small to moderate $T$. Our results aim to minimize such terms through employing tighter union bounds. We summarize our contributions:
\begin{itemize}
\item We develop a single, general algorithm that achieves a state-of-the-art finite-time regret bound in stochastic linear bandits, in combinatorial bandits with bandit feedback, and in combinatorial bandits with semi-bandit feedback. In addition, our framework is general enough to extend to settings as diverse as partial monitoring and graph bandits.
\item We show that in the combinatorial semi-bandit regime, our algorithm is computationally efficient, relying only on calls to a linear maximization oracle, and state-of-the-art, yielding a significant improvement on existing works in the non-asymptotic time horizon regime.
\item We give the first example for combinatorial bandits with semi-bandit feedback that shows that optimistic strategies such as UCB and Thompson Sampling can do arbitrarily worse than the asymptotic lower bound, and show that our algorithm improves on optimism in this setting by an arbitrarily large factor.
\item As a corollary, we obtain the first computationally efficient algorithm for pure exploration in combinatorial bandits with semi-bandit feedback, and achieve a state-of-the-art sample complexity.
\end{itemize}
This work can be seen as obtaining \emph{problem-dependent minimax} bounds---minimax bounds that depend on the arm set but hold for all values of the reward vector---and are similar in spirit to the bounds on regret minimization in MDPs given by \cite{zanette2019tighter}. For some favorable arm sets $\X$, our bounds are tighter than prior $\X$-independent minimax bounds by large dimension factors. To the best of our knowledge, we are the first to obtain such geometry-dependent minimax bounds for linear bandits.

\newcommand{\tsum}{{\textstyle \sum}}
\section{PRELIMINARIES}
Let $\mathrm{diam}(\X) = \max_{x,y \in \X} \| x - y \|_2$ denote the diameter of $\X \subseteq \R^d$. $\diag(X)$ will refer to the  operator which sets all elements in a matrix $X$ not on the diagonal to 0. $\tilO(.)$ hides logarithmic terms. $\triangle_\X := \{a \in \R^{|\X|} : \norm{a}_1 = 1, \, a_i \geq 0 \, \forall i \}$ denotes the simplex over $\X$.  We use $\lambda \in \triangle_{\X}$ to refer to probability distributions over $\X$ and $\lambda_x$ to denote the probability on $x \in \X$. We let $\tau \in [0,\infty)^{|\X|}$ refer to allocations over $\X$ and, similarly, $\tau_x$ to denote the weight on $x \in \X$. We will somewhat interchangeably use $\tau$ to refer to the vector in $\R^{|\X|}$ and the sum of its elements, $\sum_{x \in \X} \tau_x$, but it will always be clear from context which we are referring to. If $x \in \{0,1\}^d$, we will often write $i \in x$ for $x_i = 1$ and $i \not \in x$ for $x_i = 0$. Throughout, we will let $d$ denote the dimension of the ambient space and  $k = \max_{x \in \X} \| x \|_1$. 

We are interested primarily in regret minimization in linear bandits. Given some set $\X \subseteq \R^d$, at every timestep we choose $x_t \in \X$ and receive reward $x_t^\top \thetast$, for some unknown $\thetast \in \R^d$. We will define regret as:
$$ \calR_T =  T \max_{x \in \X} x^\top \thetast - \sum_{t=1}^T x_t^\top \thetast $$
Throughout, we assume that $\thetast \in [-1,1]^d$. We consider two observation models: semi-bandit feedback and bandit feedback. In the bandit feedback setting, at every timestep we observe:
$$ y_t = x_t^\top \thetast + \eta_t $$
where $\eta_t \sim \cN(0,1)$. In the semi-bandit feedback setting, we assume that our bandit instance is combinatorial, $\X \subseteq \{ 0, 1 \}^d$, and at every timestep we observe:
$$ y_{t,i} = \theta_{*,i} + \eta_{t,i}, \quad \forall i \in x_t $$
where $\eta_t \sim \cN(0,I)$. Note that, while we assume Gaussian noise for simplicity, all our results will hold with sub-Gaussian noise \citep{katz2020empirical}. 
 
In the bandit setting, after $T$ observations, our estimate of $\thetast$ will be the standard least squares estimate:
$$ \hat{\theta} = \bigg ( \sum_{t=1}^T x_t x_t^\top \bigg )^{-1} \sum_{t=1}^T x_t y_t $$
In the semi-bandit setting, we will estimate $\thetast$ coordinate-wise, forming the estimate:
$$ \hat{\theta}_i = \frac{1}{T_i} \sum_{t=1, x_{t,i} =1}^T  y_{t,i}  $$
where $T_i$ is the number of times $x_{t,i} = 1$.  We denote:
$$\Ab(\lambda) = \sum_{x\in\mathcal{X}} \lambda_x x x^\top,   \Asb(\lambda) = \diag \left (\sum_{x \in \X} \lambda_x x x^\top \right )$$
For convenience we assume the optimal arm is unique and denote it by $\xst$. As is standard, we denote the gap of arm $x$ by $\Delta_x :=  \theta_*^\t(\xst - x)$. We denote the minimum gap as $\Delmin = \min_{x \in \X \ : \ \Delta_x > 0} \Delta_x$ and the maximum gap by $\Delmax = \max_{x \in \X} \Delta_x $. 

In the combinatorial setting, $|\X|$ can often be exponentially large in the dimension, making computational efficiency non-trivial since $\X$ cannot be efficiently enumerated. As such, much of the literature on combinatorial bandits has focused on obtaining algorithms that rely only on an argmax oracle:
$$ \text{ORACLE}(v) = \argmax_{x \in \X} x^\top v $$
Efficient argmax oracles are available in many settings, for instance finding the minimum weighted matching in a bipartite graph and finding the shortest path in a directed acyclic graph.

\section{MOTIVATING EXAMPLES}

Before presenting our algorithm and main results, we present several examples that motivate the necessity of planning and the wastefulness of naive union bounds, and illustrate how our algorithm is able to make improvements in both these aspects.

First, we show that an optimistic strategy cannot be optimal for combinatorial bandits with semi-bandit feedback. Consider a generic optimistic algorithm that maintains an estimate $\widehat{\theta}_t$ of $\theta$ at round $t$ and selects the maximizer of an upper confidence bound, $x_t = \argmax_{x \in \X} x^\t \widehat{\theta}_t + \ucb(x,\{x_s\}_{s=1}^{t-1} )$. We make two assumptions on the confidence bound $\ucb(\cdot, \cdot)$. First, we assume that $\P[ \exists t \leq T, \exists x \in \X: |x^\t (\widehat{\theta}_t -\theta)| > \ucb(x,\{x_s\}_{s=1}^{t-1}) ] \leq 1/T$. Second, we assume that the confidence bound is at least as good as a confidence bound formed from taking the least squares estimate
\begin{align*}
\ucb(x,\{x_s\}_{s=1}^{t-1}) \leq  \sqrt{\alpha \norm{x}_{(\sum_{s=1}^{t-1} x_s x_s^\t )^{-1} }^2 \log(T)}
\end{align*}
where $\alpha > 0$ is a universal constant. We call this algorithm the \emph{generic optimistic algorithm} and let $\Roptim_T$ denote its regret. Then we have the following.

\begin{proposition}
\label{prop:opt_counterexample_semibandit}
Fix any $m \in \N $ and $\epsilon \in (0,1)$. Then there exists a $\calO(m)$-dimensional combinatorial bandit problem with semi-bandit feedback where:
\begin{align*}
\limsup_{T \longrightarrow \infty} \frac{\Exp[\Roptim_T]}{\log(T)} & = \Omega \left ( \frac{m}{\epsilon} \right ).
\end{align*}
and Algorithm \ref{alg:gw_ae_comp} has expected regret bounded as, for any $T$:
$$ \Exp[\calR_T] \leq \calO \left ( \min \left \{ \frac{\sqrt{m} \log (T)}{\epsilon^2}, \frac{m\log(T)}{\epsilon} \right \} \right). $$
\end{proposition}
Thus, treating $\epsilon$ as a constant, the asymptotic regret of the generic optimistic algorithm is loose by a square root dimension factor, and Algorithm \ref{alg:gw_ae_comp} in the current paper improves over optimism by an arbitrarily large factor. As it also relies on the principle of optimism, albeit in a randomized fashion, Thompson Sampling will be suboptimal by this same factor on this instance. A similar instance can also be found in the bandit feedback setting. The improvement in Algorithm \ref{alg:gw_ae_comp} is due to its ability to pull informative but suboptimal arms if the information gain outweighs the regret incurred, reducing the cumulative regret. Optimistic algorithms, in contrast, will only pull arms they believe may be optimal, and so do not effectively take into account the information gain which, in some cases, causes them to be very suboptimal. 

To illustrate the improvement we gain by applying a less naive union bound, we will consider the following combinatorial class:
\begin{align*}
\X = \left \{x \in \{0,1\}^{m+n} : \sum_{i=1}^{m} x_i = k, \sum_{i=m+1}^{m+n} x_i = \ell \right \}
\end{align*}
where $d = n+m$. This class corresponds to the Cartesian product of a Top-$k$ problem on dimension $m$ and a Top-$\ell$ problem on dimension $n$. As we will show, the minimax regret of Algorithm \ref{alg:gw_ae_comp} scales with $\gamworst(A)$, a measure of the Gaussian width of $\X$, as defined below in (\ref{eq:gamworst}). In contrast, algorithms that apply naive union bounds have regret that scales either with $(m+n) \log | \X |$ or $(m+n)^2$. The following proposition illustrates the improvement in scaling we are able to obtain, as well as the subtle dependence of minimax regret on the geometry of $\X$.
\begin{proposition}\label{prop:gw_topk_product}
For $\frakf \in \{ \band, \sband \}$, on the product of Top-$k$ instances described above, we have:
$$ \gamworst(\Af) \leq \calO(k m + \ell n ), \quad \log |\X| \geq \Omega(k + \ell) $$
This implies there exist settings of $m,n,k$, and $\ell$ such that the regret of Algorithm \ref{alg:gw_ae_comp} with either bandit feedback or semi-bandit feedback will be bounded:
$$ \Exp[\calR_T] \leq \tilO \left (d^{1/2} \sqrt{ T} \right ) $$
while algorithms employing naive union bounds will achieve regret bounds scaling at best as:
$$ \Exp[\calR_T] \leq \tilO \left (d^{2/3} \sqrt {T} \right ). $$
\end{proposition}
In the appendix we discuss in more detail how the regret scales for specific algorithms in this setting. The regret bound we present for our algorithm in Proposition \ref{prop:gw_topk_product} is in fact state-of-the-art---all other existing algorithms will incur the larger dimension dependence.

\section{EXPERIMENTAL DESIGN FOR REGRET MINIMIZATION}
 \subsection{Gaussian Width}
Before introducing our algorithm, we present a final concept critical to our results. For a fixed $\theta_*$, let $ \X_\epsilon = \{ x \in \X \ : \ \Delta_x \leq \epsilon \}$, then, for $\frakf \in \{ \band, \sband \}$:
\begin{equation}\label{eq:gamworst}
\gamworst(\Af) = \sup_{\epsilon > 0} \inf_{\lambda \in \triangle_{\X_\epsilon}} \Exp_{\eta} \left [ \sup_{x \in \X_\epsilon} x^\top \Af(\lambda)^{-1/2} \eta \right ]^2 
\end{equation}
Intuitively, $\gamworst(\Af)$ is the largest Gaussian width of any subset of $\X$ formed by taking all $x \in \X$ with gap bounded by $\epsilon$. The following results are helpful in giving some sense of the scaling of $\gamworst(\Af)$.

\begin{proposition}\label{prop:gw_bound_union}
For any $\X \subseteq \R^d$ and $\frakf \in \{ \band, \sband \}$, we have:
$$ \gamworst(\Af) \leq c\min \{ d \log |\X |, d^2 \}.$$
\end{proposition}

\begin{proposition}\label{prop:gw_bound_comb}
If $\X \subseteq \{ 0,1 \}^d$, $k = \max_{x \in \X} \| x \|_1$, and $\frakf \in \{ \band, \sband \}$, then, for $d \geq 3$:
$$ \gamworst(\Af) \leq c d k \log d. $$
\end{proposition}

Note that these upper bounds are often loose. The following results shows that, in some cases, we pay a $d$ instead of $dk$.

\begin{proposition}\label{prop:gw_topk_plus1}
There exists a combinatorial bandit instance in $\R^d$ with $k = \sqrt{d}$ where:
$$ \gamworst(\Asb) \leq c d \log(d). $$
\end{proposition}

The Gaussian width is critical in avoiding wasteful union bounds, allowing instead for geometry-dependent confidence intervals. The following confidence interval will form a key piece in our analysis.

\begin{proposition}[Tsirelson-Ibragimov-Sudakov Inequality \citep{katz2020empirical,cirel1976norms}]\label{prop:tis}
Consider playing arm $x$ $\tau_x$ times, where $\tau$ is an allocation chosen deterministically. Assume $\frakf \in \{ \mathrm{band}, \sband \}$ is set to correspond to the type of feedback received and let $\hat{\theta}$ be the least squares estimate of $\thetast$ from these observations. Then,  simultaneously for all $x \in \mathcal{X}$, with probability at least $1 - \delta$:
\begin{align*}
& |x^\top (\hat{\theta} - \theta_*)| \leq \mathbb{E}_{\eta \sim \mathcal{N}(0,I)} \left [ \sup_{x \in \mathcal{X}} x^\top \Af(\tau)^{-1/2} \eta \right ] \\
& \qquad \qquad \qquad \qquad + \sqrt{ 2\sup_{x \in \mathcal{X}} \| x \|_{\Af(\tau)^{-1}}^2 \log(2/\delta)} .
\end{align*}
\end{proposition}

\subsection{Algorithm Overview}

\begin{algorithm}[t]
\begin{algorithmic}[1]
\State \textbf{Input:} Set of arms $\mathcal{X}$,  largest gap $\Delmax$,  confidence $\delta$, total time $T$, feedback type $\mathfrak{f} \in \{\mathrm{band},\sband\}$
\State $\hat{\theta}_0 \leftarrow 0, x_1 \leftarrow 0, \hat{\Delta}_x \leftarrow 0, \ell \leftarrow 1$
\While{total pulls less than $T$}
 	 \State $\epsilon_\ell \leftarrow \Delmax 2^{-\ell}$
	  \State Let $\tau_\ell$ be a solution to:
	  \begin{align}\label{eq:vl_inefficient}
	  \begin{split}
	 & \argmin_{\tau}  \ \sum_{x \in \X} 2(\epsilon_\ell + \hat{\Delta}_x) \tau_x  \\
	& \text{ s.t. } \mathbb{E}_\eta \left [ \max_{x \in \X} \frac{(x_\ell - x)^\top \Af(\tau)^{-1/2} \eta}{\epsilon_\ell + \hat{\Delta}_x} \right ] \\
	& \qquad + \sqrt{2 \sup_{x \in \X} \frac{ \| x \|_{\Af(\tau)^{-1}}^2}{(\epsilon_\ell + \hat{\Delta}_x)^2} \log(2\ell^3/\delta)} \leq \frac{1}{128} 
	\end{split}
	\end{align}  
	\If{$ \sum_{x \in \X} (\epsilon_\ell + \hat{\Delta}_x) \tau_{\ell,x} > T \epsilon_\ell $}
		\State \textbf{break} \label{line:break_minimax2}
	\EndIf
	\State $\alpha_\ell \leftarrow$ SPARSE$(\tau_\ell,\nf)$
	 \State Pull arm $x$ $\lceil \alpha_{\ell,x} \rceil$ times, compute $\hat{\theta}_\ell$
	 \State $x_{\ell+1} \leftarrow \argmax_{x \in \X} x^\top \hat{\theta}_{\ell}$, $\hat{\Delta}_x \leftarrow \hat{\theta}_{\ell}^\top (x_{\ell +1} - x)$
	 \If{MINGAP$(\hat{\theta}_\ell,\X) > 2 \epsilon_\ell $}
	 	\State \textbf{break}
	\EndIf
	 \State $\ell \leftarrow \ell + 1$
\EndWhile
\State Pull $\hat{x} = \argmax_{x \in \X} x^\top \hat{\theta}_{\ell-1}$ for all remaining time
\end{algorithmic}
\caption{\textbf{Regret} \textbf{M}inimizing \textbf{E}xperimental \textbf{D}esign: \textbf{RegretMED}}
\label{alg:gw_ae_comp}
\end{algorithm}

We next present our algorithm, RegretMED, in Algorithm \ref{alg:gw_ae_comp}. Inspired by several recent algorithms achieving asymptotically optimal regret \citep{lattimore2017end}, at every epoch our algorithm finds a new allocation by solving an experimental design problem (\ref{eq:vl_inefficient}). This minimizes an upper bound on the regret incurred in the epoch while ensuring the allocation produced will explore enough to improve the estimates of the gaps for each arm, thereby balancing exploration and exploitation and allowing us to obtain a tight bound on finite-time regret. We apply the TIS inequality to bound the estimation error of our gaps, which motivates the constraint in (\ref{eq:vl_inefficient}). Critically, this yields a regret bound scaling with the Gaussian width of the action set.

We define SPARSE$(\tau,n) : \R^{|\X|} \rightarrow \R^{|\X|}$ to be a function taking as input an allocation and returning a new allocation that is $n$ sparse and approximating the solution to (\ref{eq:vl_inefficient}). So long as $n \geq d+1$ in the semi-bandit setting and $n \geq d^2 + d +1$ in the bandit setting, it is possible to find a distribution $\alpha$ that is $n$ sparse and will achieve the same value of the constraint and objective of (\ref{eq:vl_inefficient}), see Lemma \ref{lem:rounding}. MINGAP$(\theta,\X)$ takes as input an estimate of $\theta$ and returns the gap between the best and second best arms in $\X$ with respect to this $\theta$. It is possible to compute this quantity efficiently with only calls to a linear maximization oracle (see Appendix \ref{sec:pure_exploration_proof}).

While Algorithm \ref{alg:gw_ae_comp} takes as input $\Delmax$, we require this only to simplify the analysis. In practice, we can use an upper bound instead without changing the final regret of our algorithm by more than a logarithmic factor. Since $\Delmax \leq \sqrt{d} \mathrm{diam}(\X)$, an upper bound can be obtained without knowledge of $\thetast$.

\textbf{Key Theoretical Tools:} We briefly describe the key theoretical tools employed by RegretMED. First, we note that an experimental design based algorithm is novel in the setting of regret minimization. As we have shown, this approach allows us to perform properly on challenging instances by explicitly balancing the information gain and reward, while also yielding a computationally feasible solution in the semi-bandit regime. Our second innovation is the use of the TIS inequality to obtain tight concentration bounds. While we are not the first to utilize this in the linear bandit setting \citep{katz2020empirical}, it previously was only utilized in the best arm identification setting, and our work therefore shows how it can be applied in the regret minimization setting as well. The use of the TIS inequality yields two important improvements over more naive union bounds. First, it provides tighter confidence intervals in the non-asymptotic time regime and therefore yields improved regret bounds. Second, as we will see, it allows us to write the constraint for our experiment design problem \eqref{eq:vl_inefficient} in a form that is linear in the the decision variable. This allows us to reduce solving the optimization to calls of a linear maximization oracle, and is a key piece in showing our algorithm is computationally efficient.

\subsection{Main Regret Bound}
We now state our main regret bound. Define
\begin{align*} 
\ellmax(T) &:=  \log_2 \left (   \frac{\max_{x \in \X} \| x \|_2}{\min_{x \in \X} \| x \|_2} \left (  \Delmax \sqrt{T} + 3 \right ) \right )  \\
& = \calO(\log(T))
\end{align*}
and
$\ellmax(\thetast) := \lceil \log(4 \Delmax / \Delmin) \rceil$. 
Let $n_{\mathrm{band}} = d^2+d+1, n_{\sband} = d+1$.

\begin{theorem}\label{thm:inefficient_regret_bound} With $\frakf \in \{ \mathrm{band}, \sband \}$ set to correspond to the type of feedback received, Algorithm \ref{alg:gw_ae_comp} will have gap-dependent regret bounded, with probability $1-\delta$, as:
\begin{align*}
& c_1  \Delmax  \ellmax(\thetast)^2  (d +  \nf) \\
&  \quad + \frac{c_2 \Big (\gamworst(\Af)  \ellmax(\thetast)^2 + d \log( \ellmax(\thetast)/\delta) \Big ) }{\Delmin} 
\end{align*}
and minimax regret bounded as:
\begin{align*}
& c_1  \Delmax \ellmax(T)^2 (d + \nf)  \\
& \  +  \ellmax(T) \sqrt{c_2 (\gamworst(\Af) \ellmax(T)^2 + d \log(\ellmax(T)/\delta)) T}
\end{align*}
for absolute constants $c_1$ and $c_2$.
\end{theorem}

The proof of this result is deferred to Appendix \ref{sec:regret_proof}. See Section \ref{sec:related} and Table \ref{tab:regret_summary} for a summary of how this bound scales in particular settings of interest. As a brief comparison, in the semi-bandit feedback setting, considering expected regret, we obtain a leading term of order $\calO \left ( \frac{d \log(T)}{\Delmin} \right )$, which matches the lower bound \citep{degenne2016combinatorial}, while the previous state-of-the-art scaled as $\calO \left ( \frac{d \log^2(k) \log(T)}{\Delmin} \right ) $ \citep{perrault2020statistical}. Algorithm \ref{alg:gw_ae_comp} is then the first algorithm to achieve the lower bound for arbitrary combinatorial structures. In the bandit feedback setting our minimax regret scales as $\tilO(\sqrt{(\gamworst(\Ab) + d) T})$ while LinUCB obtains regret scaling as $\tilO(d\sqrt{T})$ \citep{abbasi2011improved}. Proposition \ref{prop:gw_bound_union} shows that we are never worse than the LinUCB regret and, as Proposition \ref{prop:gw_topk_product} shows, we can sometimes be much better. In Appendix \ref{sec:aegw}, we present a modified algorithm which avoids the factors of $\ellmax(T)$ on the leading term of the minimax regret, although it suffers from several other shortcomings.

\subsection{Computationally Efficient Algorithm}

While Algorithm \ref{alg:gw_ae_comp} can be run in settings where $\X$ is enumerable, it becomes computationally infeasible for very large $\X$, as (\ref{eq:vl_inefficient}) cannot be solved via a linear maximization oracle. In place of (\ref{eq:vl_inefficient}), consider instead solving:
 \begin{align}
 & \argmin_{\tau}  \ \sum_{x \in \X} 2(\epsilon_\ell + \hat{\Delta}_x) \tau_x   \label{eq:vl} \\
& \textstyle\text{ s.t. } \mathbb{E}_\eta \left [ \max_{x \in \X} \frac{(x_\ell - x)^\top A(\tau)^{-1/2} \eta}{\epsilon_\ell + \hat{\Delta}_x} \right ] \leq \frac{1/128}{1 + \sqrt{\pi \log(2\ell^3/\delta)}} \nonumber
\end{align}  
As we show in Theorem \ref{thm:comp_complex}, we can solve this problem with a computationally feasible algorithm in the semi-bandit feedback regime. \del{Critically, this can be computed efficiently due to the linear nature of the constraint, which is obtainable only through applying the TIS inequality.} Running this modified version of Algorithm \ref{alg:gw_ae_comp}, we obtain the following regret bound. 

\begin{theorem}\label{thm:efficient_regret_bound}
Assume $\frakf \in \{ \mathrm{band}, \sband \}$ is set to correspond to the type of feedback received. Consider running Algorithm \ref{alg:gw_ae_comp} but now setting $\tau_\ell$ to be an approximate solution to \eqref{eq:vl}. Then with probability at least $1 - \delta$, the gap-dependent regret will be bounded as:
\begin{align*}
& c_1  \Delmax  \ellmax(\thetast)^2  (d + \nf) \\
& \qquad  + \frac{c_2 \gamworst(\Af) \log(\ellmax(\thetast)/\delta) \ellmax(\thetast)^2}{\Delmin} 
\end{align*}
and the minimax regret will be bounded as:
\begin{align*}
& c_1  \Delmax \ellmax(T)^2 (d + \nf) \\
& \qquad + \ellmax(T)^2 \sqrt{c_2  \gamworst(\Af) \log(\ellmax(T)/\delta) T}
\end{align*}
for absolute constants $c_1$, $c_2$.
\end{theorem}

In the semi-bandit setting, we can apply Theorem \ref{thm:comp_complex} to compute an approximate solution to \eqref{eq:vl} in polynomial time, as described below. See Section \ref{sec:related} and Table \ref{tab:regret_summary} for an in-depth discussion of how our result compares to existing works.

Note that the minimax regret guarantees given in Theorems \ref{thm:inefficient_regret_bound} and \ref{thm:efficient_regret_bound} depend on $\thetast$ through $\gamworst(\Af)$ and $\Delmax$. This dependence can be removed by simply taking a supremum of $\gamworst(\Af)$ over $\thetast$ and using the upper bound $\Delmax \leq \sqrt{d} \mathrm{diam}(\X)$. While we state our regret bounds in high probability, expected regret bounds can also be obtained by setting $\delta = 1/T$.

\subsection{Pure Exploration with Semi-Bandit Feedback}
Although our algorithm is designed to minimize regret, a slight modification gives a computationally efficient algorithm for best arm identification in the semi-bandit feedback setting. In particular, instead of (\ref{eq:vl_inefficient}) consider solving:
 \begin{align}
 & \argmin_{\tau}  \ \sum_{x \in \X} \tau_x \label{eq:opt_exploration} \\
& \textstyle\text{ s.t. } \mathbb{E}_\eta \left [ \max_{x \in \X} \frac{(x_\ell - x)^\top \Af(\tau)^{-1/2} \eta}{\epsilon_\ell + \hat{\Delta}_x} \right ] \leq \frac{1/128}{1 + \sqrt{\pi \log(2\ell^3/\delta)}} \nonumber
\end{align}  
Then we have the following.

\begin{theorem}\label{thm:pure_exploration}
Define
\begin{align*}
\textstyle\rho^* & := \inf_{\lambda \in \triangle} \sup_{x \in \X \setminus \{x_* \}} \frac{\norm{x_*-x}^2_{\Asb(\lambda)^{-1}}}{[\thetast^\t( x_*-x)]^2} \\
\textstyle\gamma^* & := \inf_{\lambda \in \triangle}  \E_{\eta} \left [ \sup_{x \in \X \setminus \{x_* \}} \frac{(x_*-x)^\t \Asb(\lambda)^{-1/2} \eta}{ \thetast^\t(x_*-x )} \right ]^2
\end{align*}
Let $\delta \in (0,1)$. Run Algorithm \ref{alg:gw_ae_comp} but replace \eqref{eq:vl_inefficient} with \eqref{eq:opt_exploration} and omit the break on line \ref{line:break_minimax2}. Invoke Theorem \ref{thm:comp_complex} to efficiently find an approximate solution to \eqref{eq:opt_exploration}. Then, with probability $1-\delta$, the algorithm will terminate after collecting at most:
$$ c\Big ( [\gamma^* + \rho^*] \log(\ellmax(\thetast)/\delta) + d \Big )\ellmax(\thetast)  $$
samples and we will have $\hat{x} = x_*$.
\end{theorem}
We state and prove a lower bound for this problem in the appendix, Theorem \ref{thm:lb_semi_bai}, which shows that this sample complexity is near-optimal. To the best of our knowledge, this is the first general, computationally efficient, and near optimal algorithm for pure exploration with semi-bandit feedback.

\subsection{Optimization}\label{sec:main_opt}

In this section, we provide a polynomial-time algorithm for solving \eqref{eq:vl} in the semi-bandit feedback setting. The generic optimization problem can be written as follows for a fixed $T \in \N$, $\bar{x} \in \X$, $\bar{\theta} \in \R^d$, and $\beta > 0$:
\begin{align}
& \min_{\tau \in [T], \lambda \in \triangle_\X}  \tau \sum_{x \in \X} \bar{\theta}^\t(\bar{x}-x)  \lambda_x + \tau \beta \label{eq:opt_problem} \\
& \ \  \text{ s.t. } \mathbb{E}_\eta \left [ \max_{x \in \X} \frac{(\bar{x} - x)^\top \Asb (\lambda)^{-1/2} \eta}{\beta + \bar{\theta}^\t ( \bar{x}-x)} \right ]  \leq \sqrt{\tau} C \nonumber.
\end{align}
The following result shows that there exists a polynomial-time algorithm that finds an approximately optimal solution, i.e., it is within a constant approximation factor of the optimal solution.
\begin{theorem}\label{thm:comp_complex}
Let $\opt$ be the optimal value of \eqref{eq:opt_problem}. There exists an Algorithm that returns $(\bar{\tau},\bar{\lambda})$ such that $\bar{\lambda} \in \triangle_\X$, $\bar{\tau} \leq 2 T$, and, with probability at least $1-\delta - \frac{1}{2^d}$: 
\begin{align*}
& \bar{\tau} \sum_{x \in \X} \bar{\theta}^\t(\bar{x}-x)  \bar{\lambda}_x + \bar{\tau} \beta \leq 4\opt + 2 \\
& \mathbb{E}_\eta \left [ \max_{x \in \X} \frac{(\bar{x}-x)^\top \Asb(\bar{\lambda})^{-1/2} \eta}{\beta + \bar{\theta}^\t ( \bar{x}-x)} \right ]   \leq \sqrt{\bar{\tau}} C \nonumber.
\end{align*}
Furthermore, the number linear maximization oracle calls is polynomial in $(d, \beta, T, \log(1/\delta))$. 
\end{theorem}

We briefly sketch the algorithmic approach. We recast \eqref{eq:opt_problem} as a series of feasibility problems and employ the Plotkin-Shmoys-Tardos reduction of convex feasibility programs to online learning to solve each of these feasibility programs using the multiplicative weights update algorithm. To employ this reduction, we fix $\tau$ and develop a solver for the Lagrangian of  \eqref{eq:opt_problem}, $\L(\kappa ; \lambda)$, which we show to be convex and strongly-smooth in $\lambda$ over a carefully constructed subset of the simplex $\trianglem_\X \subset \triangle_\X$. We solve $\min_{\lambda \in \trianglem_\X} \L(\kappa ; \lambda)$ by employing stochastic Frank-Wolfe, which maintains sparse iterates to overcome the challenge posed by the exponential number of variables in $\L(\kappa ; \lambda)$. Evaluating the gradient requires computing for $\eta \sim N(0,I)$
\begin{align*}
\argmax_{x \in \X}  \frac{(\bar{x} - x)^\top \Asb(\lambda)^{-1/2} \eta}{\beta + \bar{\theta}^\t ( \bar{x}-x)},
\end{align*}
which can be solved using only linear maximization oracle calls via the binary search procedure from \cite{katz2020empirical}. The proof of this result and full algorithm is given in Section \ref{sec:comp}.

\begin{figure*}
     \centering
     \hfill
     \begin{minipage}[b]{0.33\textwidth}
         \centering
          \includegraphics[width=\linewidth]{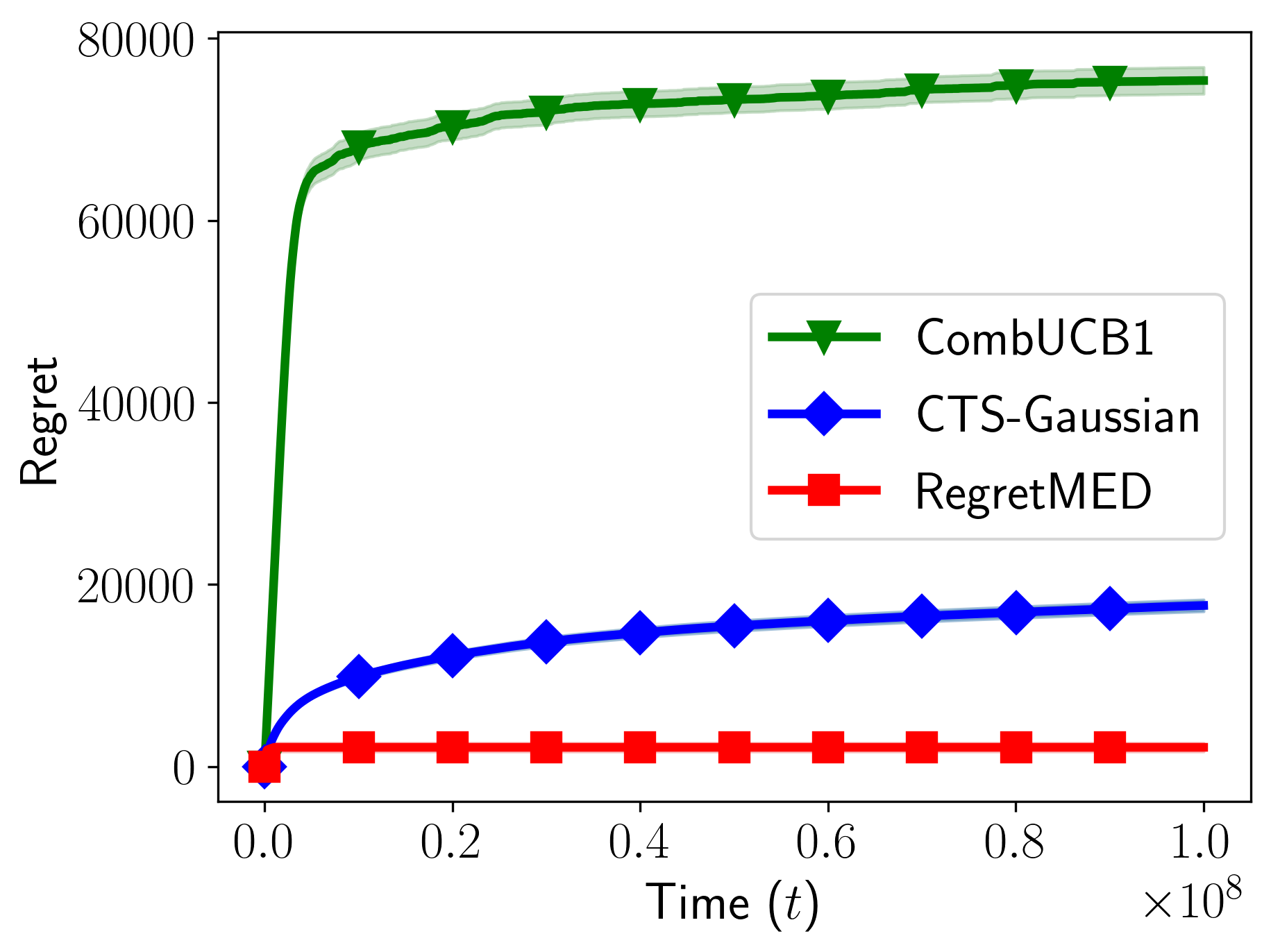}
  \caption{Resource allocation example with $d = 5$.}
  \label{fig:resource5}
     \end{minipage}
     \hfill
     \begin{minipage}[b]{0.335\textwidth}
         \centering
          \includegraphics[width=\linewidth]{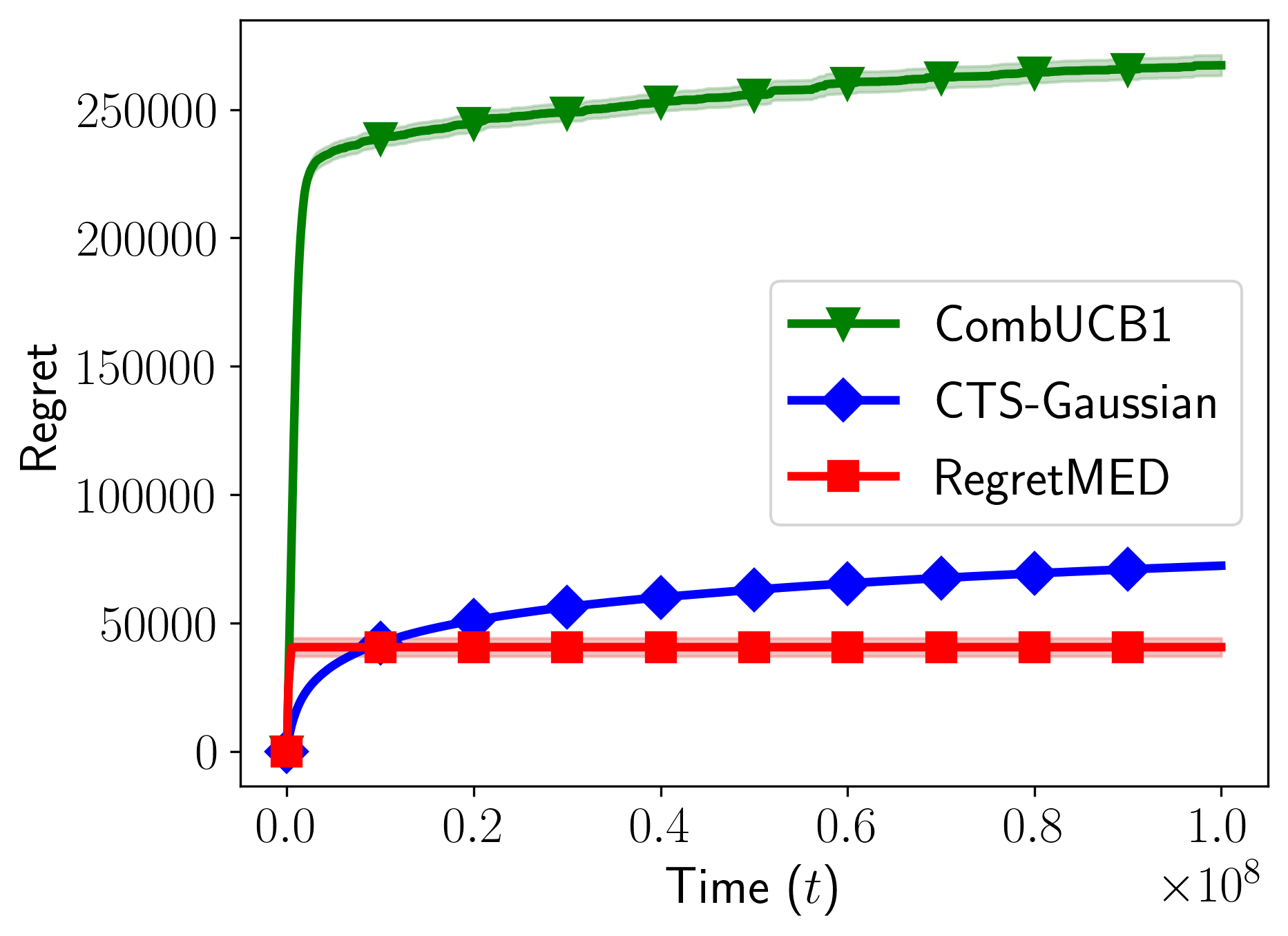}
  	\caption{Resource allocation example with $d = 25$.}
	\label{fig:resource25}
     \end{minipage}
          \begin{minipage}[b]{0.32\textwidth}
         \centering
          \includegraphics[width=\linewidth]{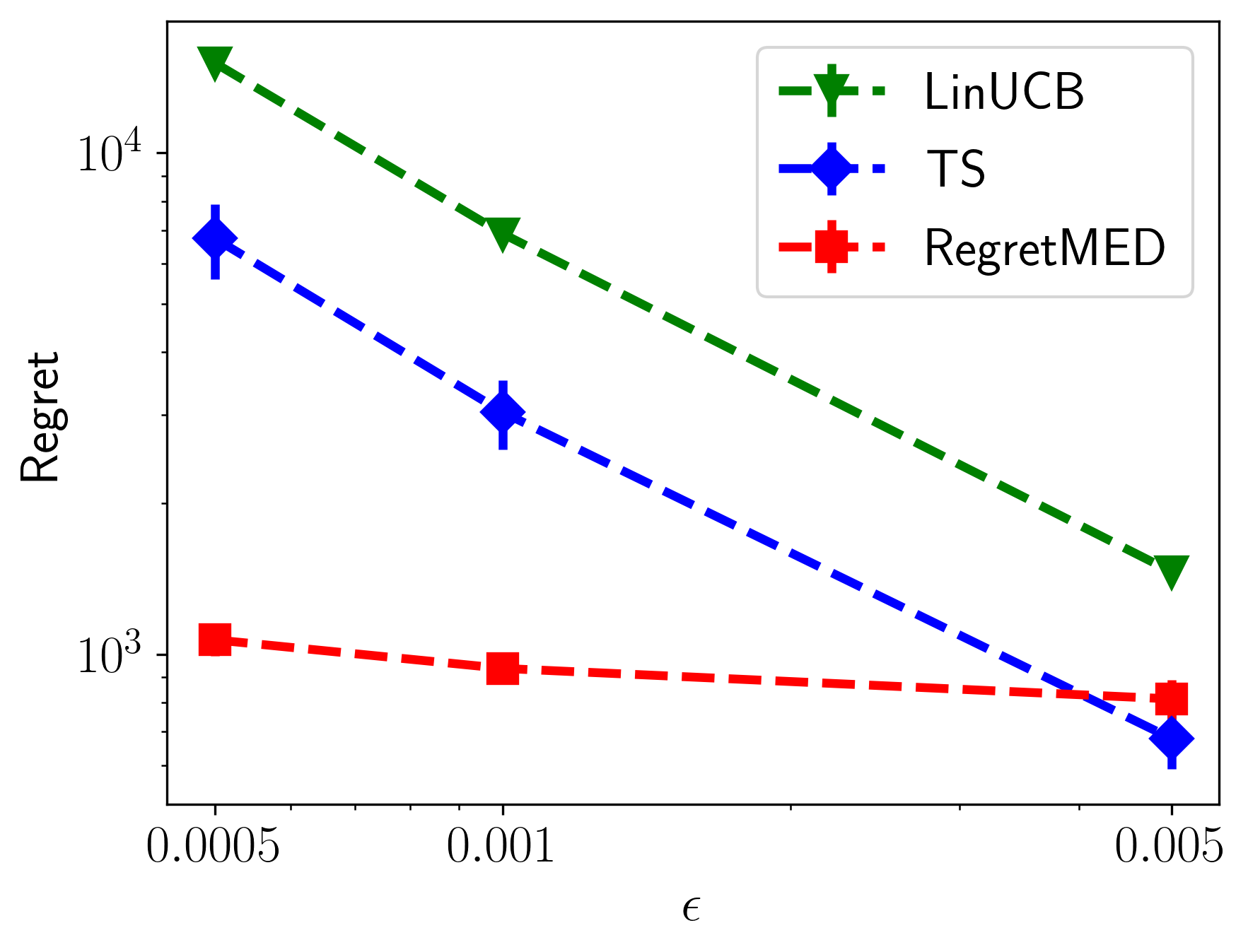}
  \caption{End of Optimism example varying $\epsilon$.}
       \label{fig:soare}
     \end{minipage}
\end{figure*}

\textbf{Rounding:} The allocation $\tau_\ell$ is not integer, so must be rounded. Naive rounding could incur problematically large regret, so we instead seek a sparse allocation, which will allow us to round without incurring significant regret. Recalling that $n_{\mathrm{band}} = d^2+d+1, n_{\sband} = d+1$, we have:

\begin{lemma}\label{lem:rounding}
Given $\tau_\ell$ a solution to \eqref{eq:vl_inefficient} or \eqref{eq:vl}, there exists an $\nf$-sparse $\alpha_\ell$ which achieves the same value of the constraint and objective of \eqref{eq:vl_inefficient} or \eqref{eq:vl}, respectively. Furthermore, in the semi-bandit setting, if we run the procedure of Theorem \ref{thm:comp_complex} to find an approximate solution to \eqref{eq:vl}, we can compute $\alpha_\ell$ in time $\poly(d,\Delmin,T,1/\delta)$.
\end{lemma}

We prove this result and state how this rounded distribution can be computed in Appendix \ref{sec:rounding}.

\section{EXPERIMENTAL RESULTS}

We next present experimental results for RegretMED in both the semi-bandit and bandit feedback settings. Every point in each plot is the average of 50 trials. The error bars indicate one standard error. 

\textbf{Semi-Bandit Feedback:} We compare the computationally efficient version of RegretMED against CombUCB1 \citep{kveton2015tight} and CTS-Gaussian \citep{perrault2020statistical}, a formulation of Thompson Sampling in the semi-bandit setting. As a test instance, we consider a resource allocation problem where an agent is tasked with maximizing profit subject to production cost. In particular, assume there are $d$ buyers, each offering a different price for a good. At each timestep the agent can sell to any number of them, but incurs an additional production cost for each item they sell. The agent observes a noisy realization of the price the buyer they sold to is willing to pay and of the production cost.  In particular, if at time $t$ we sell to $k$ buyers $x_{t_1},\ldots,x_{t_k}$, we will pay production costs $y_1,\ldots,y_k$, where $y_i$ is the production cost of producing the $i$th good. We can model this problem with $\X \subseteq \R^{2d}$, $\theta_{*,1:d}$ corresponding to the prices each buyer will pay, and $\theta_{*,d+1:2d}$ corresponding to the costs, $y_i$. 

We illustrate the result in Figures \ref{fig:resource5} and \ref{fig:resource25} for different values of $d$. In both cases, RegretMED yields a significant improvement over CTS-Gaussian and CombUCB1. Note that $|\X|$ is growing exponentially in $d$ and for $d = 25$ we have $|\X| \approx 3 \cdot 10^7$. In all experiments we set $\delta = 1/T$.

\textbf{Bandit Feedback:} In the bandit setting, we compare against LinUCB \citep{abbasi2011improved} and Thompson Sampling. For Thompson Sampling we use the Bayesian version. We run on the instance described in \cite{lattimore2017end}. In particular, in this instance $\thetast = e_1 \in \R^2$ and $\X = \{ e_1, e_2, x \}$ where $x = [1-\epsilon, 8 \epsilon]$. We  set $\delta = 1/T$ and, for each experiment, use $T = 25/\epsilon^2$, which is the natural scaling for the problem since, as shown in \cite{lattimore2017end}, optimistic algorithms will require on order $1/\epsilon^2$ pulls to determine $x$ is suboptimal. For completeness, in Appendix \ref{sec:add_exp} we include the plots of regret against time for each point in this figure.

As Figure \ref{fig:soare} illustrates, the performance of RegretMED is almost unaffected by the choice of $\epsilon$, while the performance of both TS and LinUCB degrades significantly. Optimistic algorithms are suboptimal on this instance as they do not pull the suboptimal but informative arm, $e_2$. Our results indicate that RegretMED is able to overcome this difficulty by continuing to pull $e_2$ even when it has been determined suboptimal, recognizing the information gain outweighs the regret incurred.

\section{DISCUSSION AND PRIOR ART}\label{sec:related}
\begin{table*}[t]
\centering
\resizebox{\textwidth}{!}{
\begin{tabular}{|c|c|c|c|c|}
\hline
& Lower Bound &  Prior Art & Theorem  \ref{thm:inefficient_regret_bound} & Theorem \ref{thm:efficient_regret_bound} (Efficient)  \\ \hline
Semi-Bandit &   $\Theta \left (\frac{d \log(T)}{\Delmin} \right )$  & $ \tilO \left ( \frac{d \log^2(k) \log(T)}{\Delmin} + \frac{dk^2 \Delmax}{\Delmin^2} \right ) $     &  $ \tilO \left ( \frac{d \log(T)}{\Delmin} + \frac{ \gamworst(\Asb)}{\Delmin} +  dk \right ) $          &        $ \tilO \left ( \frac{\log^2(k) \gamworst(\Asb) \log(T)}{\Delmin} + dk \right ) $    \\ 
\hline
Bandit  & $\Theta \left (\frac{d \log(T)}{\Delmin} \right )$ & $\tilO \left (\frac{d\log(T) + d\log(|\X|)}{\Delmin} \right )$ & $\tilO \left (\frac{d\log(T) + \gamworst(\Ab)}{\Delmin}  \right )$ & (Not Efficient)  \\ \hline
\end{tabular}
}
\caption{Gap-dependent expected regret guarantees in bandit and semi-bandit feedback settings. Note that lower bounds stated hold only for specific instances (e.g. standard multi-armed bandits with equal gaps).}
\label{tab:regret_summary}
\end{table*}

\textbf{Linear Bandits with Bandit Feedback:} Several of the most well-studied algorithms for regret minimization in stochastic linear bandits with bandit feedback are LinUCB \citep{abbasi2011improved}, action elimination, and LinTS \citep{lattimore2020bandit}. LinUCB achieves regret of $\tilO(d \sqrt{T})$, action elimination has regret bounded as $\tilO(\sqrt{d T \log(|\X|)})$, and Thompson Sampling has (frequentist) regret of $\tilO(d^{3/2}\sqrt{T})$. Both LinUCB and action elimination rely on wasteful union bounds---LinUCB union bounds over every direction in $\R^d$, incurring an extra factor of $\sqrt{d}$, while action elimination union bounds over every arm without regard to geometry, incurring an extra $\log(|\X|)$. By leveraging tools from empirical process theory, we develop bounds that depend on the fine-grained geometry of $\X$. Indeed, as already stated, our algorithm achieves an expected regret of $\tilO(\sqrt{\gamworst(\Ab) T})$ which, by Proposition \ref{prop:gw_bound_union}, is at least as good as, and in some cases much better than the bounds of LinUCB and action elimination (see Proposition \ref{prop:gw_topk_product}). Our bound can be seen as similar in spirit to the problem-dependent minimax bound for regret minimization in MDPs given in \cite{zanette2019tighter}.

\textbf{Combinatorial Bandits with Semi-Bandit Feedback:} Significant attention has been given to the combinatorial semi-bandit problem. \cite{kveton2015tight} handles the case where noise is correlated between coordinates, and provides a computationally efficient algorithm with a regret bound of $\tilO \left ( \frac{d k \log(T)}{\Delmin} + dk \right )$. \cite{degenne2016combinatorial} builds on this, showing that if the noise is assumed to be uncorrelated between coordinates, the $k$ on the leading term can be improved to a $\log^2(k)$. Although their algorithm is not computationally efficient, several subsequent works proposed efficient procedures that achieved similar regret bounds \citep{wang2018thompson,perrault2020statistical,cuvelier2020statistically}.

We give the first upper bound on regret (Theorem \ref{thm:inefficient_regret_bound}) that matches the lower bound on the leading $\log(T)$ term. Prior works are loose by a factor of $\log^2(k)$ and, moreover, have large additive terms that dominate until $T \geq \tilO( \exp( \frac{k^2 \Delmax}{\log^2(k) \Delmin}))$, making their bounds essentially vacuous for all practical time regimes. Although our analysis of the computationally efficient algorithm does not match the lower bound, its leading term is $\frac{d k \log^3(k)  \log(T)}{\Delmin} $ in the worst case, and, due to our smaller additive terms, our regret bound improves on the state of the art until $T \geq \tilO( \exp( \frac{k \Delmax}{\log^3(k) \Delmin}))$. Furthermore, Proposition \ref{prop:gw_topk_plus1} implies that there exist instances where Theorem \ref{thm:efficient_regret_bound} matches the state-of-the-art in the leading term, up to a single $\log(k)$ factor. While we have assumed the noise between coordinates is uncorrelated, RegretMED extends to the case where it is correlated by using $A_{\mathrm{cor}}(\lambda) = \Sigma \circ \Asb(\lambda)^{-1} \Ab(\lambda) \Asb(\lambda)^{-1}$ for $\Sigma$ an upper bound on the noise covariance and $\circ$ denoting element-wise multiplication.

While prior algorithms have tended to be based on the principle of optimism \citep{kveton2015tight,combes2015combinatorial,degenne2016combinatorial,wang2018thompson,perrault2020statistical}, we have shown that optimistic strategies are asymptotically suboptimal (see Proposition \ref{prop:opt_counterexample_semibandit}), motivating our planning-based algorithm. Additional work includes \citep{chen2016combinatorial,talebi2016optimal,perrault20a}. We summarize our results in Table \ref{tab:regret_summary}.

\textbf{Asymptotically Optimal Regret in Linear Bandits:} Another related line of work focuses on asymptotic performance \citep{lattimore2017end,combes2017minimal,hao2020adaptive,degenne2020structure,cuvelier2020statistically}. In the bandit setting asymptotic lower bounds have been shown to scale as:
\begin{align*}
& \min_{\tau} \sum_{x \in \X} \Delta_x \tau_x \quad  \text{s.t.} \quad \| x \|_{\Ab(\tau)^{-1}}^2/\Delta_x^2 \leq \frac{1}{2}, \forall x \neq \xst
\end{align*}
While we do not claim RegretMED is asymptotically optimal, we note that the optimization we are solving (\ref{eq:vl_inefficient}) closely resembles the above optimization. Indeed, at the final epoch of RegretMED, our estimates of the gaps will be sufficiently accurate so as to ensure we are playing approximately the asymptotically optimal distribution. Furthermore, as Proposition \ref{prop:opt_counterexample_semibandit} and Figure \ref{fig:soare} show, RegretMED appears to be playing the asymptotically optimal strategy in situations where optimism fails. \del{To the best of our knowledge, RegretMED is the first algorithm to plan in this way but still provide finite-time guarantees.} We leave a rigorous proof of the asymptotic qualities of RegretMED to future work. 

Concurrent to this work, several works appeared which simultaneously achieve asymptotically optimal and sub-$\calO(\sqrt{T})$ regret \citep{tirinzoni2020asymptotically,kirschner2020asymptotically}. In particular, \cite{tirinzoni2020asymptotically} achieves instance-optimal $\log T$ regret in finite time. We remark that their regret bound contains large additive terms which will dominate the leading $\log T$ term for moderate time horizons. Our primary concern is in this non-asymptotic regime, where the union bound applied is still significant, and we therefore see our work as complementary, addressing issues they do not address.

Asymptotically optimal regret has been relatively unexplored in the semi-bandit setting. Following the acceptance of this work, a very recent work \citep{cuvelier2021asymptotically} proposed a computationally efficient asymptotically optimal algorithm in the semi-bandit setting, which was the first of its kind. As with the bandit setting, our concern is with the non-asymptotic time regime, so this result is complementary to ours.

\textbf{Stochastic Multi-Armed Bandits with Side Observations:} In the stochastic multi-armed bandits with side observations problem, the agent is given a graph of $n$ nodes where each node is associated with an independent distribution. When the agent pulls a node $i$, she observes and suffers its stochastic reward and she also observes the stochastic reward of any node with an edge connected to node $i$. \cite{caron2012leveraging} proposed a UCB-like algorithm and \cite{buccapatnam2014stochastic} used a linear programming solution to show that the regret scales with the minimum dominating set. 

Using the design matrix $A_{\mathrm{graph}}(\lambda) = \sum_{i=1}^n \lambda_i \sum_{(i,j) \in E} e_j e_j^\t $, where $E$ denotes the edges in the graph, our algorithmic approach offers an explicit and natural way to model the tradeoff between estimated regret and information gain in this setting. In addition, our work suggests an algorithm for a novel extension of this problem where each node $i$ is associated with a feature vector $x_i \in \R^d$ and the expected reward of $i$ is $\theta_*^\t x_i$, that is, stochastic linear bandits with side observations.

\textbf{Partial Monitoring:} The partial monitoring problem \citep{cesa2006prediction,cesa2006regret,bartok2011minimax} is a generalization of the multi-armed bandit problem where now the learner is no longer able to directly observe the loss incurred, but only some function of it. The linear partial monitoring problem \citep{lin2014combinatorial,kirschner2020information} is a special case where the learner observes $y_t = z_{x_t}^\top \thetast + \eta_t$, for some known $z_x$, but receives reward $x_t^\top \thetast$, which is not observed. RegretMED directly generalizes to this setting if we employ the design matrix $A_{\mathrm{pm}}(\lambda) = \sum_{x \in \X} \lambda_x z_x z_x^\top$. We leave a full investigation of this application to future work.

\textbf{Pure Exploration in Multi-Armed Bandits:} There has not been a significant amount of previous work on pure exploration combinatorial bandits with semi-bandit feedback. \cite{chen2020} provide a general framework that subsumes combinatorial bandits with semi-bandit feedback but their algorithm is non-adaptive and suboptimal. Several special cases of pure exploration combinatorial bandits with semi-bandit feedback have been studied. Best arm identification (where $\X = \{e_1,\ldots, e_d\}$) has received much attention \citep{even2006action,jamieson2014lil,icml2013_karnin13,kaufmann2016complexity,chen2015optimal}. The setting in \cite{jun2016top} subsumes the top-K problem, but their approach does not generalize to other combinatorial problem instances. Concurrent to this work, \cite{jourdan2021efficient} derived an asymptotically optimal best arm identification algorithm for the semi-bandit setting. We note that our result focuses on optimality in the finite-time regime, so our results our complementary.

Our work is also related to transductive linear bandits \citep{fiez2019sequential}. In this problem, there are measurement vectors $\X \subset \R^d$, item vectors $\Z \subset \R^d$, and the agent at each round chooses $x_t \in \X$ and observes the realization of a noisy random variable with mean $x_t^\t \theta$ with the goal to identify $\argmax_{z \in \Z} \theta^\t z$ as quickly as possible. Our work on combinatorial bandits with semi-bandit feedback can be straightforwardly extended to a generalization of transductive linear bandits that allows for multiple measurements at each round. More concretely, in this setting, the agent is given a collection of subsets of $\X$, $\C \subset 2^{\X}$, and at each round, she chooses a \emph{set} of linear measurements $Y_t \subset \X$ where $Y_t \in \C$, and observes the realization of a noisy random variable with mean $x^\t \theta$ for each $x \in Y_t$. This generalization subsumes the work of \cite{wu2015identifying}, which studies a version of this problem where $\X = \Z = \{e_1, \ldots, e_d\}$.

Our algorithmic technique bridging empirical process theory and experimental design is inspired by the work on pure exploration combinatorial bandits in \cite{katz2020empirical}. The semi-bandit feedback setting in the present paper poses a new and non-trivial computational challenge since, unlike in \cite{katz2020empirical}, the number of variables in the optimization is potentially exponential in the dimension.

\section*{Acknowledgements}
AW is supported by an NSF GFRP Fellowship DGE-1762114. JKS is supported by an Amazon Research Award. The work of KJ is supported in part by grants NSF RI 1907907 and NSF CCF 2007036.

\bibliographystyle{plainnat}
\bibliography{bibliography.bib}

\appendix 

\onecolumn
\tableofcontents

\section{Action Elimination with Gaussian Width}\label{sec:aegw}

We first state an algorithm inspired by \cite{lattimore2020bandit} and prove a regret bound. This algorithm, while naive, incorporates the TIS inequality to obtain regret scaling with the Gaussian width. Furthermore, the analysis is simple and helps aid in the intuition of the proof of our main theorems. 

For $\frakf \in \{ \band, \sband \}$, denote:
$$ \gamma(\Af(\lambda),\mathcal{X}) := \mathbb{E}_{\eta \sim \mathcal{N}(0,I)} \left [ \sup_{x \in \mathcal{X}} x^\top \Af(\lambda)^{-1/2} \eta \right ]^2$$

\begin{algorithm}[H] 
\begin{algorithmic}[1]
\State \textbf{Input:} Set of arms $\mathcal{X}$, confidence $\delta$, largest gap $\Delmax$, rounding parameter $\zeta \in (0,1)$
\State $\hat{\mathcal{X}}_1 \leftarrow \mathcal{X}, \ell \leftarrow 1,$
\While{$|\hat{\mathcal{X}}_\ell| > 1$}
 	 \State Let $\hat{\lambda}_\ell$ a minimizer of $\mathbb{E}_\eta [ \max_{x \in \hat{\mathcal{X}}_\ell} x^\top A(\lambda)^{-1/2} \eta ]^2 +  \max_{x \in \hat{\mathcal{X}}_\ell} \| x \|^2_{A(\lambda)^{-1}}$
	 \State $\epsilon_\ell =  \Delmax 2^{-\ell}$, $\tau_\ell =  2(1+\zeta)\epsilon_{\ell}^{-2} (\gamma(A(\hat{\lambda}_\ell),\hat{\mathcal{X}}_\ell)   + 2 \sup_{x \in \hat{\mathcal{X}}_\ell} \| x \|_{A(\hat{\lambda}_\ell)^{-1}}^2 \log(2\ell^2/\delta))$
	 \State $\kappa_\ell \leftarrow $ ROUND($\hat{\lambda}_\ell,\lceil \tau_\ell \rceil \vee q(\zeta),\zeta$)
	 \State Pull arm $x$ $\kappa_{\ell,x}$ times, compute $\hat{\theta}_\ell$ from this data
	 \State $\hat{\X}_{\ell + 1} \leftarrow \hat{\X}_\ell \backslash \{ x \in \hat{\X}_\ell \ : \ \max_{x' \in \hat{\X}_\ell} (x' - x)^\top \hat{\theta}_\ell > 2 \epsilon_\ell \} $
	 \State $\ell \leftarrow \ell + 1$
\EndWhile
\end{algorithmic}
\caption{Gaussian Width Action Elimination (GW-AE)}
\label{alg:gw_ae}
\end{algorithm}
Here ROUND$(\lambda,N,\zeta)$ is a rounding procedure which takes as input $\lambda \in \triangle_{\X}$, $N \in \N$, and $\zeta \in (0,1)$ and outputs an allocation $\kappa \in \N^{|\X|}$ such that:
$$ \gamma(A(\kappa),\X) + \sup_{x \in \X} \| x \|_{A(\kappa)^{-1}}^2 \leq (1 + \zeta) \left ( \gamma(A(\tau \lambda),\X) + \sup_{x \in \X} \| x \|_{A(\tau \lambda)^{-1}}^2 \right ) $$
and $\sum_{x \in \X} \kappa_x = N$, so long as $N \geq q(\zeta)$. From \cite{katz2020empirical} and \cite{allen2020near}, we know such a rounding procedure exists and can be computed efficiently, and that it suffices to choose $q(\zeta) = O(d/\zeta^2)$. 

Denote:
$$\gamworstae(\Af) = \sup_{\epsilon > 0} \sup_{\Y \subseteq \X_\epsilon} \inf_{\lambda \in \triangle_{\Y \cup \xst}} \mathbb{E}_\eta \left [ \sup_{x \in \Y \cup \xst} x^\top \Af(\lambda)^{-1/2} \eta \right ]^2$$
where $\X_\epsilon := \{ x \in \X \ : \ \Delta_x \leq \epsilon \}$.

\begin{theorem}
For $\frakf \in \{ \band, \sband \}$, the absolute regret of GW-AE is bounded as:
$$ c_1 \Delmax \log(\Delmax / \Delmin ) d  + \frac{c_2 (\gamworstae(\Af) +  d \log(\log(\Delmax / \Delmin)/\delta))}{\Delmin} $$
with probability at least $1-\delta$ and minimax regret as:
$$ c_1 \Delmax \log(\Delmax / \Delmin) d + c_2 \sqrt{ (\gamworstae(\Af) + d \log( \log(\Delmax / \Delmin)/\delta))T} $$
with probability at least $1-\delta$. Here $c_1,c_2$ are absolute constants. 
\end{theorem}
If desired, noting that $\tau_\ell \geq \epsilon_\ell^{-2}$ which implies that we will have at most $\calO(\log(T))$ rounds, the $\log(\Delmax / \Delmin)$ could be replaced with a term $\calO(\log(T))$, as in Theorem \ref{thm:inefficient_regret_bound}. 

While this result closely resembles Theorem \ref{thm:inefficient_regret_bound}, there are several major shortcomings. First, this algorithm does not plan as effectively as it only pulling arms with gap less than $\epsilon_\ell$, which could cause it to forego pulling informative yet suboptimal arms, something Algorithm \ref{alg:gw_ae_comp} improves on. In particular, the regret bound stated for Algorithm \ref{alg:gw_ae_comp} in Proposition \ref{prop:opt_counterexample_semibandit} will not hold for this algorithm. In a sense, this algorithm can be thought of as being optimistic. Second, it is always the case that $\gamworst(\Af) \leq \gamworstae(\Af)$. The parameter $\gamworstae(\Af)$ could be tightened by altering the constant factors in Algorithm \ref{alg:gw_ae} so as to guarantee that, on the good event \emph{all} arms with gap less than $\epsilon$, for some $\epsilon$, are in $\hat{\X}_\ell$. However, even with this tightening we will always have $\gamworst(\Af) \leq \gamworstae(\Af)$. Finally, Algorithm \ref{alg:gw_ae} does not seem to admit a computationally feasible solution in the combinatorial bandit setting.

\begin{proof}
From Proposition \ref{prop:tis}, we will have that:
\begin{align*}
x^\top (\hat{\theta}_\ell - \theta_*) & \leq  \mathbb{E}_{\eta \sim \mathcal{N}(0,I)} \left [ \sup_{x \in \hat{\mathcal{X}}_\ell} x^\top A(\kappa_\ell)^{-1/2} \eta \right ] + \sqrt{2 \sup_{x \in \hat{\mathcal{X}}_\ell} \| x \|_{A(\kappa_\ell)^{-1}}^2 \log(2\ell^2/\delta)} \\
& \leq \epsilon_\ell 
\end{align*}
for all $x \in \hat{\mathcal{X}}_\ell$ simultaneously with probability $1 - \delta/\ell^2$. The second inequality holds by our choice of $\tau_\ell$ and Kiefer-Wolfowitz and Proposition \ref{prop:semi_kf}. Let:
$$ \mathcal{E}_{x,\ell}(\mathcal{V}) = \{ | \langle x, \hat{\theta}_\ell - \theta^* \rangle | \leq \epsilon_\ell \} $$ 
where $\hat{\theta}_\ell$ is computed assuming $\mathcal{V}$ is the active set in the above algorithm. Then using the following calculation from \cite{jamiesonbandits}:
\begin{align*}
\mathbb{P} \left [ \bigcup_{\ell = 1}^\infty \bigcup_{x \in \hat{\mathcal{X}}_\ell} \mathcal{E}_{x,\ell}(\hat{\mathcal{X}}_\ell)^c \right ] & \leq \sum_{\ell =1 }^\infty \sum_{\mathcal{V} \subseteq \mathcal{X}} \mathbb{P} \left [ \bigcup_{x \in \mathcal{V}} \mathcal{E}_{x,\ell}(\mathcal{V})^c \right ] \mathbb{P} [ \hat{\mathcal{X}}_\ell = \mathcal{V} ] \\
& \leq \sum_{\ell =1 }^\infty \sum_{\mathcal{V} \subseteq \mathcal{X}}\frac{\delta}{\ell^2} \mathbb{P} [ \hat{\mathcal{X}}_\ell = \mathcal{V} ] \\
& \leq \delta
\end{align*}
so the good event, that all the arm rewards are well estimated for all rounds, holds with high probability. Assume henceforth that the good event $\mathcal{E} = \cap_{\ell = 1}^\infty \cap_{x \in \mathcal{X}}  \mathcal{E}_{x,\ell}(\mathcal{V})$ holds. Following identically the argument from \cite{jamiesonbandits}, we will have that $x_* \in \hat{\X}_\ell$ and $\max_{x \in \hat{\X}_\ell} (x_* - x)^\top \thetast \leq 8\epsilon_\ell$ for all $\ell$. We assume the good event holds for the remainder of the proof.

\newcommand{\calY}{\mathcal{Y}}

We can now follow the same argument as Lemma 12 of  \cite{katz2020empirical}. Take $\calY \subseteq \X_\epsilon$ for some $\epsilon$ and let $\lambda_1 \in \triangle_{\calY}$ be the distribution that minimizes: 
$$\max_{x \in \mathcal{Y}} \| x \|^2_{A(\lambda)^{-1}}$$
and $\lambda_2 \in \triangle_{\Y}$ the distribution that minimizes:
$$\mathbb{E}_\eta [ \max_{x \in \mathcal{Y}} x^\top A(\lambda)^{-1/2} \eta ]^2$$
Let $\lambda = \frac{1}{2}(\lambda_1 + \lambda_2)$. Then we will have that:
$$ 2 A(\lambda_i)^{-1} \succeq A (\lambda)^{-1}$$
From this it immediately follows that:
$$ \max_{x \in \mathcal{Y}} \| x \|^2_{A(\lambda)^{-1}} \leq 2 \max_{x \in \mathcal{Y}} \| x \|^2_{A(\lambda_1)^{-1}} \leq 2d$$
where the last inequality holds by Kiefer-Wolfowitz and Proposition \ref{prop:semi_kf}. Also:
$$ \mathbb{E}_\eta [ \max_{x \in \mathcal{Y}} x^\top A(\lambda)^{-1/2} \eta ]^2 \leq 2 \mathbb{E}_\eta [ \max_{x \in \mathcal{Y}} x^\top A(\lambda_2)^{-1/2} \eta ]^2 \leq 2 \gamworstae(A)$$
Since $\hat{\X}_\ell$ will always contain only arms with gap less than $\epsilon$ for some $\epsilon$, we then have that:
$$ \tau_\ell \leq c (1 + \zeta) \epsilon_\ell^{-2} ( 2 d \log(2\ell^2/\delta) + \gamworstae(A)) $$

Using these bounds and noting that $\lceil \log_2(8\Delmax / (\Delmin \vee \nu)) \rceil$ upper bounds the number of rounds, we can upper bound the regret as:
\begin{align*}
& \sum_{x \in \mathcal{X} \backslash \{ x^* \}} \Delta_x T_x \\
& \leq T \nu + \sum_{\ell = 1}^{\lceil \log_2(8\Delmax / (\Delmin \vee \nu)) \rceil} 8 \epsilon_\ell (\tau_\ell + q(\zeta) + 1) \\
& \leq T \nu + 8 \Delmax \lceil \log_2(8\Delmax / (\Delmin \vee \nu)) \rceil (q(\zeta) + 1)  +  \sum_{\ell=1}^{\lceil \log_2(8\Delmax / (\Delmin \vee \nu)) \rceil} c (1+\zeta) \epsilon_{\ell}^{-1} (\gamworstae(A), \hat{\X}_\ell) + 2d\log(2\ell^2/\delta)) \\
& \leq T \nu + 8 \Delmax \lceil \log_2(8\Delmax / (\Delmin \vee \nu)) \rceil (q(\zeta) + 1)   +  \sum_{\ell=1}^{\lceil \log_2(8\Delmax / (\Delmin \vee \nu)) \rceil} c (1+\zeta) \epsilon_{\ell}^{-1} (\gamworstae(A) + 2d\log(2\ell^2/\delta)) \\
& \leq T \nu + 8 \Delmax \lceil \log_2(8\Delmax / (\Delmin \vee \nu)) \rceil (q(\zeta) + 1)  +  \sum_{\ell=1}^{\lceil \log_2(8\Delmax / (\Delmin \vee \nu)) \rceil} c (1+\zeta) \frac{2^\ell}{\Delmax} (\gamworstae(A) + 2d\log(2\ell^2/\delta)) \\
& \leq T \nu + 8 \Delmax \lceil \log_2(8\Delmax / (\Delmin \vee \nu)) \rceil (q(\zeta) + 1)  + \frac{c (1+\zeta) (\gamworstae(A) + 2 d \log(2 \log_2^2(16\Delmax / (\Delmin \vee \nu))/\delta))}{\Delmin \vee \nu}
\end{align*}
Optimizing this over $\nu$ gives the final regret of:
$$ 8 \Delmax \lceil \log_2(8\Delmax / (\Delmin)) \rceil (q(\zeta) + 1)  + \sqrt{c (1+\zeta) (\gamworstae(A)+ d \log( \log(\Delmax / (\Delmin))/\delta))T} $$
and choosing $\nu = 0$ gives the absolute regret bound.
\end{proof}

\section{Regret Bound Proofs}\label{sec:regret_proof}

\begin{proof}[Proof of Theorem \ref{thm:efficient_regret_bound}]
Throughout we will let $\calR_\ell$ denote the regret incurred in round $\ell$, and $\calR_{1:\ell}$ the regret incurred from rounds 1 through $\ell$. We assume $A(\tau)$ corresponds to the type of feedback received. The first part of this proof closely mirrors the proof of Theorem 5 of \cite{katz2020empirical}. We will prove this result for $\tau_\ell$ being a $(\nu,\zeta)$-optimal solution to (\ref{eq:vl}), where we calll a solution to (\ref{eq:vl}) $(\nu,\zeta)$-optimal if $\widehat{\opt} \leq \nu \opt + \zeta$, where $\widehat{\opt}$ is the value of the objective attained by the approximate solution, and $\opt$ the value attained by the optimal solution.

\textbf{Good event:} We will define $\calS_\ell$ as the following:
$$ \calS_\ell := \{ x \in \X \ : \ \Delta_x \leq \epsilon_\ell \} $$
Let $\delta_k = \delta / (2k^3) $ and define the events:
$$ \calE_{k,j} = \left \{ \sup_{z,z' \in \calS_j} |(z-z')^\top (\hat{\theta}_k - \thetast)| \leq (1 + \sqrt{\pi \log(1/\delta_k)}) \E_\eta \left [ \sup_{z,z' \in \calS_j} (z-z')^\top A(\tau_k)^{-1/2} \eta \right ] \right \} $$
$$ \calE = \bigcap_{k=1}^\infty \bigcap_{j=0}^k \calE_{k,j} $$
Proposition \ref{prop:tis} gives that with probability at least $1-\delta/k^3$:
\begin{align*}
\sup_{z,z' \in \calS_j} |(z-z')^\top (\hat{\theta}_k - \thetast)| & \leq  \E_\eta \left [ \sup_{z,z' \in \calS_j} (z-z')^\top A(\tau_k)^{-1/2} \eta \right ] + \sqrt{2 \max_{z,z' \in \calS_j} \| z - z' \|_{A(\tau_k)^{-1}}^2 \log(1/\delta_k)} \\
& \overset{(a)}{\leq} (1 + \sqrt{\pi \log(1/\delta_k)}) \E_\eta \left [ \sup_{z,z' \in \calS_j} (z-z')^\top A(\tau_k)^{-1/2} \eta \right ]
\end{align*}
where $(a)$ follows by Lemma 11 of \cite{katz2020empirical}. It follows then that $\P[\calE_{k,j}^c] \leq \delta / k^3$, which implies that:
$$ \P[\calE^c]  \leq \sum_{k=1}^\infty \sum_{j=0}^k \P[\calE_{k,j}^c] \leq  \sum_{k=1}^\infty \sum_{j=0}^k \frac{\delta}{k^3} \leq 3\delta$$

\textbf{Estimation error:} Henceforth we assume $\calE$ holds. We proceed by induction to show that the gaps are always well-estimated. First we prove the base case. Let $k = 1$ and consider any $x \in \X$. Then:
\begin{align*}
|(\xst - x)^\top (\hat{\theta}_1 - \thetast)| & \leq \sup_{z,z' \in \X} | (z - z')^\top (\hat{\theta}_1 - \thetast)| \\
& \leq (1 + \sqrt{\pi \log(1/\delta_1)}) \E_\eta \left [ \sup_{z,z' \in \X} (z - z')^\top A(\tau_1)^{-1/2} \eta  \right ] \\
& \overset{(a)}{=} 2(1 + \sqrt{\pi \log(1/\delta_1)}) \E_\eta \left [ \sup_{z \in \X} (x_1 - z')^\top A(\tau_1)^{-1/2} \eta  \right ] \\
& \overset{(b)}{\leq}  \epsilon_1/8
\end{align*}
where $(a)$ follows by Proposition 7.5.2 of \cite{vershynin2018high} and $(b)$ follows since $\tau_1$ is a feasible solution to (\ref{eq:vl}). 

For the inductive step, assume that, for all $x \in \calS_k$:
$$ |(\xst - x)^\top (\hat{\theta}_k - \thetast)| \leq \epsilon_k/8 $$
and for all $x \in \calS_k^c$:
$$ |(\xst - x)^\top (\hat{\theta}_k - \thetast)| \leq \Delta_x/8 $$
Consider round $k+1$ and take $x \in \calS_{k+1}^c$. There then exists some $k' \leq k$ such that $x \in \calS_{k'} \backslash \calS_{k'+1}$. Then:
\begin{align*}
\frac{|(\xst - x)^\top (\hat{\theta}_{k+1} - \thetast)|}{ \Delta_x} & \leq \sup_{z,z' \in \calS_{k'}} \frac{|(z - z')^\top (\hat{\theta}_{k+1} - \thetast)|}{ \Delta_x} \\
& \leq (1 + \sqrt{\pi \log(1/\delta_{k+1})}) \E_\eta \left [ \sup_{z,z' \in \calS_{k'}} \frac{(z-z')^\top A(\tau_{k+1})^{-1/2} \eta }{ \Delta_x}  \right ] \\
& \overset{(a)}{=} 2 (1 + \sqrt{\pi \log(1/\delta_{k+1})}) \E_\eta \left [ \sup_{z \in \calS_{k'}} \frac{(x_{k+1} -z)^\top A(\tau_{k+1})^{-1/2} \eta }{ \Delta_x}  \right ] \\
& \overset{(b)}{\leq} 4 (1 + \sqrt{\pi \log(1/\delta_{k+1})}) \E_\eta \left [ \sup_{z \in \calS_{k'}} \frac{(x_{k+1} -z)^\top A(\tau_{k+1})^{-1/2} \eta }{\epsilon_{k+1} +  \Delta_x}  \right ] \\
& \overset{(c)}{\leq} 8 (1 + \sqrt{\pi \log(1/\delta_{k+1})}) \E_\eta \left [ \sup_{z \in \calS_{k'}} \frac{(x_{k+1} -z)^\top A(\tau_{k+1})^{-1/2} \eta }{\epsilon_{k+1} +  \Delta_z}  \right ] \\
& \leq 8 (1 + \sqrt{\pi \log(1/\delta_{k+1})}) \E_\eta \left [ \sup_{z \in \X} \frac{(x_{k+1} -z)^\top A(\tau_{k+1})^{-1/2} \eta }{\epsilon_{k+1} +  \Delta_z}  \right ] \\
& \overset{(d)}{\leq} 16 (1 + \sqrt{\pi \log(1/\delta_{k+1})}) \E_\eta \left [ \sup_{z \in \X} \frac{(x_{k+1} -z)^\top A(\tau_{k+1})^{-1/2} \eta }{\epsilon_{k+1} +  \hat{\Delta}_z}  \right ] \\
& \overset{(e)}{\leq} 1/8
\end{align*}
where $(a)$ follows by Proposition 7.5.2 of \cite{vershynin2018high}, $(b)$ follows since $\Delta_x \geq \epsilon_{k+1}$ by virtue of the fact that $x \in \calS_{k+1}^c$, so $\Delta_x \geq (\epsilon_{k+1} + \Delta_x)/2$, $(c)$ follows since $\Delta_x \in [\epsilon_{k'+1},\epsilon_{k'}]$ and for any $z \in \calS_{k'}$, we will have theta $\Delta_z \leq \epsilon_{k'}$, so $\epsilon_{k+1} + \Delta_x \geq \epsilon_{k+1} + \epsilon_{k'+1} \geq \epsilon_{k+1} + \Delta_z/2$, $(d)$ holds by the inductive hypothesis and Lemma 1 of \cite{katz2020empirical} and taking $\hat{\Delta}_z$ to be the estimate of $\Delta_z$ at round $k+1$, and $(e)$ holds since $\tau_{k+1}$ is a feasible solution to (\ref{eq:vl}). We can perform a similar calculation to get the same thing for $x \in \calS_{k+1}$, allowing us to conclude that, for all $x \in \calS_{k+1}$:
$$ |(\xst - x)^\top (\hat{\theta}_{k+1} - \thetast)| \leq \epsilon_{k+1}/8 $$
and for all $x \in \calS_{k+1}^c$:
$$ |(\xst - x)^\top (\hat{\theta}_{k+1} - \thetast)| \leq \Delta_x/8 $$
From this and Lemma 1 of \cite{katz2020empirical}, it follows that for all $\ell$ and $x \in \calS_\ell $:
$$\Delta_x \leq \hat{\Delta}_x + |\hat{\Delta}_x - \Delta_x| \leq \hat{\Delta}_x + \epsilon_\ell/2 \leq \hat{\Delta}_x + \epsilon_\ell  $$
and for $x \in \calS_\ell^c$:
$$\Delta_x \leq \hat{\Delta}_x + |\hat{\Delta}_x - \Delta_x| \leq 2 \hat{\Delta}_x \leq 2 \hat{\Delta}_x + 2 \epsilon_\ell $$
So the objective of (\ref{eq:vl}) upper bounds the real regret.  Further, on the good event, using Lemma 1 from \cite{katz2020empirical}, for any $\ell$ and $x \in \X$, we have:
\begin{equation}\label{eq:est_to_real_regret}
\frac{1}{2} (\epsilon_\ell + \Delta_x) \leq \epsilon_\ell + \hat{\Delta}_x \leq \frac{3}{2} (\epsilon_\ell + \Delta_x) 
\end{equation}
This implies that if we remove arm $x$ from $\hat{\X}_\ell$:
$$ \hat{\Delta}_x > 2\epsilon_\ell  \implies \hat{\Delta}_x + \epsilon_\ell > 3 \epsilon_\ell  \implies \frac{3}{2} ( \epsilon_\ell + \Delta_x) > 3 \epsilon_\ell  \implies \Delta_x > \epsilon_\ell$$
So, on the good event, if $\hat{\Delta}_x > 2\epsilon_\ell$, we will have identified the best arm correctly.

\textbf{Bounding the Round Regret:} From the previous section, we know that on the good event all our gaps will be well-estimated. From (\ref{eq:est_to_real_regret}), it follows that the constraint in (\ref{eq:vl}) is tighter than the following constraint:
\begin{equation}\label{eq:vl_constraint_up}
\mathbb{E}_\eta \left [ \max_{x \in \X} \frac{(x_\ell - x)^\top A(\tau)^{-1/2} \eta}{\epsilon_\ell + \Delta_x} \right ] \leq \frac{1}{256 (1 + \sqrt{\pi \log(2\ell^3/\delta)})} 
\end{equation}
so any $\tau$ satisfying this inequality is also a feasible solution to (\ref{eq:vl}).

Consider drawing some $\eta$ and let $x_\eta$ be the point $x \in \X$ that achieves the maximum above (if the solution is not unique, break ties by choose $x_\eta$ randomly from the $x \in \X$ for which the maximum is attained). If we assume that $x_\eta \in \mc{S}_\ell$, then it follows that:
\begin{align*}
\max_{x \in \X} \frac{(x_\ell - x)^\top A(\tau)^{-1/2} \eta}{\epsilon_\ell + \Delta_x} & = \max_{x \in \mc{S}_\ell} \frac{(x_\ell - x)^\top A(\tau)^{-1/2} \eta}{\epsilon_\ell + \Delta_{x}} \\
& \leq \max_{x \in \mc{S}_\ell} \frac{(x_\ell - x)^\top A(\tau)^{-1/2} \eta}{\epsilon_\ell }\\
& \overset{(a)}{\leq} \sum_{j=1}^\ell \max_{x \in \mc{S}_j} \frac{(x_\ell - x)^\top A(\tau)^{-1/2} \eta}{\epsilon_j } 
\end{align*}
where $(a)$ follows since we will always have:
$$ \max_{x \in \mc{S}_j} \frac{(x_\ell - x)^\top A(\tau)^{-1/2} \eta}{\epsilon_j }  \geq 0 $$
since $x_\ell \in \mc{S}_j$ for $j \leq \ell$ by Lemma 1 of \cite{katz2020empirical}. Assume that $x_\eta \in \mc{S}_k \backslash \mc{S}_{k+1}$. Then:
\begin{align*}
\max_{x \in \X} \frac{(x_\ell - x)^\top A(\tau)^{-1/2} \eta}{\epsilon_\ell + \Delta_x} & = \max_{x \in \mc{S}_k \backslash \mc{S}_{k+1}} \frac{(x_\ell - x)^\top A(\tau)^{-1/2} \eta}{\epsilon_\ell + \Delta_{x}} \\
& \overset{(a)}{\leq} 2 \max_{x \in \mc{S}_k \backslash \mc{S}_{k+1}} \frac{(x_\ell - x)^\top A(\tau)^{-1/2} \eta}{\epsilon_k}\\
& \leq \max_{x \in \mc{S}_k } \frac{(x_\ell - x)^\top A(\tau)^{-1/2} \eta}{\epsilon_k} \\
& \leq \sum_{j=1}^\ell \max_{x \in \mc{S}_j} \frac{(x_\ell - x)^\top A(\tau)^{-1/2} \eta}{\epsilon_j } 
\end{align*}
where $(a)$ uses the fact that for all $x \in \mc{S}_k \backslash \mc{S}_{k+1}$, $\Delta_x \in [\epsilon_{k+1},\epsilon_k]$, and the last inequality follows as above. We therefore have that:
$$ \mathbb{E}_\eta \left [ \max_{x \in \X} \frac{(x_\ell - x)^\top A(\tau)^{-1/2} \eta}{\epsilon_\ell + \Delta_x} \right ]  \leq \sum_{j=1}^\ell \frac{1}{\epsilon_j} \mathbb{E}_\eta \left [ \max_{x \in \mc{S}_j} (x_\ell - x)^\top A(\tau)^{-1/2} \eta \right ] $$
\newcommand{\lamjgw}{\lambda_j^{\mathrm{gw}}}
\newcommand{\taujgw}{\tau_j^{\mathrm{gw}}}
Let $\lamjgw$ be the solution to:
$$ \lamjgw = \argmin_{\lambda \in \triangle_{\calS_j}} \Exp_\eta [ \max_{x \in \calS_j} (x_\ell - x)^\top A(\lambda)^{-1/2} \eta] $$
Let $\bar{\tau} = \ell^2  \sum_{j=1}^\ell \taujgw $ and $\taujgw = 65536 \gamworst(A) \epsilon_j^{-2} (1 + \sqrt{\pi \log(2\ell^3/\delta)})^2 \lamjgw$. Then:
\begin{align*}
\mathbb{E}_\eta \left [ \max_{x \in \mc{S}_j} (x_\ell - x)^\top A(\taujgw)^{-1/2} \eta \right ] & = \frac{\mathbb{E}_\eta \left [ \max_{x \in \mc{S}_j} (x_\ell - x)^\top A(\lamjgw)^{-1/2} \eta \right ]}{\sqrt{\taujgw}} \\
& = \frac{\epsilon_j \mathbb{E}_\eta \left [ \max_{x \in \mc{S}_j} (x_\ell - x)^\top A(\lamjgw)^{-1/2} \eta \right ]}{ \sqrt{\gamworst(A)} 256 (1 + \sqrt{\pi \log(2\ell^3/\delta)})} \\
& \leq \frac{\epsilon_j}{ 256 (1 + \sqrt{\pi \log(2\ell^3/\delta)})} 
\end{align*}
Given this:
\begin{align*}
\sum_{j=1}^\ell \frac{1}{\epsilon_j} \mathbb{E}_\eta \left [ \max_{x \in \mc{S}_j} (x_\ell - x)^\top A(\bar{\tau})^{-1/2} \eta \right ] & \overset{(a)}{\leq} \frac{1}{\ell} \sum_{j=1}^\ell \frac{1}{\epsilon_j} \mathbb{E}_\eta \left [ \max_{x \in \mc{S}_j} (x_\ell - x)^\top A(\taujgw )^{-1/2} \eta \right ] \\
& \leq  \frac{1}{\ell} \sum_{j=1}^\ell \frac{1}{\epsilon_j} \frac{\epsilon_j}{256 (1 + \sqrt{\pi \log(2\ell^3/\delta)})} \\
& = \frac{1}{256 (1 + \sqrt{\pi \log(2\ell^3/\delta)})}
\end{align*}
where $(a)$ holds by the Sudakov-Fernique inequality (Theorem 7.2.11 of \cite{vershynin2018high}). Thus, $\bar{\tau}$ satisfies (\ref{eq:vl_constraint_up}) and so is a feasible solution to (\ref{eq:vl}). Let $\tau_\ell^*$ be the optimal solution to (\ref{eq:vl}), then:
$$ \sum_{x \in \X} 2(\epsilon_\ell + \hat{\Delta}_x) \tau_{\ell,x}^* \leq \sum_{x \in \X}2 (\epsilon_\ell + \hat{\Delta}_x) \bar{\tau}_x \leq  \sum_{x \in \X} 3(\epsilon_\ell + \Delta_x) \bar{\tau}_x = \sum_{x \in \X} 3 \Delta_x \bar{\tau}_x + 3 \epsilon_\ell \bar{\tau}$$
The first term can be bounded by the regret bounded given in Lemma \ref{lem:deterministic_ae2}:
$$  \sum_{x \in \X} 3 \Delta_x \bar{\tau}_x \leq  c_1 \Delmax  \ell d + \frac{c_2 \ell^2 \log(\ell/\delta) \gamworst(A)}{\epsilon_\ell}  $$
By construction we'll have that:
$$ \bar{\tau} = c \sum_{k=1}^\ell \epsilon_k^{-2} \gamworst(A) (1 + \sqrt{\pi \log(2 \ell^3/\delta)})^2 \leq c \gamworst(A) (1 + \sqrt{\pi \log(2 \ell^3/\delta)})^2 \epsilon_\ell^{-2}$$
so:
$$ 3 \epsilon_\ell \bar{\tau} \leq \frac{c \gamworst(A) \log(2 \ell^3 / \delta)}{\epsilon_\ell} $$
Recalling that $\tau_\ell$ is a $(\nu,\zeta)$-optimal solution to (\ref{eq:vl}), the above implies that:
\begin{equation}\label{eq:fl_up}
\sum_{x \in \X} 2(\epsilon_\ell + \hat{\Delta}_x) \tau_{\ell,x} \leq (1 + \nu) \sum_{x \in \X} 2(\epsilon_\ell + \hat{\Delta}_x) \tau_{\ell,x}^* + \zeta \leq (1 + \nu)  \left ( c_1 \Delmax  \ell d + \frac{c_2 \ell^2 \log(\ell/\delta) \gamworst(A)}{\epsilon_\ell}  \right ) + \zeta 
\end{equation}
We in fact play $\alpha_\ell$, as this will attain the same objective value and so the same regret bound. However, $\alpha_\ell$ may not be integer, so we will pull every arm $\lceil \alpha_{\ell,x} \rceil$ times. Note that the rounded solution still meets the constraint from (\ref{eq:vl}). Assume we are playing the rounded solution given by Lemma \ref{lem:rounding}, then rounding the solution will incur additional regret of at most $ \Delmax \nf$. Since $\sum_{x \in \X} 2 (\epsilon_\ell + \hat{\Delta}_x) \tau_{x}$ upper bounds the real regret of playing $\tau_x$, we'll have:
$$ \calR_\ell \leq (1 + \nu) \left ( c_1 \Delmax  \ell d + \frac{c_2 \ell^2 \log(\ell/\delta) \gamworst(A)}{\epsilon_\ell}  \right ) + \Delmax \nf + \zeta $$
We can then bound the regret incurred after $\ell$ stages as:
\begin{align}\label{eq:regret_l}
\begin{split}
\calR_{1:\ell} & \leq \sum_{k=1}^\ell (1 + \nu) \left ( c_1 \Delmax  k d + \frac{c_2 k^2 \log(k/\delta) \gamworst(A)}{\epsilon_k}  \right ) +  \ell \Delmax \nf + \ell \zeta \\
& \leq c_1 (1 + \nu)  \Delmax  \ell^2  d +  \ell \Delmax \nf + \ell \zeta + \sum_{k = 1}^{\ell} \frac{c_2 (1 + \nu) k^2 \log(\ell/\delta) \gamworst(A)}{\epsilon_k} \\
& \leq c_1 (1 + \nu)  \Delmax  \ell^2  d +  \ell \Delmax \nf + \ell \zeta + c_2 (1 + \nu) \log(\ell/\delta) \gamworst(A) \sum_{k=1}^\ell k^2 2^k \\
& \leq c_1 (1 + \nu)  \Delmax  \ell^2  d +  \ell \Delmax \nf + \ell \zeta + \frac{c_2 (1 + \nu) \ell^2 \log(\ell/\delta) \gamworst(A)}{\epsilon_\ell} 
\end{split}
\end{align}

\newcommand{\bl}{\bar{\ell}}
\newcommand{\lmle}{\ell_{\text{mle}}}
\textbf{Minimax Regret:} Denote the objective to (\ref{eq:vl}) at round $\ell$ evaluated at $\tau_\ell$ by:
$$ f_\ell := \sum_{x \in \X} 2 (\epsilon_\ell + \hat{\Delta}_x) \tau_{\ell,x} $$
By (\ref{eq:fl_up}) we can upper bound:
\begin{align}\label{eq:f_bound}
\begin{split}
f_\ell & \leq (1 + \nu)  \left ( c_1 \Delmax  \ell  d + \frac{c_2 \ell^2 \log(\ell/\delta) \gamworst(A)}{\epsilon_\ell}  \right ) + \zeta \\
& \leq c_1 (1 + \nu) \Delmax  \ell  d + \zeta + \frac{c_2 (1 + \nu) \ell^2 \log(\ell/\delta)  \gamworst(A)}{\epsilon_\ell} \\
& =: C_1 + \frac{C_2}{\epsilon_\ell}
\end{split}
\end{align}
Let $\bl$ be the first round for which:
$$ T \epsilon_\ell \leq C_1 + \frac{C_2}{\epsilon_\ell} $$
Note that, if $\epsilon_\ell$ solves this with equality, then:
$$ \epsilon_\ell = \frac{C_1}{2T} + \frac{1}{2} \sqrt{\frac{4 C_2}{T} + \frac{C_1^2}{T^2}} $$
is the only non-negative solution. It follows then that:
$$ \epsilon_{\bl} \leq \frac{C_1}{2T} + \frac{1}{2} \sqrt{\frac{4 C_2}{T} + \frac{C_1^2}{T^2}} \leq \frac{C_1}{T} + \sqrt{\frac{C_2}{T}}$$
Since $\epsilon_{\bl}$ is the largest such solution, it follows that $2 \epsilon_{\bl}$ doesn't satisfy this inequality so:
$$ 2 \epsilon_{\bl} > \frac{C_1}{2T} + \frac{1}{2} \sqrt{\frac{4 C_2}{T} + \frac{C_1^2}{T^2}} \geq \sqrt{\frac{C_2}{T}} $$
so in particular:
$$ \frac{1}{\epsilon_{\bl}} \leq \sqrt{\frac{4T}{C_2}} $$
Assume that $f_\ell \leq T \epsilon_\ell$ for all $\ell$. Using the monotonicity of $\epsilon_\ell$, for $\ell \geq \bl$, we'll have:
$$ f_\ell \leq T \epsilon_\ell \leq T \epsilon_{\bl} \leq C_1 + \sqrt{C_2 T} $$
Furthermore, by (\ref{eq:regret_l}), we'll have that the total regret up to round $\bl$ will be bounded as:
\begin{align*}
\calR_{1:\bl} & \leq c_1 (1 + \nu)  \Delmax  \bl^2  d +   \bl \Delmax \nf + \bl \zeta + \frac{c_2 (1 + \nu) \bl^2 \log(\bl/\delta) \gamworst}{\epsilon_{\bl}} \\
& \leq C_1 \bl +  \bl \Delmax \nf  +  \frac{C_2}{\epsilon_{\bl}} \\
&  \leq C_1 \bl +  \bl \Delmax \nf  + \sqrt{4 C_2 T}
\end{align*}
So in this case, since by Lemma \ref{lem:num_round_bound} there are at most $\ellmax(T)$ rounds, and since $f_\ell + \Delmax \nf $ upper bounds the regret of round $\ell$, we'll have that the total regret will be bounded as:
$$ \calR_T \leq \ellmax(T) \left ( C_1 + \Delmax \nf + 3 \sqrt{C_2 T} \right ) $$
Now assume there is some round such that $f_\ell > T \epsilon_\ell$ and denote this round as $\lmle$. By construction, it will be the case that the MLE at this point has gap at most $\epsilon_{\lmle}$, so the total regret incurred from playing the MLE for the remainder of time will be bounded as $T \epsilon_{\lmle}$. Further, note that by (\ref{eq:f_bound}):
$$ T \epsilon_{\lmle} < f_{\lmle} \leq C_1 + \frac{C_2}{\epsilon_{\lmle}} $$
By definition $\bl$ is the first round where $T \epsilon_{\ell} \leq C_1 + \frac{C_2}{\epsilon_\ell}$, so it follows that $\lmle \geq \bl$. We can then upper bound the total regret incurred as:
$$ \calR_T \leq \sum_{\ell=1}^{\bl} f_\ell + \sum_{\ell = \bl + 1}^{\lmle - 1} f_\ell + T \epsilon_{\lmle} +  \lmle \Delmax \nf$$
From (\ref{eq:regret_l}), as above, we can bound:
$$ \sum_{\ell=1}^{\bl} f_\ell \leq C_1\bl + \frac{C_2}{\epsilon_{\bl}} \leq C_1\bl + \sqrt{4 C_2 T} $$
Since by definition we'll have that $f_\ell \leq T \epsilon_\ell$ for $\ell \in [\bl+1,\lmle-1]$, the second term can be bounded as:
$$  \sum_{\ell = \bl + 1}^{\lmle - 1} f_\ell \leq T \sum_{\ell = \bl + 1}^{\lmle - 1} \epsilon_\ell \leq (\lmle - \bl - 2) T \epsilon_{\bl} \leq (\lmle - \bl - 2) ( C_1 + \sqrt{C_2 T}) $$
Finally:
$$ T \epsilon_{\lmle} \leq T \epsilon_{\bl} \leq C_1 + \sqrt{C_2 T} $$
Combining this, we have that:
$$ \calR_T \leq  \ellmax(T) ( C_1  +  \Delmax \nf +  4 \sqrt{C_2 T})$$

\textbf{Absolute Regret:}
Assume:
$$ T > \frac{C_1}{\Delmin} + \frac{C_2}{\Delmin^2} $$
then we'll have that $\epsilon_{\bl} < \Delmin$, so the algorithm will exit before reaching round $\epsilon_{\bl}$. In this case, since there are at most $\lceil \log(4\Delmax/\Delmin) \rceil$ stages by Lemma \ref{lem:num_round_bound} and since, as noted above, on the good event, once $|\hat{\X}_\ell| = 1$, we will have identified the best arm and so will incur 0 regret for the rest of time, (\ref{eq:regret_l}) gives:
\begin{align*}
\begin{split}
\calR_T & \leq  c_1 (1+\nu) \Delmax  \log_2(\Delmax / \Delmin)^2  d +  \lceil \log(4\Delmax/\Delmin) \rceil \Delmax \nf + \lceil \log(4\Delmax/\Delmin) \rceil \zeta \\
& \qquad \qquad \qquad +  \frac{c_2 (1+\nu) \gamworst(A) \log(\log(\Delmax/\Delmin)/\delta) \log_2(\Delmax / \Delmin)^2}{\Delmin}
\end{split}
\end{align*}

By definition, it will always be the case that $\epsilon_{\lmle} > \Delmin$, if it exists, as we would have otherwise exited the algorithm already. By (\ref{eq:regret_l}), we'll then have:
\begin{align*}
\calR_T & \leq \calR_{1:\lmle} + T \epsilon_{\lmle} \\
& \leq c_1 (1+\nu) \Delmax  \log_2(\Delmax / \Delmin)^2  d +  \lceil  \log_2(4\Delmax / \Delmin) \rceil \Delmax \nf + \lceil \log(4\Delmax/\Delmin) \rceil \zeta\\
& \qquad \qquad \qquad + \frac{c_2 (1+\nu) \gamworst(A) \log(\log(\Delmax/\Delmin)/\delta) \log_2(\Delmax / \Delmin)^2}{\epsilon_{\lmle}} + T \epsilon_{\bl}\\
& \overset{(a)}{\leq} c_1 (1+\nu) \Delmax  \log_2(\Delmax / \Delmin)^2  d +  \lceil  \log_2(4\Delmax / \Delmin) \rceil \Delmax \nf + \lceil \log(4\Delmax/\Delmin) \rceil \zeta\\
& \qquad \qquad \qquad + \frac{c_2 (1+\nu) \gamworst(A) \log(\log(\Delmax/\Delmin)/\delta) \log_2(\Delmax / \Delmin)^2}{\epsilon_{\lmle}} + C_1 + \frac{C_2}{\epsilon_{\bl}} \\
& \leq 2C_1\log_2(\Delmax / \Delmin) +  \lceil  \log_2(4\Delmax / \Delmin) \rceil \Delmax \nf + \frac{2C_2}{\Delmin}
\end{align*}
where $(a)$ holds by the definition of $\bl$. If round $\lmle$ is never reached, then the upper bound above still holds, as we can still bound $\calR_T \leq \calR_{1:\lmle}$, the regret we would have incurred had we reached $\lmle$. 

Finally, by Theorem \ref{thm:comp_complex} we can choose $\nu = 4$, $\zeta = 2$, and we will be able to compute the solution efficiently.
\end{proof}

\begin{proof}[Proof of Theorem \ref{thm:inefficient_regret_bound}]
The proof of this result is very similar to the proof of Theorem \ref{thm:efficient_regret_bound} but we include the points where it differs for the sake of completeness. Unless otherwise noted, all notation is defined as in the proof of Theorem \ref{thm:efficient_regret_bound}.

\textbf{Good event:} Define the events:
$$ \calE_{k,j} = \left \{ \sup_{z,z' \in \calS_j} |(z-z')^\top (\hat{\theta}_k - \thetast)| \leq  \E_\eta \left [ \sup_{z,z' \in \calS_j} (z-z')^\top A(\tau_k)^{-1/2} \eta \right ] + \sqrt{2 \max_{z,z' \in \calS_j} \| z - z' \|_{A(\tau_k)^{-1}}^2 \log(1/\delta_k)}\right \} $$
$$ \calE = \bigcap_{k=1}^\infty \bigcap_{j=0}^k \calE_{k,j} $$
Proposition \ref{prop:tis} implies that $\P[\calE_{k,j}^c] \leq \delta/k^3$ so:
$$ \P[\calE^c]  \leq \sum_{k=1}^\infty \sum_{j=0}^k \P[\calE_{k,j}^c] \leq  \sum_{k=1}^\infty \sum_{j=0}^k \frac{\delta}{k^3} \leq 3\delta$$

\textbf{Estimation error:} Henceforth we assume $\calE$ holds. We proceed by induction to show that the gaps are always well-estimated. First we prove the base case. Let $k = 1$ and consider any $x \in \X$. Then:
\begin{align*}
|(\xst - x)^\top (\hat{\theta}_1 - \thetast)| & \leq \sup_{z,z' \in \X} | (z - z')^\top (\hat{\theta}_1 - \thetast)| \\
& \leq \E_\eta \left [ \sup_{z,z' \in \X} (z-z')^\top A(\tau_1)^{-1/2} \eta \right ] + \sqrt{2 \max_{z,z' \in \X} \| z - z' \|_{A(\tau_1)^{-1}}^2 \log(1/\delta_k)} \\
& \overset{(a)}{=} \E_\eta \left [ \sup_{z \in \X} (x_1 - z)^\top A(\tau_1)^{-1/2} \eta \right ] + \sqrt{2 \max_{z,z' \in \X} \| z - z' \|_{A(\tau_1)^{-1}}^2 \log(1/\delta_k)} \\
& \overset{(b)}{\leq}  \epsilon_1/8
\end{align*}
where $(a)$ follows by Proposition 7.5.2 of \cite{vershynin2018high} and $(b)$ follows since $\tau_1$ is a feasible solution to (\ref{eq:vl_inefficient}). For the inductive step, assume that, for all $x \in \calS_k$:
$$ |(\xst - x)^\top (\hat{\theta}_k - \thetast)| \leq \epsilon_k/8 $$
and for all $x \in \calS_k^c$:
$$ |(\xst - x)^\top (\hat{\theta}_k - \thetast)| \leq \Delta_x/8 $$
Consider round $k+1$ and take $x \in \calS_{k+1}^c$. There then exists some $k' \leq k$ such that $x \in \calS_{k'} \backslash \calS_{k'+1}$. Then:
\begin{align*}
\frac{|(\xst - x)^\top (\hat{\theta}_{k+1} - \thetast)|}{ \Delta_x} & \leq \sup_{z,z' \in \calS_{k'}} \frac{|(z - z')^\top (\hat{\theta}_{k+1} - \thetast)|}{ \Delta_x} \\
& \leq \E_\eta \left [ \sup_{z, z' \in \calS_{k'}} (z' - z)^\top A(\tau_{k+1})^{-1/2} \eta \right ] + \sqrt{2 \max_{z,z' \in \calS_{k'}} \| z - z' \|_{A(\tau_{k+1})^{-1}}^2 \log(1/\delta_{k+1})} \\
& \overset{(a)}{=}  2 \E_\eta \left [ \sup_{z \in \calS_{k'}} \frac{(x_{k+1} - z)^\top A(\tau_{k+1})^{-1/2} \eta}{\Delta_x} \right ] + \sqrt{8 \max_{z \in \calS_{k'}} \frac{\| z  \|_{A(\tau_{k+1})^{-1}}^2}{\Delta_x^2} \log(1/\delta_{k+1})} \\
& \overset{(b)}{\leq} 4 \E_\eta \left [ \sup_{z \in \calS_{k'}} \frac{(x_{k+1} - z)^\top A(\tau_{k+1})^{-1/2} \eta}{\epsilon_{k+1} + \Delta_x} \right ] + \sqrt{32 \max_{z \in \calS_{k'}} \frac{\| z \|_{A(\tau_{k+1})^{-1}}^2}{(\epsilon_{k+1} + \Delta_x)^2} \log(1/\delta_{k+1})}  \\
& \overset{(c)}{\leq} 8 \E_\eta \left [ \sup_{z \in \calS_{k'}} \frac{(x_{k+1} - z)^\top A(\tau_{k+1})^{-1/2} \eta}{\epsilon_{k+1} + \Delta_z} \right ] + \sqrt{128 \max_{z \in \calS_{k'}} \frac{\| z  \|_{A(\tau_{k+1})^{-1}}^2}{(\epsilon_{k+1} + \Delta_z)^2} \log(1/\delta_{k+1})} \\
& \leq 8 \E_\eta \left [ \sup_{z \in \X} \frac{(x_{k+1} - z)^\top A(\tau_{k+1})^{-1/2} \eta}{\epsilon_{k+1} + \Delta_z} \right ] + \sqrt{128 \max_{z \in \X} \frac{\| z  \|_{A(\tau_{k+1})^{-1}}^2}{(\epsilon_{k+1} + \Delta_z)^2} \log(1/\delta_{k+1})} \\
& \overset{(d)}{\leq} 16 \E_\eta \left [ \sup_{z \in \X} \frac{(x_{k+1} - z)^\top A(\tau_{k+1})^{-1/2} \eta}{\epsilon_{k+1} + \Delta_z} \right ] + \sqrt{512 \max_{z \in \X} \frac{\| z  \|_{A(\tau_{k+1})^{-1}}^2}{(\epsilon_{k+1} + \Delta_z)^2} \log(1/\delta_{k+1})} \\
& \overset{(e)}{\leq} 1/8
\end{align*}
where $(a)$ follows by Proposition 7.5.2 of \cite{vershynin2018high}, $(b)$ follows since $\Delta_x \geq \epsilon_{k+1}$ by virtue of the fact that $x \in \calS_{k+1}^c$, so $\Delta_x \geq (\epsilon_{k+1} + \Delta_x)/2$, $(c)$ follows since $\Delta_x \in [\epsilon_{k'+1},\epsilon_{k'}]$ and for any $z \in \calS_{k'}$, we will have theta $\Delta_z \leq \epsilon_{k'}$, so $\epsilon_{k+1} + \Delta_x \geq \epsilon_{k+1} + \epsilon_{k'+1} \geq \epsilon_{k+1} + \Delta_z/2$, $(d)$ holds by the inductive hypothesis and Lemma 1 of \cite{katz2020empirical} and taking $\hat{\Delta}_z$ to be the estimate of $\Delta_z$ at round $k+1$, and $(e)$ holds since $\tau_{k+1}$ is a feasible solution to (\ref{eq:vl}). We can perform a similar calculation to get the same thing for $x \in \calS_{k+1}$, allowing us to conclude that, for all $x \in \calS_{k+1}$:
$$ |(\xst - x)^\top (\hat{\theta}_{k+1} - \thetast)| \leq \epsilon_{k+1}/8 $$
and for all $x \in \calS_{k+1}^c$:
$$ |(\xst - x)^\top (\hat{\theta}_{k+1} - \thetast)| \leq \Delta_x/8 $$
From here the remaining calculations on the gap estimates performed in the proof of Theorem \ref{thm:efficient_regret_bound} hold almost identically.

\textbf{Bounding the Round Regret:} 
From (\ref{eq:est_to_real_regret}), it follows that the constraint in (\ref{eq:vl_inefficient}) is tighter than the following constraint:
\begin{equation}
 \mathbb{E}_\eta \left [ \max_{x \in \X} \frac{(x_\ell - x)^\top A(\tau)^{-1/2} \eta}{\epsilon_\ell + \Delta_x} \right ] + \sqrt{2 \max_{x \in \X} \frac{\| x \|_{A(\tau)^{-1}}^2}{(\epsilon_\ell + \Delta_x)^2} \log(2\ell^3/\delta)} \leq \frac{1}{256}
\end{equation}
so any $\tau$ satisfying this inequality is also a feasible solution to (\ref{eq:vl_inefficient}).

From here we follow the same pattern as in the proof of Theorem \ref{thm:efficient_regret_bound}. We handle each term in the constraint separately. For the second term, note that we can upper bound:
\begin{align*}
\max_{x \in \X} \frac{\| x \|_{A(\tau)^{-1}}^2}{(\epsilon_\ell + \Delta_x)^2} & \leq \max \left \{ \max_{x \in \calS_\ell} \frac{\| x \|_{A(\tau)^{-1}}^2}{(\epsilon_\ell + \Delta_x)^2} , \max_{j < \ell} \max_{x \in \calS_j \backslash \calS_{j+1}} \frac{\| x \|_{A(\tau)^{-1}}^2}{(\epsilon_\ell + \Delta_x)^2} \right \} \\
& \leq 2\max \left \{ \epsilon_\ell^{-2} \max_{x \in \calS_\ell} \| x \|_{A(\tau)^{-1}}^2 , \max_{j < \ell} \epsilon_j^{-2} \max_{x \in \calS_j \backslash \calS_{j+1}} \| x \|_{A(\tau)^{-1}}^2 \right \} \\
& \leq 2\max \left \{ \epsilon_\ell^{-2} \max_{x \in \calS_\ell} \| x \|_{A(\tau)^{-1}}^2 , \max_{j < \ell} \epsilon_j^{-2} \max_{x \in \calS_j } \| x \|_{A(\tau)^{-1}}^2 \right \} \\
& \leq 2 \max_{j \leq \ell} \epsilon_j^{-2} \max_{x \in \calS_j} \| x \|_{A(\tau)^{-1}}^2
\end{align*}
We now choose $\bar{\tau} = \ell^2 \sum_{j=1}^\ell \tau_j(1) + 1048576 d \log(2\ell^3/\delta) \sum_{j=1}^\ell \epsilon_j^{-2} \lambda_{j}^{\mathrm{kf}}$, where $\lambda_{j}^{\mathrm{kf}}$ is the distribution minimizing $\max_{x \in \calS_j} \| x \|_{A(\lambda)^{-1}}^2$. By the same argument as in the proof of Theorem \ref{thm:efficient_regret_bound}, effectively ignoring the second term, we will have:
$$ \mathbb{E}_\eta \left [ \max_{x \in \X} \frac{(x_\ell - x)^\top A(\bar{\tau})^{-1/2} \eta}{\epsilon_\ell + \Delta_x} \right ] \leq \frac{1}{512} $$
For the second term, by the Kiefer-Wolfowitz Theorem in the bandit case, and Proposition \ref{prop:semi_kf} in the semi-bandit case, we'll have:
\begin{align*}
\max_{j \leq \ell} \epsilon_j^{-2} \max_{x \in \calS_j} \| x \|_{A(\bar{\tau})^{-1}}^2 & \leq \max_{j \leq \ell} \epsilon_j^{-2} \max_{x \in \calS_j} \| x \|_{A(c d \log(2\ell^3/\delta) \epsilon_j^{-2} \lambda_j^{\mathrm{kf}})^{-1}}^2 \\
& \leq \frac{1}{c d \log(2\ell^3/\delta)} \max_{j \leq \ell} \max_{x \in \calS_j}  \| x \|_{A(\lambda_j^{\mathrm{kf}})^{-1}}^2 \\
& \leq \frac{1}{1048576 \log(2\ell^3/\delta)}
\end{align*}
So:
$$  \sqrt{2 \max_{x \in \X} \frac{\| x \|_{A(\bar{\tau})^{-1}}^2}{(\epsilon_\ell + \Delta_x)^2} \log(2\ell^3/\delta)} \leq \frac{1}{512} $$
From this it follows $\bar{\tau}$ is a feasible solution to (\ref{eq:vl_inefficient}). Furthermore, by Lemma \ref{lem:deterministic_ae2}, the total regret incurred by playing $\ell^2 \sum_{j=1}^\ell \tau_j(1)$ is bounded by:
$$  c_1 \Delmax \ell d + \frac{c_2 \ell^2 \gamworst(A)}{\epsilon_\ell} $$
and the total regret incurred playing $c d \log(2\ell^3/\delta) \sum_{j=1}^\ell \epsilon_j^{-2} \lambda_{j}^{\mathrm{kf}}$ is bounded as:
$$ c_1 \Delmax \ell d + \frac{c_2  d \log(2 \ell^3/\delta)}{\epsilon_\ell}$$
Following the same argument as in Theorem \ref{thm:efficient_regret_bound}, it follows that:
$$ \calR_\ell \leq c_1 \Delmax \ell d + \frac{c_2 (\ell^2 \gamworst(A) + d\log(2 \ell^3/\delta))}{\epsilon_\ell} $$
From here the argument follows identically to the proof of Theorem \ref{thm:comp_complex}, so we omit the remainder of the proof.
\end{proof}

\begin{lemma}\label{lem:deterministic_ae2}
Given an $\ell$ such that $\epsilon_\ell > \Delmin$, let $\lambda_k$ be any distribution supported on $\calS_k$ and for any $\xi$ set:
$$ \tau_k =  \xi \epsilon_k^{-2} $$
Play the distributions $\kappa_k \leftarrow$ {\normalfont ROUND}$(\lambda_k,\lceil \tau_k \rceil \vee q(1/2), 1/2)$ for $k=1,\ldots,\ell$, where ROUND is defined as in Section \ref{sec:aegw}. Then the total gap-dependent regret incurred by this procedure is bounded by:
$$  c_1 \Delmax  \ell  d + \frac{c_2 \xi}{\epsilon_\ell} $$
\end{lemma}
\begin{proof}
We can think of this procedure as a deterministic variant of action elimination. We can bound the regret incurred as:
\begin{align*}
 \sum_{x \in \mathcal{X} \backslash \{ x^* \}} \Delta_x T_x  & \leq \sum_{k=1}^{\ell} \epsilon_k (\tau_k + q(1/2) + 1) \\
 & \leq  \Delmax \ell (q(1/2) + 1) + \sum_{k=1}^{\ell} \epsilon_k  \tau_k \\
  & \leq   \Delmax \ell (q(1/2) + 1)  +  \xi \sum_{k=1}^{\ell} \epsilon_k^{-1} \\
 & \leq   \Delmax \ell (q(1/2) + 1) + \xi \sum_{k=1}^{\ell} \frac{2^k}{\Delmax} \\
 & \leq \Delmax \ell (q(1/2) + 1) + \frac{c \xi }{ \epsilon_\ell}
\end{align*}
The results on the rounding procedure follow from \cite{katz2020empirical,allen2020near}. 
\end{proof}

\begin{lemma}\label{lem:num_round_bound}
Given a $T$, Algorithm \ref{alg:gw_ae_comp} will run for at most:
$$\ellmax(T) := \log_2 \left (   \frac{\max_{x \in \X} \| x \|_2}{\min_{x \in \X} \| x \|_2} \left (  \mathrm{diam}(\X) \| \theta \|_2 \sqrt{T} + 3 \right ) \right ) + 1 $$
rounds. Furthermore, regardless of $T$, Algorithm \ref{alg:gw_ae_comp} will run for at most:
$$ \lceil \log_2(4\Delmax/\Delmin) \rceil $$
rounds.
\end{lemma}
\begin{proof}
Note that $\tau_\ell$ must satisfy:
$$ \mathbb{E}_\eta \left [ \max_{x \in \X} \frac{(x_\ell - x)^\top A(\tau_\ell)^{-1/2} \eta}{\epsilon_\ell + \hat{\Delta}_x} \right ] \leq \frac{1}{128 (1 + \sqrt{\pi \log(2\ell^3/\delta)})} $$
However:
\begin{align*}
\mathbb{E}_\eta \left [ \max_{x \in \X} \frac{(x_\ell - x)^\top A(\tau_\ell)^{-1/2} \eta}{\epsilon_\ell + \hat{\Delta}_x} \right ] & \overset{(a)}{\geq} \frac{1}{\sqrt{2\pi}} \max_{x,y \in \X} \left \| A(\tau_\ell)^{-1/2} \left ( \frac{x}{\epsilon_\ell + \hat{\Delta}_x} - \frac{y}{\epsilon_\ell + \hat{\Delta}_y} \right ) \right \|_2 \\
& \geq \frac{1}{\sqrt{2\pi}} \left \| A(\tau_\ell)^{-1/2} \left ( \frac{x^*}{\epsilon_\ell + \hat{\Delta}_{x^*}} - \frac{\xmax}{\epsilon_\ell + \hat{\Delta}_{\xmax}} \right ) \right \|_2 \\
& \overset{(b)}{\geq} \frac{1}{\sqrt{2\pi \tau_\ell}} \frac{1}{\max_{x \in \X} \| x \|_2} \left \|  \frac{x^*}{\epsilon_\ell + \hat{\Delta}_{x^*}} - \frac{\xmax}{\epsilon_\ell + \hat{\Delta}_{\xmax}}  \right \|_2 \\ 
& \geq  \frac{1}{\sqrt{2\pi \tau_\ell}} \frac{1}{\max_{x \in \X} \| x \|_2} \left ( \frac{\| x^* \|_2}{\epsilon_\ell + \hat{\Delta}_{x^*}} - \frac{\| \xmax \|_2}{\epsilon_\ell + \hat{\Delta}_{\xmax}} \right ) \\
& \overset{(c)}{\geq} \frac{1}{\sqrt{2\pi \tau_\ell}} \frac{1}{\max_{x \in \X} \| x \|_2} \left ( \frac{2\| x^* \|_2}{3\epsilon_\ell} - \frac{2\| \xmax \|_2}{\Delmax} \right ) \\
& \geq \frac{2}{3\sqrt{2\pi \tau_\ell}} \left ( \frac{\min_{x \in \X} \| x \|_2}{\max_{x \in \X} \| x \|_2}  \frac{1}{\epsilon_\ell} - \frac{3}{\Delmax} \right )
\end{align*}
where $(a)$ follows by Proposition 7.5.2 of \cite{vershynin2018high}, $(b)$ follows since for any $\lambda$:
$$ A(\lambda) \preceq (\max_{x \in \X} \| x \|_2^2) I $$
and $(c)$ follows by (\ref{eq:est_to_real_regret}). Thus:
\begin{align*}
\tau_\ell & \geq \frac{ 4 (128 (1 + \sqrt{\pi \log(2\ell^3/\delta)}))^2}{18 \pi} \left ( \frac{\min_{x \in \X} \| x \|_2}{\max_{x \in \X} \| x \|_2}  \frac{1}{\epsilon_\ell} - \frac{3}{\Delmax} \right )^2 \\
& \geq \left ( \frac{\min_{x \in \X} \| x \|_2}{\max_{x \in \X} \| x \|_2}  \frac{1}{\epsilon_\ell} - \frac{3}{\Delmax} \right )^2 \\
& = \frac{1}{\Delmax^2} \left ( \frac{\min_{x \in \X} \| x \|_2}{\max_{x \in \X} \| x \|_2}  2^\ell - 3 \right )^2
\end{align*}
where the final equality holds since $\epsilon_\ell = \Delmax 2^{-\ell}$. If round $\ell$ is the last round the algorithm completes before terminating, we'll have that $T \geq \tau_\ell$, so:
$$ T \geq \frac{1}{\Delmax^2} \left ( \frac{\min_{x \in \X} \| x \|_2}{\max_{x \in \X} \| x \|_2}  2^\ell - 3 \right )^2 \implies \log_2 \left (  \frac{\max_{x \in \X} \| x \|_2}{\min_{x \in \X} \| x \|_2} \left ( \Delmax \sqrt{T} + 3 \right ) \right ) \geq \ell $$
The first conclusion follows by Lemma \ref{lem:delmax_bound}.

For the second conclusion, note that, as we showed above, on the good event we will have that for all $x \in \hat{\X}_\ell$, $\Delta_x \leq 2 \epsilon_\ell$. Thus, once $\epsilon_\ell \leq \Delmin / 4$, we can guarantee that for any $x \in \hat{\X}_\ell$, $\Delta_x \leq \Delmin/2$ which implies that $x$ is the optimal arm so $|\hat{\X}_\ell|=1$ and the algorithm will have terminated. It follows that:
$$ \epsilon_\ell = \Delmax 2^{-\ell} \leq \Delmin / 4 \implies \ell \leq \log_2(4\Delmax/\Delmin)$$
\end{proof}

\begin{lemma}\label{lem:delmax_bound}
$$ \Delmax \leq \| \theta \|_2 \mathrm{diam}(\X) $$
\end{lemma}
\begin{proof}
$$ \Delmax = \langle \theta, x^* - \xmax \rangle \leq  \| \theta \|_2 \max_{x,y \in \X} \| x - y \|_2 $$
\end{proof}

\section{Pure Exploration Proofs}\label{sec:pure_exploration_proof}

For the sake of clarity, we rewrite the pure exploration algorithm (see Algorithm \ref{alg:gw_pure_exp_comp}).
\begin{algorithm}[H] 
\begin{algorithmic}[1]
\State \textbf{Input:} Set of arms $\mathcal{X}$,  largest gap $\Delmax$,  confidence $\delta$, total time $T$
\State $\hat{\mathcal{X}}_1 = \X, \hat{\theta}_0 = 0, \ell \leftarrow 1$
\While{$|\hat{\mathcal{X}}_\ell| > 1$ and total pulls less than $T$}
 	 \State $x_\ell \leftarrow \argmax_{x \in \X} x^\top \hat{\theta}_{\ell-1}$, $\epsilon_\ell \leftarrow \Delmax 2^{-\ell}$
	
	  \State Let $\tau_\ell$ be a solution to:
 \begin{align}\label{eq:pure_exp_opt}
  \begin{split}
 & \argmin_{\tau}  \ \sum_{x \in \X}  \tau_x  \\
& \text{ s.t. } \mathbb{E}_\eta \left [ \max_{x \in \X} \frac{(x_\ell - x)^\top A(\tau)^{-1/2} \eta}{\epsilon_\ell + \hat{\Delta}_x} \right ] \leq \frac{1}{128 (1 + \sqrt{\pi \log(2\ell^3/\delta)})} 
\end{split}
\end{align} 
	\State $\alpha_\ell \leftarrow$ SPARSE$(\tau_\ell,\nf)$
	 \State Pull arm $x$ $\alpha_{\ell,x}$ times, compute $\hat{\theta}_\ell$
	\If{MINGAP$(\widehat{\theta}_{\ell}, \X) \geq 3\epsilon_{\ell}/2 $}
		\State \textbf{break} \label{line:break_minimax}
	\EndIf
	 \State Pull arm $x$ $\lceil \tau_{\ell,x} \rceil$ times, compute $\hat{\theta}_\ell$ from this data, form gap estimates $\hat{\Delta}_x$ from $\hat{\theta}_\ell$

	 \State $\ell \leftarrow \ell + 1$
\EndWhile
\State \Return $ \argmax_{x \in \X} x^\top \hat{\theta}_{\ell}$ 
\end{algorithmic}
\caption{Computationally Efficient Pure Exploration Algorithm Semi-Bandit Feedback}
\label{alg:gw_pure_exp_comp}
\end{algorithm}
Theorem \eqref{thm:efficient_regret_bound} shows that we can solve \eqref{eq:pure_exp_opt} in polynomial-time, but note  that it is easier to solve \eqref{eq:pure_exp_opt} approximately by calling stochastic Frank-Wolfe to solve
\begin{align*}
\inf_{\lambda \in \triangle} \mathbb{E}_\eta \left [ \max_{x \in \X} \frac{(x_\ell - x)^\top A(\tau)^{-1/2} \eta}{\epsilon_\ell + \hat{\Delta}_x} \right ] 
\end{align*}
and the convergence rate shown in Lemma \ref{lem:sfw_convergence} applies.

The MINGAP subroutine (Algorithm \ref{alg:unique}), originally provided in \cite{chen2017nearly}, is a computationally scalable method to compute the empirical gap between the empirically best arm and the empirically second best arm. It uses at most $d$ calls to the linear maximization oracle.
\begin{algorithm}
\begin{algorithmic}[1]
\State \textbf{Input:} $\X$, estimate $\tilde{\theta}$
\State $\tilde{x} \longleftarrow  \argmax_{x \in \X} \tilde{\theta}^\t x$
\State $\widehat{\Delta}_{min} \longleftarrow \infty$
\For{$i=1,2,\ldots,d$ s.t. $i \in \tilde{x}$}
\State \begin{align*}
\tilde{\theta}^{(i)} = \begin{cases} 
      \tilde{\theta}_j  &  j \neq i \\
      -\infty& j = i 
   \end{cases}
\end{align*}
\State $\tilde{x}^{(i)} \longleftarrow \argmax_{x \in \X} x^\t \tilde{\theta}^{(i)}$\;
\If{$\tilde{\theta}^\t(\tilde{x}-\tilde{x}^{(i)})  \leq \widehat{\Delta}_{min}$}
\State $\widehat{\Delta}_{min} \longleftarrow \tilde{\theta}^\t(\tilde{x}-\tilde{x}^{(i)}) $
\EndIf
\EndFor
\State  \Return $\widehat{\Delta}_{min}$
\end{algorithmic}
\caption{MINGAP}
\label{alg:unique}
\end{algorithm}

We note that the correctness and sample complexity proofs are quite similar to the proof of Theorem in \cite{katz2020empirical}, but we include it for the sake of completeness. The main contribution of our paper for the pure exploration problem is a computational method to solve \eqref{eq:pure_exp_opt} even when the number of variables $|\X|$ is exponential in the dimension.

\begin{proof}[Proof of Theorem \ref{thm:pure_exploration}]

\textbf{Step 1: A good event and well-estimated gaps}
Using the identical argument to the first two steps of the proof of Theorem \ref{thm:efficient_regret_bound}, we have that with probability at least $1-\delta$ at every round $k$, for all $x \in \calS_{k}$:
\begin{align}
|(\xst - x)^\top (\hat{\theta}_{k+1} - \thetast)| \leq \epsilon_{k}/8 \label{eq:pure_exp_gap_est_1}
\end{align}
and for all $x \in \calS_{k+1}^c$:
\begin{align}
 |(\xst - x)^\top (\hat{\theta}_{k+1} - \thetast)| \leq \Delta_x/8. \label{eq:pure_exp_gap_est_2}
\end{align} 
For the remainder of the proof we suppose that this good event holds.

\textbf{Step 2: Correctness.} It is enough to show at round $k$, if $x_k \neq x_*$, then the \textsc{Unique}$(\X, \widehat{\theta}_k, \epsilon_k)$ returns false. Inspecting \textsc{Unique}, a sufficient condition is to show that  $(x_k -\xst)^\t\widehat{\theta}_k - \epsilon_k \leq 0$. By \eqref{eq:pure_exp_gap_est_1} and \eqref{eq:pure_exp_gap_est_2}, we have that
\begin{align*}
(x_k -\xst)^\t\widehat{\theta}_k - \epsilon_k & = (x_k -\xst)^\t(\widehat{\theta}_k-\theta) - \Delta_{x_k} - \epsilon_k \\
& \leq \max(\frac{\Delta_{x_k}}{8},\frac{\epsilon_k}{8})  - \Delta_{x_k} - \epsilon_k \\
& \leq 0
\end{align*}
proving correctness.

\textbf{Step 3: Bound the Sample Complexity.} Letting $\tilde{x}_k = \argmax_{x \neq x_k} \widehat{\theta}_k^\t x$,  \textsc{Unique}$(\Z, \widehat{\theta}_k, \epsilon_k)$ at round $k$ checks whether $\widehat{\theta}_k^\t(x_k - \tilde{x}_k)$ is at least $\epsilon_k$, and terminates if it is. Thus, \eqref{eq:pure_exp_gap_est_1} and \eqref{eq:pure_exp_gap_est_2}, the algorithm terminates and outputs $\xst$ once $k \geq c \log(\Delmax /\Delmin)$. 

Thus, the sample complexity is upper bounded by
\begin{align}
\sum_{k=1}^{c \log(\Delmax/\Delmin)}  \sum_{x \in \X} \ceil{\alpha_{k,x}} \leq c^\prime [ \log(\Delmax/\Delmin)d + \sum_{k=1}^{c \log(\Delmax/\Delmin)} \inf_{\lambda \in \triangle} \E_{\eta \sim N(0,I)}[ \max_{x \in \X} \frac{(x_k-x)^\t \Asb(\lambda)^{-1/2} \eta}{2^{-k} \Gamma + \widehat{\theta}^\t_{k}(x_k - x) }]^2 ] \label{eq:comp_fc_ub_main}
\end{align}
where we used the fact that the rounding procedure can use $O(d)$ points in the semi-bandit case. Thus, it suffices to upper bound the second term in the above expression.
Fix $\lambda \in \triangle$. Then,
\begin{align*}
\E_{\eta \sim N(0,I)}[ \max_{x \in \X} \frac{(x_k-x)^\t \Asb(\lambda)^{-1/2} \eta}{\epsilon_k + \widehat{\theta}^\t_{k}(x_k - x) }]^2 & \leq c \E_{\eta \sim N(0,I)}[ \max_{x \in \X} \frac{(x_k-x)^\t \Asb(\lambda)^{-1/2} \eta}{\epsilon_k + \Delta_x }]^2 \\
& \leq c^\prime [\E_{\eta \sim N(0,I)}[ \max_{x \in \X} \frac{(\xst -x)^\t \Asb(\lambda)^{-1/2} \eta}{\epsilon_k + \Delta_x }]^2 \\
& + \E_{\eta \sim N(0,I)}[ \max_{x \in \X} \frac{(z_*-x_k)^\t \Asb(\lambda)^{-1/2} \eta}{\epsilon_k + \Delta_x }]^2] 
\end{align*}
Fix $x_0 \in \X \setminus \{\xst \}$. The first term is bounded as follows.  
\begin{align}
\E_{\eta \sim N(0,I)}[& \max_{x \in \X} \frac{(\xst-x)^\t \Asb(\lambda)^{-1/2} \eta}{\epsilon_k + \Delta_x }]^2 \nonumber \\
& = \E_{\eta \sim N(0,I)}[ \max_{x \in \X \setminus \{\xst \}} \max(\frac{(\xst-x)^\t \Asb(\lambda)^{-1/2} \eta}{\epsilon_k + \Delta_x }, 0)]^2 \nonumber \\
& \leq 8\E_{\eta \sim N(0,I)}[ \max_{x \in \X \setminus \{\xst \}} \frac{(\xst-x)^\t \Asb(\lambda)^{-1/2} \eta}{\epsilon_k + \Delta_x }]^2 +  8\frac{\norm{\xst- x_0}_{\Asb(\lambda)^{-1}}^2}{\epsilon_k + \Delta_{x_0})^2 } \label{eq:comp_ub_samp_comp_2} \\
& \leq 8[\E_{\eta \sim N(0,I)}[ \max_{x \in \X \setminus \{\xst \}} \frac{(\xst-x)^\t \Asb(\lambda)^{-1/2} \eta}{ \Delta_x }]^2 \nonumber \\
& +  \max_{x \neq \xst} \frac{\norm{\xst- x}_{\Asb(\lambda)^{-1}}^2}{ \Delta_x^2 }] \label{eq:comp_ub_final_1}
\end{align}
where we obtained line \eqref{eq:comp_ub_samp_comp_2} using exercise 7.6.9 in \cite{vershynin2018high}. 

We also have that
\begin{align}
\E_{\eta \sim N(0,I)}[ \max_{x \in \X} \frac{(\xst-x_k)^\t \Asb(\lambda)^{-1/2} \eta}{\epsilon_k + \Delta_x }]^2 & \leq \E_{\eta \sim N(0,I)}[ \max(\frac{(\xst-x_k)^\t \Asb(\lambda)^{-1/2} \eta}{\epsilon_k  },0)]^2 \nonumber \\
& \leq c \frac{\norm{\xst-x_k}_{\Asb(\lambda)^{-1}}^2}{\epsilon_k^2} \nonumber \\
& \leq c \frac{\norm{\xst-x_k}_{\Asb(\lambda)^{-1}}^2}{\Delta_{x_k}^2} \label{eq:comp_ub_samp_comp_3} \\
& \leq c \max_{x \in \X \setminus \{\xst\}}  \frac{\norm{\xst-x}_{\Asb(\lambda)^{-1}}^2}{\Delta_{x}^2} \label{eq:comp_ub_final_2}
\end{align}
where line \eqref{eq:comp_ub_samp_comp_3} follows since \eqref{eq:pure_exp_gap_est_1}, \eqref{eq:pure_exp_gap_est_2}, and Lemma 1 in \cite{katz2020empirical} imply that $x_k \in S_{k+2}$. 

\eqref{eq:comp_fc_ub_main}, \eqref{eq:comp_ub_final_1}, and \eqref{eq:comp_ub_final_2} together imply that
\begin{align*}
\sum_{k=1}^{c \log(\Gamma/\Delta_{min})} \sum_{x \in \X} \ceil{\alpha_{k,x}} \leq c \log(\Delmin/\Delmin)[d + \gamma^* + \rho^*],
\end{align*}
completing the proof.

\end{proof}

\subsection{Lower Bound}

In this section, we prove a lower bound for the combinatorial bandit setting with semi-bandit feedback. Fix a model $\theta$ and let $\nu_{\theta,i}$ denote the distribution of the observations when arm $i$ is pulled. In this setting, at each round $t$, $Z^{(t)} \sim N(\theta, I)$ is drawn and
\begin{align*}
(\nu_{\theta,i})_j = \begin{cases}
Z_j^{(t)} & j \in x_i \\
0 & j \not \in x_i
\end{cases}.
\end{align*} 

\begin{definition}
We say that an Algorithm is $\delta$-PAC if for any instance $(\X, \theta_*)$, it returns $x \in \X$ with the largest mean with probability at least $1-\delta$. 
\end{definition}

\begin{theorem}\label{thm:lb_semi_bai}
Fix an instance $(\theta_*, \X)$ such that $\X \subset \{0,1\}^d$ and $\xst = \argmax_{x \in \X} x^\t \theta$ is unique. 
Let $\A$ be a $\delta$-PAC algorithm and let $T$ be its total number of pulls on $(\theta_*, \X)$. Then, 
\begin{align*}
\E_{\theta_*}[T] \geq \log(1/2.4 \delta)\rho^* := \log(1/2.4 \delta)\inf_{\lambda \in \simp} \sup_{x \in \X \setminus \{x_* \}} \frac{\norm{x_*-x}^2_{\Asb(\lambda)^{-1}}}{\theta^\t (\xst-x )^2}.
\end{align*}
\end{theorem}

The proof is quite similar to the proof of Theorem 1 in \cite{fiez2019sequential}. 

\begin{proof}
For simplicity, label $\X = \{x_1, \ldots, x_m \}$ and $\xst = x_1$. Define the set of alternative instances $\mc{O} = \{\theta : \argmax_{x \in \X} x^\t \theta \neq x_1 \}$. 
Let $T_i$ denote the random number of times that $x_i$ is pulled during the game. Then, noting that the standard transportation Lemma from \cite{kaufmann2016complexity} easily generalizes to semi-bandit feedback, we have that for any $\theta \in \mc{O}$,
\begin{align*}
\sum_{i=1}^n \E_{\theta_*}[T_i] \kl( \nu_{\theta_*,i} | \nu_{\theta,i}) \geq \ln(1/2.4 \delta)
\end{align*}
By a standard argument (see for example Theorem 1 \cite{fiez2019sequential}), this implies that
\begin{align*}
\E_{\theta_*}[T] \geq \ln(1/2.4 \delta) \min_{\lambda \in \simp} \max_{\theta \in \mc{O}} \frac{1}{\sum_{i=1}^m \lambda_i \kl( \nu_{\theta_*,i} | \nu_{\theta,i})}.
\end{align*}
Let $\epsilon > 0$. For each $k \neq 1$, define
\begin{align*}
\theta^{(k)} = \theta_* - \frac{[(x_1-x_k)^\t \theta_* + \epsilon] \Asb(\lambda)^{-1} (x_1 -x_k)}{(x_1 -x_k)^\t \Asb(\lambda)^{-1} (x_1 -x_k)}.
\end{align*}
Note that
\begin{align*}
(x_k -x_1)^\t \theta^{(k)} = \epsilon
\end{align*}
showing that $\theta^{(k)} \in \mc{O}$. Note that using the identity for the KL-divergence for a multivariate Gaussian, we have that
\begin{align*}
\kl( \nu_{\theta_*,i} | \nu_{\theta^{(k)},i}) & = \frac{1}{2} \sum_{j \in x_i} (e_j^\t (\theta_* - \theta^{(k)})^2 \\
& = \frac{1}{2} (x_k^\t \theta_* + \epsilon)^2 \sum_{j \in x_i} \frac{(x_1 - x_k)^\t \Asb(\lambda)^{-1} e_j e_j^\t \Asb(\lambda)^{-1} (x_1-x_k)}{[(x_1-x_k)^\t \Asb(\lambda)^{-1} (x_1 - x_k)]^2}.
\end{align*} 
Then, we have that
\begin{align*}
\E_{\theta_*}[T] & \geq \ln(1/2.4 \delta) \min_{\lambda \in \simp} \max_{\theta \in \mc{O}} \frac{1}{\sum_{i=1}^m \lambda_i \kl( \nu_{\theta_*,i} | \nu_{\theta,i})} \\
& \geq \ln(1/2.4 \delta) \min_{\lambda \in \simp} \max_{k \neq 1} \frac{1}{\sum_{i=1}^m \lambda_i \kl( \nu_{\theta_*,i} | \nu_{\theta^{(k)},i})} \\
& =2\ln(1/2.4 \delta) \min_{\lambda \in \simp} \max_{k \neq 1} \frac{\norm{ x_1 - x_k}^4_{\Asb(\lambda)^{-1}}  }{(x_k^\t \theta_* + \epsilon)^2 \sum_{i=1}^m \lambda_i \sum_{j \in x_i} (x_1 - x_k)^\t \Asb(\lambda)^{-1} e_j e_j^\t \Asb(\lambda)^{-1} (x_1-x_k)} \\
& =2\ln(1/2.4 \delta) \min_{\lambda \in \simp} \max_{k \neq 1} \frac{\norm{ x_1 - x_k}^4_{\Asb(\lambda)^{-1}} }{(x_k^\t \theta_* + \epsilon)^2  (x_1 - x_k)^\t \Asb(\lambda)^{-1} \Asb(\lambda) \Asb(\lambda)^{-1} (x_1-x_k)} \\
& =2\ln(1/2.4 \delta) \min_{\lambda \in \simp} \max_{k \neq 1} \frac{\norm{ x_1 - x_k}^2_{\Asb(\lambda)^{-1}} }{(x_k^\t \theta_* + \epsilon)^2  } .
\end{align*}
Since $\epsilon > 0$ was arbitrary, we may let $\epsilon \longrightarrow 0$, obtaining the result.
\end{proof}

Next, we state and prove a lower bound for the \emph{non-interactive MLE}: it chooses an allocation $\{x_{I_1},x_{I_2},\dots, x_{I_T}\} \in \X$ prior to the game, then observes $ y_{t,i} = \theta_{*,i} + \eta_{t,i},  \forall i \in x_{I_t} $
where $\eta_t \sim \cN(0,I)$, and forms the MLE $ \hat{\theta}_i = \frac{1}{T_i} \sum_{t=1, x_{I_t,i} =1}^T  y_{t,i}  $ and outputs $\widehat{x} = \argmax_{x \in \X}  z^\t \widehat{\theta} $. Since the non-interactive MLE may use knowledge of $\thetast$ in choosing its allocation and the estimator and recommendation rules are very natural, we view the sample complexity of the non-interactive MLE as a good benchmark to measure the sample complexity of algorithms against. The following lower bound for the non-interactive MLE resembles Theorem 3 in \cite{katz2020empirical}.

\begin{theorem}
\label{thm:mle_lower_bound}
Fix $\X \subset \{0,1\}^d$ and $\thetast \in \R^d$. Let $\delta \in (0,0.015]$. There exists a universal constant $c>0$ such that if the non-interactive MLE uses less than $c( \gamma^*+ \log(1/\delta)\rho^{\ast})$ samples, it makes a mistake with probability at least $\delta$.
\end{theorem}

The proof is quite similar to the proof of Theorem 3 in \cite{katz2020empirical}, so we merely sketch it here.

\begin{proof}
Consider the combinatorial bandit protocol with $\X \subset \{0,1\}^d$ as the collection of sets: at each round $t \in \N$, the agent picks $J_t \in [d]$ and observes $\theta_{J_t} + N(0,1)$  (see \cite{katz2020empirical} for a more precise definition). Let $T^\prime \in \N$ and fix an allocation $I_1, \ldots, I_{T^\prime} \in [d]$. 
Define
\begin{align*}
\gamma^*_{\text{combi}}(I_1,\ldots,I_{T^\prime}) & = \E_{\eta \sim N(0,I)}[ \sup_{x \in \X \setminus \{x_*\}} \frac{(x_*-x)^\t (\sum_{ s=1}^{T^\prime} e_{I_s} e_{I_s}^\t )^{-1/2}  \eta}{\Delta_x}]^2 \\
\rho^*_{\text{combi}}( I_1,\ldots,I_{T^\prime}) & = \sup_{x \in \X \setminus \{x_*\}} \frac{\norm{x_*-x}^2_{(\sum_{ s=1}^{T^\prime} e_{I_s} e_{I_s}^\t )^{-1}}}{\Delta_x^2}.
\end{align*}
Theorem 3 in \cite{katz2020empirical} shows that there exists a universal constant $c>0$ such that if $c \leq  \gamma^*(I_1,\ldots,I_{T^\prime})$ or $c \leq \log(1/\delta)\rho^{\ast}(I_1,\ldots,I_{T^\prime})$, the with probability at least $\delta$, the oracle MLE makes a mistake.

Now, consider the semi-bandit problem and wlog suppose that $\X = \{x_1, \ldots, x_m\}$. Now, fix an allocation $x_{J_1},\ldots, x_{J_T} \in \X$ for the semi-bandit problem. Define $\lambda_i = \frac{1}{T}\sum_{s=1}^{T} \one\{J_s = i \}$. Suppose that $T \leq 1/2  \frac{1}{c} \log(1/\delta) \rho^* + \gamma^*]\leq \frac{1}{c} \max(\log(1/\delta) \rho^*, \gamma^*) $. Then,
\begin{align*}
cT \leq  \gamma^* = \min_{\lambda \in \triangle}  \gamma^*(\lambda)  \leq \gamma^*(\lambda)
\end{align*}
where 
\begin{align*}
\gamma^*(\lambda) & :=  \E_{\eta} \left [ \sup_{x \in \X \setminus \{x_* \}} \frac{(x_*-x)^\t \Asb(\lambda)^{-1/2} \eta}{ \thetast^\t(x_*-x )} \right ]^2.
\end{align*}
Now, rearranging the above inequality,we have that 
\begin{align*}
c \leq \gamma^*(T \lambda).
\end{align*}
Note that the allocation $T \lambda$ for the semi-bandit problem specifies an allocation $I_1,\ldots,I_{T^\prime}$ for the combinatorial bandit problem and the stochastic process (and non-interactive MLE algorithm) is the same on both problems. Thus, $\gamma^*(T \lambda)$ can be interpreted as $\gamma^*_{\text{combi}}(I_1,\ldots,I_{T^\prime})$ in the combinatorial bandit protocol for some allocation $I_1,\ldots, I_{T^\prime}$, and we may apply the proof of Theorem 3 to obtain that with probability at least $\delta$, the oracle MLE makes a mistake.

\end{proof}

\section{Computational Complexity Results}\label{sec:comp}

\subsection{Algorithmic Approach}
In this section, we present the main computational algorithms and results in the paper, culiminating in the proof of Theorem \ref{thm:main_comp}, which immediately implies Theorem \ref{thm:comp_complex}. For simplicity label $\X= \{x_1,\ldots, x_m\}$. We can always find $\tilde{x}_1, \ldots, \tilde{x}_d \in \X$ such that $\cup_{i=1}^d \tilde{x}_i = [d]$ in $d$ linear maximization oracle calls. For each $i \in [d]$, create a cost vector: 
\begin{align*}
v^{(i)}_j = \begin{cases}
\infty & j = i \\
0 & j \neq i
\end{cases} 
\end{align*}
and set $\tilde{x}_i = \argmax_{x \in \X} x^\t v^{(i)}$. Thus, by reordering we may suppose that $\cup_{i=1}^d x_i = [d]$. Now, define
\begin{align*}
{\trianglem}= \{\lambda \in \triangle : \lambda_i \geq \psi \, \, \forall i \in [d] \}
\end{align*}
where $\psi \leq 1/d$. We optimize over $\trianglem$ due to its computational benefits,e.g., controlling the second partial order derivatives of the Lagrangian of \eqref{eq:opt_problem}. 

Algorithm \ref{alg:main_compute} is the main algorithm (see Theorem \ref{thm:main_comp} for its guarantee); it essentially does a grid search over the time horizon variable, $\tau \in [T]$. Note that for a fixed $\tau \in [T]$, we have that for all $\lambda \in \triangle$
\begin{align*}
\tau \sum_{x \in \X} [\beta+\bar{\theta}^\t(\bar{x}-x) ] \lambda_x = \tau \beta + \tau \sum_{x \in \X} \bar{\theta}^\t(\bar{x}-x)  \lambda_x 
\end{align*}
and thus we can ignore the term $\tau \beta$.  Thus, Algorithm \ref{alg:main_compute} calls Algorithm \ref{alg:binary_search_mult_weight} to solve for a fixed $\tau \in [T]$ the following optimization problem. 
\begin{align}
\min_{\lambda \in \trianglem} & \tau \sum_{x \in \X} \bar{\theta}^\t(\bar{x}-x)  \lambda_x \label{eq:opt_fixed_tau} \\
& \text{ s.t. } \mathbb{E}_\eta \left [ \max_{x \in \X} \frac{(\bar{x}-x)^\top \Asb(\lambda)^{-1/2} \eta}{\beta + \bar{\theta}^\t ( \bar{x}-x)} \right ]  \leq \sqrt{\tau} C \nonumber
\end{align}
To solve the above optimization problem, we convert it into a series of convex feasibility programs of the following form: $\exists? \lambda \in \trianglem$ such that
\begin{align*}
 & \tau \sum_{x \in \X} \bar{\theta}^\t(\bar{x}-x)  \lambda_x \leq \mwtest \\
& \text{ s.t. } \mathbb{E}_\eta \left [ \max_{x \in \X} \frac{(\bar{x}-x)^\top \Asb(\lambda)^{-1/2} \eta}{\beta + \bar{\theta}^\t ( \bar{x}-x)} \right ]  \leq \sqrt{\tau} C 
\end{align*}
and perform binary search over $\mwtest$. To solve each of these convex feasibility programs, we employ the Plotkin-Shmoys-Tardos reduction to online learning and apply Algorithm \ref{alg:mult_weight}, a multiplicative weights update style algorithm. Lemmas \ref{lem:mw_feasibility} and \ref{lem:bin_search} provide the guarantees for the multiplicative weights update algorithm and for the binary search procedure, respectively.

The Plotkin-Shmoys-Tardos reduction requires a method for solving for arbitrary $\kappa_1, \kappa_2 \in [0,1]$:
\begin{align*}
\min_{\lambda \in \trianglem} \L(\kappa_1,\kappa_2 ; \tau; \lambda) := \kappa_1 \tau \sum_{x \in \X} \bar{\theta}^\t(\bar{x}-x)  \lambda_x + \kappa_2(\mathbb{E}_\eta \left [ \max_{x \in \X} \frac{(\bar{x}-x)^\top \Asb(\lambda)^{-1/2} \eta}{\beta + \bar{\theta}^\t ( \bar{x}-x)} \right ]  - \sqrt{\tau} C).
\end{align*}
To solve the above optimization problem, we use stochastic Frank-Wolfe (see Algorithm \ref{alg:sfw_sb}). Defining for a fixed $\eta \in \R^d$, 
\begin{align*}
\L( \kappa_1, \kappa_2 ; \tau; \lambda; \eta ) =\kappa_1 \tau \sum_{x \in \X} \bar{\theta}^\t(\bar{x}-x)  \lambda_x + \kappa_2( \max_{x \in \X} \frac{(\bar{x}-x)^\top \Asb(\lambda)^{-1/2} \eta}{\beta + \bar{\theta}^\t ( \bar{x}-x)}  - \sqrt{\tau} C).
\end{align*}
we see that
\begin{align*}
\E_{ \eta \sim N(0,I)} [\L( \kappa_1, \kappa_2 ; \tau; \lambda; \eta )] = \L(\kappa_1, \kappa_2 ; \tau; \lambda).
\end{align*}
See Lemma \ref{lem:sfw_convergence} for our convergence result on stochastic Frank-Wolfe.

Finally, we note that each of our algorithms uses a global variable $\tol$, which for the theory we set to $\tfrac{(\sqrt{2}-1)C}{4}$. We note that $C$ scales  as $\frac{1}{\sqrt{\log(\frac{1}{\delta})}}$ and thus a polynomial dependence on $1/\tol$ results in a polynomial dependence on $\log(1/\delta)$.

 \begin{algorithm}[H] 
\begin{algorithmic}[1]
\State \textbf{Input:} Tolerance parameter $\textsc{tol} \in (0,1)$, $\delta \in (0,1)$
\State $k \longleftarrow 1$, $\bar{\tau}_k \longleftarrow 2^k $
\While{$\bar{\tau}_k  \leq T$ }
 	 \State $(\textsc{feasible}_k, \lambda_k ) \longleftarrow \text{binSearch}(\bar{\tau}_k, \frac{\delta}{\log_2(T)})$
 	 \State $k \longleftarrow k+1$, $\bar{\tau}_k \longleftarrow 2^k $
\EndWhile
\If{$\textsc{feasible}_k$ is False for all $k$}
\State \Return ''Program is not feasible"
\EndIf
\State $\widehat{k}_* \longleftarrow \argmin_k \{\bar{\tau}_k \sum_{x \in \X} \bar{\theta}^\t (\bar{x}-x) \lambda_{k,x} : \textsc{feasible}_k \text{ is True } \}$
\State \Return $(2 \bar{\tau}_{\widehat{k}_*}, \lambda_{\widehat{k}_*})$ 
\end{algorithmic}
\caption{Main}
\label{alg:main_compute}
\end{algorithm}

 \begin{algorithm}[H] 
\begin{algorithmic}[1]
\State \textbf{Input:} $\bar{\tau} > 0$,  $\delta \in (0,1)$, Tolerance parameter $\textsc{tol} > 0$
\State $\textsc{LOW} \longleftarrow 0$, $\textsc{HIGH} \longleftarrow 2Td$
\State $(\textsc{feasible}, \bar{\lambda}) \longleftarrow \text{MW}(\bar{\tau}, \textsc{high}, \frac{\delta}{\ceil{\log_2(2Td/\tol)}+1})$ \texttt{\color{blue}{Check if program is feasible}}
\If{\textsc{feasible} is False}
\State \Return $(\textsc{feasible}, \bar{\lambda})$
\EndIf
\While{$\textsc{HIGH}  - \textsc{LOW}  \geq \tol$ }  \hfill $\triangleright$ \texttt{\color{blue}{Initiate binary search}}
\State $\mwtest \longleftarrow \frac{\textsc{LOW} + \textsc{HIGH}}{2}$
\State $(\textsc{feasible}, \bar{\lambda}) \longleftarrow \text{MW}(\bar{\tau}, \mwtest,\frac{\delta}{\ceil{\log_2(2Td/\tol)}+1})$
\If{\textsc{feasible}}
\State $\textsc{LOW}  \longleftarrow \mwtest$
\Else
\State $\textsc{high}  \longleftarrow \mwtest$
\EndIf
\EndWhile
\State $(\textsc{feasible}, \bar{\lambda}) \longleftarrow \text{MW}(\bar{\tau}, \textsc{high}, \frac{\delta}{\ceil{\log_2(2Td/\tol)}+1})$
\State \Return $(\textsc{feasible}, \bar{\lambda})$ 
\end{algorithmic}
\caption{Binary Search (binSearch)}
\label{alg:binary_search_mult_weight}
\end{algorithm}

 \begin{algorithm}[H] 
\begin{algorithmic}[1]
\State \textbf{Input:} $\bar{\tau} > 0$, $\mwtest>0$, Failure probability $\delta \in (0,1)$, Tolerance parameter $\textsc{tol} > 0$
\State $\rho = \max(2dT, c \frac{d}{\beta \psi^{1/2}})$ for an appropriately chosen universal constant $c > 0$ (see the proof of Lemma \ref{lem:mw_feasibility})
\State $\eta = \min(\frac{\tol}{4 \rho}, 1/2)$, $R \longleftarrow \frac{16 \rho^2 \ln(2)}{\tol^2}$
 	 \State \textsc{Feasible} $\longleftarrow \textsc{True}$  \hfill $\triangleright$ \texttt{\color{blue}{Assume feasible program}}
\State $w_i^{(1)} \longleftarrow 1$ for $i \in [2]$  \hfill $\triangleright$ \texttt{\color{blue}{Initiate weights}}
\For{$r=1,2, \ldots, R$}
	\State $p_1^{(r)} \longleftarrow w^{(r)}_1/(w_1^{(r)} +w_2^{(r)})$ and $p_2^{(r)} \longleftarrow w_2^{(r)}/(w_1^{(r)}+w_2^{(r)})$
 	 \State $\lambda^{(r)} \longleftarrow \text{SFW}(p_1^{(r)}, p_2^{(r)}, \frac{\delta}{2R})$
 	 \State Define
	\begin{align*}
 	 h_1(\lambda) & := \bar{\tau} \sum_{x \in \X} \bar{\theta}^\t(\bar{x}-x)  \lambda_x - \mwtest  \\
	\widehat{h}_2(\lambda^{(r)}) &= \text{estSup}(\lambda^{(r)},  \frac{\delta}{3 R}) - \sqrt{\tau}C \\
	\widehat{h}^{(r)}(\lambda^{(r)}) & := p_1^{(r)}  h_1(\lambda^{(r)}) +  p_2^{(r)} \widehat{h}_2(\lambda^{(r)})
\end{align*}	 	 
 	 \State 
 	 \If{$\widehat{h}^{(r)}(\lambda^{(r)}) > 2\textsc{tol}$}
 	 \State \textsc{Feasible} $\longleftarrow \textsc{False}$ \hfill $\triangleright$ \texttt{\color{blue}{Declare infeasible program}}
 	 \State Break
 	 \EndIf 
 	 \State $w_1^{(r+1)} \longleftarrow w_1^{(r)} (1 + \eta  h_1(\lambda^{(r)}) )$  \hfill $\triangleright$ \texttt{\color{blue}{Update weights}}
 	 \State $w_2^{(r+1)} \longleftarrow w_2^{(r)}(1+\eta \widehat{h}_2(\lambda^{(r)}))$
\EndFor
 \State $\bar{\lambda}^{(r)} = \frac{1}{r} \sum_{s=1}^r \bar{\lambda}_s$
\State \Return $(\textsc{feasible}, \bar{\lambda}^{(r)})$ 
\end{algorithmic}
\caption{Multiplicative Weights Update Algorithm for Combinatorial Bandits with Semi-Bandit Feedback (MW)}
\label{alg:mult_weight}
\end{algorithm}

\begin{algorithm}[H] 
\begin{algorithmic}[1]
\State \textbf{Input:} $\tau \geq 0$, $\kappa_1, \kappa_2 \in [0,1]$,  $\delta \in (0,1)$.
\State $\Rsfw =\frac{8 \frac{1}{\beta \psi^{5/2} } d }{\tol}$
\State $(q_r)_{r \in [R]} \in [0,1]^R$ such that $q_r = \frac{2}{r+1}$ and $(p_r)_{r \in [R]} \in \N^R$ such that $p_r = c \frac{1}{d \psi^2 q_r}\log(r^2/\delta)$ for an appropriately chosen universal constant $c > 0$ (see the proof of Lemma \ref{lem:sfw_convergence})
\State Initialize $\lambda_1 \in \trianglem$ by setting $\lambda_{1,i} = 1/d$ if $i \in [d]$ and otherwise set $\lambda_{1,i} = 0$.
\For{$r=1,2, \ldots, \Rsfw$}
 	 \State Draw $\eta^{(1)}, \ldots, \eta^{(p_r)} \sim N(0,I)$
	  \State Compute
	  \begin{align*}
	\tilde{\nabla}_r = \frac{1}{p_r} \sum_{j=1}^{p_r} \nabla \L(\kappa_1, \kappa_2; \tau; \lambda_r; \eta_j)
	\end{align*}  
	 \State Compute 
	 \begin{align*}
	i_r \longleftarrow \argmax_{i \in [m] }- \tilde{\nabla}_{r,i}  =  -[\kappa_1 \tau \bar{\theta}^\t x_i +\kappa_2 \frac{1}{2} \frac{1}{p_r} \sum_{j=1}^{p_r}\frac{1}{[\beta + \bar{\theta}^\t(\bar{x}-\tilde{x}_j) ]}  \sum_{k \in (\bar{x} \Delta \tilde{x}_j) \cap x_i } \frac{\eta_k}{(\sum_{l: k \in x_l} \lambda_l )^{3/2}}   ]
	 \end{align*}
	 where 
	 \begin{align*}
\tilde{x}_j =  \argmax_{x \in \X} \frac{\sum_{i \in \bar{x} \Delta x} \frac{\eta_i^{(j)}}{\sum_{x^\prime : i \in x^\prime} \lambda_{x^\prime} }}{\beta+ \bar{\theta}^\t(\bar{x}-x)}.
\end{align*}
is computed using Algorithm \ref{alg:compute_max}.
\State
\begin{align*}
(v_r)_i & = \begin{cases}
\begin{cases}
0 & i \not \in [d] \\
\psi & i \in [d] \setminus \{ i_r \} \\
1- (d-1) \psi & i = i_r
\end{cases}, \qquad \qquad i_r \in [d] \\
\begin{cases}
\psi &  \hspace{1.2cm}  i \in [d] \\
1- d \psi  & \hspace{1.2cm} i = i_r 
\end{cases}, \qquad \qquad  \qquad  i_r \not \in [d]
\end{cases}
\end{align*}
\State
	 \begin{align*}
	 \lambda_{r+1} \longleftarrow q_r v_r +(1-q_r) \lambda_r
	 \end{align*}
\EndFor

\State \Return $\lambda_{\Rsfw}$
\end{algorithmic}
\caption{Stochastic Frank-Wolfe for Semi-Bandit Feedback (SFW)}
\label{alg:sfw_sb}
\end{algorithm}

\subsubsection{Subroutines}

Algorithm \ref{alg:compute_max}, originally provided in \cite{katz2020empirical}, uses binary search and calls to the linear maximization oracle to compute 
\begin{align*}
 \frac{(\bar{x}- x)^\t A(\lambda)^{-1/2} \eta}{\beta + \bar{\theta}^\t(\bar{x} - x) }.
\end{align*}
Algorithm \ref{alg:est_exp_sup} estimates 
\begin{align*}
\E_{\eta}[ \max_{x \in \X} \frac{(\bar{x}-x)^\top \Asb(\lambda)^{-1/2} \eta}{\beta + \bar{\theta}^\t ( \bar{x}-x)}].
\end{align*}

\begin{algorithm}[H]
\begin{algorithmic}[1]
\State Define the following functions
\begin{align*}
g(\lambda;\eta;x) & :=   \frac{(\bar{x}- x)^\t A(\lambda)^{-1/2} \eta}{\beta + \bar{\theta}^\t(\bar{x} - x) }\\
g( \lambda; \eta;r) & := \max_{x \in \X} x^\t(A(\lambda)^{-1/2} \eta + r \bar{\theta}) - r(\beta + \bar{\theta}^\t x ) -\bar{x}^\t A(\lambda)^{-1/2} \eta \\
g(\lambda; \eta;r;x) & := x^\t(A(\lambda)^{-1/2} \eta + r \bar{\theta}) - r(\beta + \bar{\theta}^\t \bar{x} ) -\bar{x}^\t A(\lambda)^{-1/2} \eta 
\end{align*}
\State Define
\begin{align*}
\textsc{low} = 0, \qquad \textsc{high} = 2 
\end{align*}
\While{$g(\lambda; \eta: \textsc{high}) \geq 0$}
\State $\textsc{high}  \longleftarrow 2\cdot\textsc{high} $
\EndWhile
\While{$g(\lambda; \eta; \textsc{low}) \neq 0 $ }
\If{$g(\lambda; \eta;\frac{1}{2}(\textsc{high} + \textsc{low}) ) < 0$}
\State $\textsc{low}  \longleftarrow \frac{1}{2}(\textsc{high} + \textsc{low}) $
\Else
\State $\textsc{high}  \longleftarrow \frac{1}{2}(\textsc{high} + \textsc{low}) $
\State $\textsc{low} \longleftarrow g(\lambda;\eta;x^\prime)$ for some $x^\prime \in \argmax_{x \in \X} g(\lambda;\eta;\textsc{low};x)$
\EndIf
\EndWhile
\State Return $\textsc{low}$
\end{algorithmic}
 \caption{$\text{computeMax}$}
 \label{alg:compute_max}
\end{algorithm}

 \begin{algorithm}[H] 
\begin{algorithmic}[1]
\State \textbf{Input:} $\lambda \in \trianglem$, failure probability $\delta > 0$, Tolerance parameter $\textsc{tol} > 0$, , 
\State $t = c \log(1/\delta) \frac{d}{\beta^2 \psi \textsc{tol}^2}$
\State Draw $\eta_1, \ldots, \eta_t \sim N(0,I)$ 
\State Compute $g_s =  \max_{x \in \X} \frac{(\bar{x}-x)^\top \Asb(\lambda)^{-1/2} \eta_s}{\beta + \bar{\theta}^\t ( \bar{x}-x)}$ for $s =1,\ldots, t$ using Algorithm \ref{alg:compute_max}.
\State \Return $\frac{1}{t} \sum_{s=1}^t g_s$ 
\end{algorithmic}
\caption{ Estimate expected suprema (estimateSup)}
\label{alg:est_exp_sup}
\end{algorithm}

\subsection{Main Optimization Proofs}

For the sake of simplicity, we assume that $T$ is a power of $2$, and that the optimization problem is feasible. If the optimization problem is infeasible, we can determine this by applying stochastic Frank-Wolfe (see Lemma \ref{lem:sfw_convergence}).
For simplicity, we also assume that $\bar{\theta}^\t(\bar{x}-x) \leq \Delmax \leq 2d$ since typically it is assumed that $\norm{\theta}_\infty \leq 1$ and whp $\norm{\widehat{\theta}_{\ell}}_\infty = O(1)$ at every round $\ell$. Further, note that whenever the algorithm is applied $C \leq 1$, and we assume this henceforth. We introduce the following functions to bound the number of linear maximization oracle calls:
\begin{align*}
\oca(d, \beta, \psi, \tol, 1/\delta, 1/\xi) & =O(\frac{d^2}{\tol^3 \beta^4 \psi^8} [d+ \log(\frac{d}{\beta \xi \psi \tol})]) \\
\ocb(d, \beta, \psi, \tol, 1/\delta, 1/\xi) & = O( \log(1/\delta) \frac{d}{\beta^2 \psi \textsc{tol}^2}[d  + \log(\frac{d \Delmax}{\beta \xi })]) \\
\occ(d, \beta, \psi, \tol, 1/\delta, 1/\xi) & =\frac{(d T)^2 + \frac{d^2}{\beta^2 \psi}}{\tol^2}[\oca(d, \beta, \psi, \tol, 1/\delta, 1/\xi) + \ocb(d, \beta, \psi, \tol, 1/\delta, 1/\xi)]
\end{align*}
Note these are polynomial in $(d, \beta, \psi, 1/\tol, \log(1/\delta), 1/\xi) $. Our algorithms share a global parameter $\tol$; it suffices to set $\tol=\tfrac{(\sqrt{2}-1) C}{4}$. 
 Define
\begin{align*}
M = \log_2(T) \log_2(\frac{2 T d}{\tol}).
\end{align*}

We say a random variable $X$ is sub-Gaussian with parameter $\sigma^2$ and write $X \in \sg(\sigma^2)$ if for all $\lambda \in \R $
\begin{align*}
\E[e^{\lambda(X - \E[X])}] \leq e^{\lambda^2\sigma^2/2}.
\end{align*}

The following Lemma provides the convergence guarantee for stochastic Frank-Wolfe in the semi-bandit setting (see Algorithm \ref{alg:sfw_sb}).

\begin{lemma}
\label{lem:sfw_convergence}
Let $\delta \in (0,1)$, $\xi \in (0,1]$, $\kappa_1, \kappa_2 \in [0,1]$. With probability at least $1-\delta$ Algorithm \ref{alg:sfw_sb} returns $\lambda_{\Rsfw} \in \triangle$ such that
\begin{align*}
\L( \kappa_1,\kappa_2; \tau; \lambda_{\Rsfw})  \leq   \min_{\lambda \in \trianglem} \L(\kappa_1,\kappa_2; \tau; \lambda) + \tol
\end{align*}
Furthermore, with probability at least $1- \frac{c\xi }{2^d}$, 
the number of oracle calls is bounded by
\begin{align*}
\oca(d, \beta, \psi, 1/\tol, \log(1/\delta), 1/\xi).
\end{align*} 
\end{lemma}

\begin{proof}

For simplicity, we focus on the case where $\kappa_1=\kappa_2 = 1$ (the other cases are similar). We write $\L(\lambda)$  and $\L(\lambda;\eta)$ as abbreviations for $\L(\kappa_1, \kappa_2 ; \lambda)$  and $\L(\kappa_1, \kappa_2 ; \lambda;\eta)$.

\textbf{Step 1: Bound the number of iterations of stochastic Frank-Wolfe.} $ \L( \lambda)$ is convex in $\lambda$ by Proposition \ref{prop:sb_convex}. Furthermore, $\max_{\lambda,\lambda^\prime \in \trianglem} \norm{\lambda -\lambda^\prime}_1 \leq 2$. Thus, by Proposition \ref{prop:sfw_gen}, it suffices to show
\begin{enumerate}
\item \textbf{Smoothness:} $\norm{\nabla \L( \lambda)  - \nabla \L(\lambda^\prime) }_\infty \leq L \norm{\lambda-\lambda^\prime}_1$ for an appropriate choice of $L$
\item \textbf{Small deviation with high probability:} $p_r$ is chosen sufficiently large to ensure that with probability at least $1-\delta/r^2$
\begin{align*}
 \norm{\wt{\nabla}_r - \nabla \L(\lambda_{r-1})}_\infty \leq \frac{L q_r}{2}
\end{align*} 
\end{enumerate}

\textbf{Step 1.1: Smoothness.} Let $\lambda, \lambda^\prime \in \trianglem$ and fix $i \in [m]$. It suffices to show that 
\begin{align*}
|\frac{\partial \L( \lambda )}{\partial \lambda_i } - \frac{\partial \L( \lambda^\prime )}{\partial \lambda_i^\prime }| \leq L \norm{\lambda - \lambda^\prime}_1.
\end{align*}
For the sake of abbreviation, define $g(\lambda) := \frac{\partial \L(\lambda )}{\partial \lambda_i } $. By Lemma \ref{lem:differentiable_2}, we have that $\L(\lambda)$ is twice differentiable and that
\begin{align*}
\frac{\partial^2 \L(\lambda)}{\partial \lambda_i \partial \lambda_j} & = \E[\frac{\partial^2 \L(\lambda;\eta)}{\partial \lambda_i \partial \lambda_j} \one\{B\}] \\
& = \E[ \frac{3}{4} \frac{1}{(\beta+ \bar{\theta}^\t (\bar{x}-\tilde{x})} \sum_{k \in (\bar{x} \Delta \tilde{x}) \cap x_i \cap x_j} \frac{\eta_k}{(\sum_{l : k \in x_l} \lambda_l)^{5/2}}: 
\tilde{x} = \argmax_{x \in \X} \frac{\sum_{i \in \bar{x} \Delta x} \frac{1}{\sum_{x^\prime : i \in x^\prime} \lambda_{x^\prime} }}{\beta+ \bar{\theta}^\t(\bar{x}-x)} 
]
\end{align*}
where 
\begin{align*}
B = \{\eta :  |\argmax_{x \in \X} \frac{\sum_{i \in \bar{x} \Delta x} \frac{1}{\sum_{x^\prime : i \in x^\prime} \lambda_{x^\prime} }}{\beta+ \bar{\theta}^\t(\bar{x}-x)} | =1 \}.
\end{align*}
For any $\lambda \in \trianglem$,
\begin{align*}
| \frac{\partial g(\lambda )}{\partial \lambda_j}| & =| \E[ \frac{3}{4} \frac{1}{(\beta+ \bar{\theta}^\t (\bar{x}-\tilde{x})} \sum_{k \in (\bar{x} \cap \tilde{x}) \cap x_i \cap x_j} \frac{\eta_k}{(\sum_{l : k \in x_l} \lambda_l)^{5/2}} \one\{B\} :
\tilde{x} = \argmax_{x \in \X} \frac{\sum_{i \in \bar{x} \Delta x} \frac{1}{\sum_{x^\prime : i \in x^\prime} \lambda_{x^\prime} }}{\beta+ \bar{\theta}^\t(\bar{x}-x)} .
] | \\
& \leq  \E[ |\frac{3}{4} \frac{1}{(\beta+ \bar{\theta}^\t (\bar{x}-\tilde{x})} \sum_{k \in (\bar{x} \cap \tilde{x}) \cap x_i \cap x_j} \frac{\eta_k}{(\sum_{l : k \in x_l} \lambda_l)^{5/2}}| \one\{B\} : \tilde{x} = \argmax_{x \in \X} \frac{\sum_{i \in \bar{x} \Delta x} \frac{1}{\sum_{x^\prime : i \in x^\prime} \lambda_{x^\prime} }}{\beta+ \bar{\theta}^\t(\bar{x}-x)} . ]  \\
& \leq \E \frac{3}{4} \frac{1}{\beta} \sum_{k=1}^d |\eta_k| \\
& \leq c\frac{1}{\beta \psi^{5/2}} d . 
\end{align*}
where we used Jensen's inequality and $c > 0$ is a universal constant. 

Now, by the mean value theorem, there exists $s \in [0,1]$ such that
\begin{align*}
|g(\lambda) - g(\lambda^\prime)| & \leq |\nabla g(s \lambda + (1-s) \lambda^\prime)^\t (\lambda - \lambda^\prime)| \\
& \leq \norm{\nabla g(s \lambda + (1-s) \lambda^\prime)}_\infty \norm{\lambda-\lambda^\prime}_1 \\
& \leq c \frac{1}{\beta  \psi^{5/2}} d \norm{\lambda-\lambda^\prime}_1
\end{align*}
where the second inequality follows by Holder's Inequality. Thus,
\begin{align*}
\norm{ \nabla \L(\kappa ; \lambda )- \nabla \L(\kappa ; \lambda^\prime ) }_\infty \leq  c \frac{1}{\beta \psi^{5/2} } d \norm{\lambda - \lambda^\prime}_1 
\end{align*}
For the sake of brevity, we write $L =\frac{1}{\beta \psi^{5/2} } d^{3/2} $ for the remainder of the proof.

\textbf{Step 1.2: Small deviation with high probability.} Now, we show that $p_r$ is chosen sufficiently large to ensure that with probability at least $1-\delta/r^2$
\begin{align}
\norm{\wt{\nabla}_r - \nabla \L(\lambda_{r-1})}_\infty \leq \frac{L q_r}{2}. \label{eq:sfw_sb_small_dev}
\end{align} 
Recall that
\begin{align*}
[\wt{\nabla}_r - \nabla \L(\lambda_{r-1})]_i & = [\wt{\nabla}_r - \nabla \L(\lambda_{r-1})]_i \\
& = \frac{1}{p_r} [\sum_{j=1}^{p_r} \frac{1}{2} \frac{1}{[\beta + \bar{\theta}^\t(\bar{x}-\tilde{x}_j) ]}  \sum_{k \in (\bar{x} \Delta \tilde{x}_j) \cap x_i } \frac{\eta_k}{(\sum_{l: k \in x_l} \lambda_l )^{3/2}}   \\
& - \E[(\frac{1}{2} \frac{1}{[\beta + \bar{\theta}^\t(\bar{x}-\tilde{x}) ]}  \sum_{k \in (\bar{x} \Delta \tilde{x}) \cap x_i } \frac{\eta_k}{(\sum_{l : k \in x_l} \lambda_l )^{3/2}} :\tilde{x} = \argmax_{x \in \X} \frac{\sum_{i \in \bar{x} \Delta x} \frac{1}{\sum_{x^\prime : i \in x^\prime} \lambda_{x^\prime} }}{\beta+ \bar{\theta}^\t(\bar{x}-x)} .
]
 \end{align*}
	 where 
	 \begin{align*}
\tilde{x}_j = \argmax_{x \in \X} \frac{\sum_{i \in \bar{x} \Delta x} \frac{\eta_i^{(j)}}{\sum_{x^\prime : i \in x^\prime} \lambda_{x^\prime} }}{\beta+ \bar{\theta}^\t(\bar{x}-x)} .
.
\end{align*}
Note that
\begin{align*}
|\frac{1}{p_r} \sum_{j=1}^{p_r} \frac{1}{2} \frac{1}{[\beta + \bar{\theta}^\t(\bar{x}-\tilde{x}_j) ]}  \sum_{k \in \tilde{x}_j \cap x_i } \frac{\eta_k}{(\sum_{l: k \in x_l} \lambda_l )^{3/2}} | \leq \frac{1}{2} \frac{1}{p_r} \frac{1}{\beta \psi^{3/2}} \sum_{k=1}^d |\eta_k|.
\end{align*}
Since 
\begin{align*}
\frac{1}{2} \frac{1}{p_r} \frac{1}{\beta \psi^{3/2}} \sum_{k=1}^d |\eta_k| \in \sg(\frac{c \frac{1}{\beta^2 \psi^{3}} d}{p_r})
\end{align*}
we then have that by Lemma 2.6.8 in \cite{vershynin2018high},
\begin{align*}
[\wt{\nabla}_r - \nabla \L(\lambda_{r-1})]_i \in \sg(\frac{c \frac{1}{\beta^2 \psi^{3}} d}{p_r}).
\end{align*}
Therefore, since $|\X| \leq 2^d$ and since $p_r = c \frac{\frac{1}{\beta^2 \psi^{3}} d^2}{L^2 q_r^2}$ for an appropriately chosen universal constant, by a standard sub-Gaussian tail bound \eqref{eq:sfw_sb_small_dev} follows.

\textbf{Step 2: Bound the number of linear maximization oracle calls.} Next, we bound the number of linear maximization oracle calls. At each round $r$, there is one linear maximization oracle call from finding the minimizing direction wrt the gradient over $\trianglem$, but the dominant source of linear maximization oracles at each round is due to applying Algorithm \ref{alg:compute_max} several times. Thus, it suffices to bound the number of linear maximization oracle calls due to Algorithm \ref{alg:compute_max}. Define the following event
\begin{align*}
\mc{E}_k & = \{\text{ the $k$th application of Algorithm \ref{alg:compute_max} requires } O(d  + \log(\frac{d}{\beta }) + \log(\Delmax\xi k^2)) \text{ oracle calls} \} \\
\mc{E} &  = \cap_k \mc{E}_k
\end{align*}
Then, we have that 
\begin{align*}
\Pr(\mc{E}) & = \prod_{r=1}^\infty \Pr(\mc{E}_r | \cap_{s=1}^{r-1} \mc{E}_s) \geq \prod_{r=1}^\infty (1-\frac{\xi}{2^d r^2}) = \frac{\sin(\pi \frac{\xi}{2^d})}{\pi+\frac{\xi}{2^d}} \geq 1 - \frac{\xi}{2^d}.
\end{align*}
where we used the independence of each draw of a multivariate Gaussian in the algorithm and Lemma \ref{lem:find_argmax}. The number of calls of Algorithm \ref{alg:compute_max} at each iteration is upper bounded by $O(p_{\Rsfw})$ and, thus, the total number of oracle calls is upper bounded by
\begin{align*}
O({\Rsfw} \cdot p_{\Rsfw} &[d  + \log(\frac{d}{\beta }) + \log(\Delmax \Rsfw/\xi)]) \\
& \leq O(\frac{d^2}{\tol^3 \beta^4 \psi^8}  [d+ \log(\frac{d}{\beta \xi}) + \log(\frac{d^2}{\tol^3 \beta^4 \psi^8} )] \\
& =  O(\frac{d^2}{\tol^3 \beta^4 \psi^8} [d+ \log(\frac{d}{\beta \xi \psi \tol})]) \\
& = \oca(d, \beta,  \psi, \tol, 1/\delta, 1/\xi).
\end{align*}

\end{proof}

The following Lemma shows that the Multiplicative Weight Update algorithm (Algorithm \ref{alg:mult_weight}) either finds an approximately feasible solution or if there is no approximately feasible solution, determines infeasibility.

\begin{lemma}
\label{lem:mw_feasibility}
Fix $\tau, \mwtest \geq 0$ and let $\delta \in (0,1)$. Define
\begin{align*}
P_\epsilon = \{ \lambda \in \trianglem : \mathbb{E}_\eta \left [ \max_{x \in \X} \frac{(\bar{x}-x)^\top \Asb(\lambda)^{-1/2} \eta}{\beta + \bar{\theta}^\t ( \bar{x}-x)} \right ]-\sqrt{\tau} C \leq \epsilon, \\
\tau \sum_{x \in \X}  \bar{\theta}^\t(\bar{x}-x)  \lambda_x - \mwtest \leq \epsilon \}
\end{align*}
With probability at least $1-\delta - \frac{1}{2^d M}$, if MW($\tau, \mwtest$) does not declare infeasibility, then MW($\tau, \mwtest$) returns $\bar{\lambda} \in P_{4 \tol}$ and if MW($\tau, \mwtest$) declares infeasibility, then $P_0$ is infeasible. Furthermore, on the same event, MW($\tau, \mwtest$) uses at most $\occ(d, \beta,  \psi, \tol, 1/\delta, 1/\xi)$ linear maximization oracle calls.
\end{lemma}

\begin{proof}
The algorithm uses the Plotkin-Shmoys-Tardos reduction to online learning and essentially runs the multiplicative weights update algorithm (see \cite{arora2012multiplicative}) where there is an expert for each constraint. Define
 	 \begin{align*}
 	 h_1(\lambda) & := \bar{\tau} \sum_{x \in \X} \bar{\theta}^\t(\bar{x}-x)  \lambda_x - \mwtest \\
 	 h_2(\lambda) & := \mathbb{E}_\eta \left [ \max_{x \in \X} \frac{(\bar{x}-x)^\top \Asb(\lambda)^{-1/2} \eta}{\beta + \bar{\theta}^\t ( \bar{x}-x)} \right ]-\sqrt{\bar{\tau}} C \\
 	 h^{(r)}(\lambda) & := p_1^{(r)}  h_1(\lambda) +  p_2^{(r)} h_2(\lambda).
 	 \end{align*} 
 	 At each round $r$, the algorithm chooses a distribution, $p_1^{(r)}$ and $p_2^{(r)}$, over the constraints and the adversary uses the stochastic Frank-Wolfe algorithm to find $\lambda^{(r)}$ such that
 	 \begin{align*}
 	  h^{(r)}(\lambda^{(r)}) \leq \min_{\lambda \in \trianglem} h^{(r)}(\lambda) + \textsc{tol}.
 	 \end{align*}
The reward for expert/constraint 1 is $h_1(\lambda^{(r)})$ and the reward for expert/constraint 2 is $\widehat{h}_2(\lambda^{(r)})$. 

Let $\mc{E}_r$ denote the event that  $\lambda^{(r)} = \text{SFW}(p_1^{(r)},p_2^{(r)}, \frac{\delta}{2R})$ satisfies
\begin{align*}
h^{(r)}(\lambda^{(r)}) \leq \min_{\lambda \in \trianglem} h^{(r)}(\lambda) + \textsc{tol}.
\end{align*}
uses at most $\oca(d, \beta, \psi, \tol, 1/\delta, 1/\xi)$ linear maximization oracle calls. Define $\mc{E} = \cap_r \mc{E}_r$ Further, define the following events
\begin{align*}
\mc{F}_r & = \{ 	 | \text{estSup}(\lambda^{(r)}, , \frac{\delta}{2 R}) - \mathbb{E}_\eta \left [ \max_{x \in \X} \frac{(\bar{x}-x)^\top \Asb(\lambda^{(r)})^{-1/2} \eta}{\beta + \bar{\theta}^\t ( \bar{x}-x)} \right ]| \leq \tol \\
& \text{ and estSup uses  } \ocb(d, \beta, \psi, \tol, 1/\delta, 1/\xi) \text{ oracle calls} \} \\
\mc{F} & = \cap_r \mc{F}_r 
\end{align*}
By Lemmas \ref{lem:sfw_convergence} and \ref{lem:estimate_sup} applied with $\xi = \frac{1}{RM}$ and the law of total probability, we have that
\begin{align*}
\Pr(\mc{E}^c \cup \mc{F}^c ) \leq \sum_{r=1}^R \Pr(\mc{E}^c_r \cup \mc{F}^c_r  | \cap_{s=1}^{r-1} \mc{E}_s \cap \mc{F}_s ) \leq \sum_{r=1}^R \frac{\delta}{R} + \frac{1}{2^dR M} = \delta + \frac{1}{2^d M } 
\end{align*} 
Now, for the remainder of the proof we assume that $\mc{E} \cap \mc{F} $ occurs.

Suppose that at some round $r \in [R]$ Algorithm \ref{alg:sfw_sb} returns $\lambda^{(r)}$ such that $ \widehat{h}^{(r)}(\lambda^{(r)})  > 2\tol$. Then, since $\mc{F}$ implies that
\begin{align*}
|h^{(r)}(\lambda^{(r)}) -  \widehat{h}^{(r)}(\lambda^{(r)})| \leq | \text{estSup}(\lambda^{(r)}, , \frac{\delta}{2 R}) - \mathbb{E}_\eta \left [ \max_{x \in \X} \frac{(\bar{x}-x)^\top \Asb(\lambda^{(r)})^{-1/2} \eta}{\beta + \bar{\theta}^\t ( \bar{x}-x)} \right ]| \leq \tol 
\end{align*}
we have that on $\mc{E} \cap \mc{F}$
 	 \begin{align*}
 	2\tol <   \widehat{h}^{(r)}(\lambda^{(r)}) \leq \tol + h^{(r)}(\lambda^{(r)}) \leq \min_{\lambda \in \trianglem} h^{(r)}(\lambda) + 2\tol.
 	 \end{align*}
Therefore, it follows that for every $\lambda \in \trianglem$, 
\begin{align*}
\max(\mathbb{E}_\eta \left [ \max_{x \in \X} \frac{(\bar{x}-x)^\top \Asb(\lambda)^{-1/2} \eta}{\beta + \bar{\theta}^\t ( \bar{x}-x)} \right ]-\sqrt{\tau} C, \tau \sum_{x \in \X}  \bar{\theta}^\t(\bar{x}-x)  \lambda_x - \mwtest) > 0.
\end{align*}
Thus, the algorithm correctly declares infeasibility of the convex feasibility program.

Next, suppose that the Algorithm \ref{alg:sfw_sb} returns $\lambda^{(r)}$ such that $ \widehat{h}^{(r)}(\lambda^{(r)})  \leq 2\tol$ at every round $r$. Then, we show that the algorithm returns $\bar{\lambda}^{(R)} \in P_{4 \tol}$.  To apply Theorem \ref{thm:mult_weights}, a standard result for the multiplicative weights update algorithm, we must show that for any $\lambda^{(r)} \in \trianglem$ returned during the execution of the Algorithm
\begin{align}
 \max(h_1(\lambda^{(r)}), \widehat{h}_2(\lambda^{(r)})) \leq \rho = \max(2dT, c \frac{d}{\beta \psi^{1/2}}) \label{eq:mw_apply_condition}
\end{align}
where $\rho$ is defined in Algorithm \ref{alg:mult_weight}. We have that
 	 \begin{align*}
 	 h_1(\lambda^{(r)}) & := \bar{\tau} \sum_{x \in \X} \bar{\theta}^\t(\bar{x}-x)  \lambda_x^{(r)} - \mwtest \leq 2d T
 	 \end{align*}
 	since $\bar{\tau} \leq T$, $\mwtest \geq 0$, and we assume that $\bar{\theta}^\t(\bar{x}-x)  \leq 2d$. Furthermore,
	\begin{align*}
	\widehat{h}_2(\lambda^{(r)}) & = \text{estSup}(\lambda^{(r)}, \frac{\delta}{3 R}) - \sqrt{\tau}C \\
	& \leq   \mathbb{E}_\eta \left [ \max_{x \in \X} \frac{(\bar{x}-x)^\top \Asb(\lambda^{(r)})^{-1/2} \eta}{\beta + \bar{\theta}^\t ( \bar{x}-x)} \right ]| + \tol - \sqrt{\tau}C \\
	& \leq \frac{1}{\beta \psi^{1/2} }\E[\sum_{i=1}^d |\eta_i|] \\
	& \leq c \frac{d}{\beta \psi^{1/2}}
\end{align*}
for a suitably chosen constant $c > 0$ where we used that fact that $\tol = \frac{(\sqrt{2}-1) C}{4}$. 

Thus, we have shown \eqref{eq:mw_apply_condition} and therefore may apply Theorem \ref{thm:mult_weights}, which implies on $\mc{E} \cap \mc{F}$ that
\begin{align*}
\frac{\sum_r \mathbb{E}_\eta \left [ \max_{x \in \X} \frac{(\bar{x}-x)^\top \Asb(\lambda^{(r)})^{-1/2} \eta}{\beta + \bar{\theta}^\t ( \bar{x}-x)} \right ]-\sqrt{\tau} C}{R} & = \frac{\sum_r h_2(\lambda^{(r)})}{R} \\
& \leq  \frac{\sum_r \widehat{h}_2(\lambda^{(r)}) + \tol}{R} \\
 &  \leq 2\tol + \frac{\sum_r \widehat{h}^{(r)}(\lambda^{(r)})  }{R} \\
 & \leq 4 \tol
\end{align*} 
Now, finally, applying Lemma \ref{prop:sb_convex}, we have that
\begin{align*}
\mathbb{E}_\eta \left [ \max_{x \in \X} \frac{(\bar{x}-x)^\top \Asb(\frac{1}{T}\sum_t \lambda^{(r)})^{-1/2} \eta}{\beta + \bar{\theta}^\t ( \bar{x}-x)} \right ]-\sqrt{\tau} C \leq  \frac{\sum_r \mathbb{E}_\eta \left [ \max_{x \in \X} \frac{(\bar{x}-x)^\top \Asb(\lambda^{(r)})^{-1/2} \eta}{\beta + \bar{\theta}^\t ( \bar{x}-x)} \right ]-\sqrt{\tau} C}{R} \leq 4 \tol
\end{align*}
This shows that $\bar{\lambda}^{(R)}$ approximately satisfies one of the constraints; showing approximate satisfaction of the other constraint follows by a similar argument. Thus, we conclude that $\bar{\lambda}^{(R)} \in P_{4 \tol}$.
\end{proof}

The following Lemma shows that Algorithm \ref{alg:binary_search_mult_weight} approximately solves the optimization problem \eqref{eq:opt_fixed_tau}.

\begin{lemma}
\label{lem:bin_search}
Fix $\tau \in > 0$ and let $\delta \in (0,1)$. Let $\optt_\tau$ be the value of
\begin{align*}
\min_{ \lambda \in \triangle} & \tau \sum_{x \in \X} [\beta+\bar{\theta}^\t(\bar{x}-x)]  \lambda_x  \\
& \text{ s.t. } \mathbb{E}_\eta \left [ \max_{x \in \X} \frac{(\bar{x}-x)^\top \Asb(\lambda)^{-1/2} \eta}{\beta + \bar{\theta}^\t ( \bar{x}-x)} \right ]  \leq \sqrt{\tau} C \nonumber.
\end{align*}
If for all $\lambda \in \trianglem$, 
\begin{align*}
\mathbb{E}_\eta \left [ \max_{x \in \X} \frac{(\bar{x}-x)^\top \Asb(\lambda)^{-1/2} \eta}{\beta + \bar{\theta}^\t ( \bar{x}-x)} \right ]  > \sqrt{\tau} C + 4\tol.
\end{align*}
then with probability at least $1-\delta - \frac{1}{\log_2(T) 2^d}$ Algorithm \ref{alg:binary_search_mult_weight} declares the program infeasible.  If  
\begin{align*}
\mathbb{E}_\eta \left [ \max_{x \in \X} \frac{(\bar{x}-x)^\top \Asb(\lambda)^{-1/2} \eta}{\beta + \bar{\theta}^\t ( \bar{x}-x)} \right ]  \leq \sqrt{\tau} C
\end{align*}
then with probability at least $1-\delta - \frac{1}{\log_2(T) 2^d}$ Algorithm \ref{alg:binary_search_mult_weight} returns $\bar{\lambda} \in \trianglem$ such that
\begin{align*}
\tau \sum_{x \in \X} \bar{\theta}^\t (\bar{x}-x) \bar{\lambda}_x & \leq \optt_\tau +4\tol \\
\mathbb{E}_\eta \left [ \max_{x \in \X} \frac{(\bar{x}-x)^\top \Asb(\lambda)^{-1/2} \eta}{\beta + \bar{\theta}^\t ( \bar{x}-x)} \right ]  & \leq  \sqrt{\tau} C + 4\tol.
\end{align*}
Furthermore, Algorithm \ref{alg:binary_search_mult_weight}  uses a number of oracle calls that is upper bounded by $\log_2(2Td/\tol) \cdot \occ(d, \beta,  \psi, \tol, 1/\delta, 1/\xi)$.
\end{lemma}

\begin{proof}
Algorithm \ref{alg:binary_search_mult_weight} applies Algorithm \ref{alg:mult_weight} at most $\log_2(2Td/\tol)$ times on a using a predetermined set of values for $\mwtest \in [0,2Td]$, which we denote $\mwtest_1, \ldots, \mwtest_l$. Define the event 
\begin{align*}
\mc{E}_i & = \{ \text{if MW}(\tau, \mwtest_i)\text{ does not declare infeasibility, then MW(}\tau, \mwtest_i) \text{ returns }\bar{\lambda} \in P_{4 \tol} \\
& \text{ and if MW}(\tau, \mwtest_i) \text{ declares infeasibility,} P_0 \text{ is infeasible.} \} \\
& \cap \{\text{MW}(\tau, \mwtest_i) \text{ uses at most } \occ(d, \beta,  \psi, \tol, 1/\delta, 1/\xi) \text{ oracle calls } \} \\
\mc{E} & = \cap_i \mc{E}_i. 
\end{align*}
where $P_\epsilon$ is defined in Lemma \ref{lem:mw_feasibility}. Then, by the union bound, we have that $\Pr(\mc{E}) \geq 1 - \delta - \frac{1}{2^d \log_2(T)}$. Suppose $\mc{E}$ occurs for the remainder of the proof.

First, consider the case that for all $\lambda \in \trianglem$, 
\begin{align*}
\mathbb{E}_\eta \left [ \max_{x \in \X} \frac{(\bar{x}-x)^\top \Asb(\lambda)^{-1/2} \eta}{\beta + \bar{\theta}^\t ( \bar{x}-x)} \right ]  > \sqrt{\tau} C + 4\tol.
\end{align*}
Then, on the event $\mc{E}$, we have that the Algorithm \ref{alg:binary_search_mult_weight} declares infeasibility of the program.

Now, suppose there exists $\lambda \in \trianglem$ such that
\begin{align*}
\mathbb{E}_\eta \left [ \max_{x \in \X} \frac{(\bar{x}-x)^\top \Asb(\lambda)^{-1/2} \eta}{\beta + \bar{\theta}^\t ( \bar{x}-x)} \right ]  \leq \sqrt{\tau} C.
\end{align*}

Note that for any $\lambda \in \triangle$, we have that
\begin{align*}
\tau \sum_{x \in \X}  [\beta+\bar{\theta}^\t(\bar{x}-x) ] \lambda_x = \beta \tau + \sum_{x \in \X}  \bar{\theta}^\t(\bar{x}-x)  \lambda_x 
\end{align*}
and thus the objective does not depend on $\beta$ and $\beta$ can be dropped from the objective. Using the event $\mc{E}$, if 
\begin{align*}
Q(\mwtest)  :=  \{ \lambda \in \trianglem : & \mathbb{E}_\eta \left [ \max_{x \in \X} \frac{(\bar{x}-x)^\top \Asb(\lambda)^{-1/2} \eta}{\beta + \bar{\theta}^\t ( \bar{x}-x)} \right ]-\sqrt{\tau} C \leq 4 \tol, \\
& \tau \sum_{x \in \X}  \bar{\theta}^\t(\bar{x}-x)  \lambda_x - \mwtest \leq 4\tol \}
\end{align*}
is empty, then Algorithm \ref{alg:mult_weight} declares the program infeasible; otherwise, Algorithm \ref{alg:mult_weight} finds $\bar{\lambda} \in Q(\mwtest)$. Then, by a standard binary search argument, the result follows. 
\end{proof}

The following Theorem establishes that Algorithm \ref{alg:main_compute} approximately solves the main optimization problem \eqref{eq:opt_problem}. It directly implies Theorem \ref{thm:comp_complex}.

\begin{theorem}\label{thm:main_comp}
Let $\delta \in (0,1)$. Suppose $\tol = \frac{(\sqrt{2}-1)C}{4}$, $\psi = \min(\frac{1}{4 d \Delmax T}, \frac{1}{4 d})$. Let $\opt$ be the value of
\begin{align}
\min_{\tau \in [T], \lambda \in \triangle} & \tau \sum_{x \in \X}[\epsilon + \bar{\theta}^\t(\bar{x}-x) ] \lambda_x \label{eq:main_opt_thm} \\
& \text{ s.t. } \mathbb{E}_\eta \left [ \max_{x \in \X} \frac{(\bar{x}-x)^\top \Asb(\lambda)^{-1/2} \eta}{\beta + \bar{\theta}^\t ( \bar{x}-x)} \right ]  \leq \sqrt{\tau} C \nonumber.
\end{align}
With probability at least $1-\delta - \frac{1}{2^d}$, Algorithm \ref{alg:main_compute} returns $(\bar{\tau},\bar{\lambda})$ such that $\bar{\lambda} \in \triangle$, $\bar{\tau} \leq 2 T$, and 
\begin{align*}
\bar{\tau} \sum_{x \in \X} [\epsilon + \bar{\theta}^\t(\bar{x}-x)]  \bar{\lambda}_x & \leq 4\opt + 2 \\
\mathbb{E}_\eta \left [ \max_{x \in \X} \frac{(\bar{x}-x)^\top \Asb(\bar{\lambda})^{-1/2} \eta}{\beta + \bar{\theta}^\t ( \bar{x}-x)} \right ]  & \leq \sqrt{\bar{\tau}} C \nonumber.
\end{align*}
Furthermore, Algorithm  \ref{alg:main_compute} uses a number of oracle calls that is polynomial in $(d, \beta, \psi, \log(1/\delta))$
\end{theorem}

\begin{proof}
\textbf{Step 0.} Let $\optt$ be the value of
\begin{align*}
\min_{\tau \in [T], \lambda \in \trianglem} & \tau \sum_{x \in \X} [\beta+\bar{\theta}^\t(\bar{x}-x) ] \lambda_x  \\
& \text{ s.t. } \mathbb{E}_\eta \left [ \max_{x \in \X} \frac{(\bar{x}-x)^\top \Asb(\lambda)^{-1/2} \eta}{\beta + \bar{\theta}^\t ( \bar{x}-x)} \right ]  \leq \sqrt{\tau} C \nonumber.
\end{align*}
and let $\optt_k$ be the value of
\begin{align*}
\min_{\lambda \in \triangle} & \bar{\tau}_k \sum_{x \in \X} [\beta+ \bar{\theta}^\t(\bar{x}-x)]  \lambda_x  \\
& \text{ s.t. } \mathbb{E}_\eta \left [ \max_{x \in \X} \frac{(\bar{x}-x)^\top \Asb(\lambda)^{-1/2} \eta}{\beta + \bar{\theta}^\t ( \bar{x}-x)} \right ]  \leq \sqrt{\bar{\tau}}_k C \nonumber.
\end{align*}

Let $\mc{E}_k$ denote the event that if for all $\lambda \in \trianglem$, 
\begin{align*}
\mathbb{E}_\eta \left [ \max_{x \in \X} \frac{(\bar{x}-x)^\top \Asb(\lambda)^{-1/2} \eta}{\beta + \bar{\theta}^\t ( \bar{x}-x)} \right ]  > \sqrt{\bar{\tau}_k} C + 4\tol.
\end{align*}
then $\text{binSearch}(\bar{\tau}_k, \frac{\delta}{\log_2(T)})$ declares the program infeasible and if  
\begin{align*}
\mathbb{E}_\eta \left [ \max_{x \in \X} \frac{(\bar{x}-x)^\top \Asb(\lambda)^{-1/2} \eta}{\beta + \bar{\theta}^\t ( \bar{x}-x)} \right ]  \leq \sqrt{\tau} C
\end{align*}
then $\text{binSearch}(\bar{\tau}_k, \frac{\delta}{\log_2(T)})$ returns $\bar{\lambda}_k$ that satisfies
\begin{align*}
\tau \sum_{x \in \X}[\beta+ \bar{\theta}^\t (\bar{x}-x) \bar{\lambda}_{k,x}] & \leq \optt_{\bar{\tau}_k} +4\tol \\
\mathbb{E}_\eta \left [ \max_{x \in \X} \frac{(\bar{x}-x)^\top \Asb(\bar{\lambda}_k)^{-1/2} \eta}{\beta + \bar{\theta}^\t ( \bar{x}-x)} \right ]  & \leq  \sqrt{\tau} C + 4\tol.
\end{align*}
Further, define $\mc{E}  = \cap_k \mc{E}_k$. By Lemma \ref{lem:bin_search} and a union bound, we have that $\Pr(\mc{E}) \geq 1-\delta - \frac{1}{2^d}$. We suppose $\mc{E}$ holds for the rest of the proof.

\textbf{Step 1. } 
First, we show that Algorithm \ref{alg:main_compute} returns $(\bar{\tau},\bar{\lambda})$ such that
\begin{align*}
\bar{\tau} \sum_{x \in \X} [\beta+ \bar{\theta}^\t(\bar{x}-x)]  \bar{\lambda}_x & \leq \optt + 4 \tol \\
\mathbb{E}_\eta \left [ \max_{x \in \X} \frac{(\bar{x}-x)^\top \Asb(\bar{\lambda})^{-1/2} \eta}{\beta + \bar{\theta}^\t ( \bar{x}-x)} \right ] & \leq \sqrt{\bar{\tau}} C \nonumber.
\end{align*}
By assumption the optimization problem in \eqref{eq:opt_problem} is feasible and, hence, $\opt \neq \infty$ and thus by the event $\mc{E}$, the algorithm finds at least one nearly feasible solution, i.e., $\textsc{feasible}_k$ is not False for all $k$. Let $(\tau_*,\lambda_*)$ attain the optimal value in the optimization problem \eqref{eq:opt_problem}. Let $k_*$ such that $\bar{\tau}_{k_*} \in [\tau_*, 2 \tau_*]$. By event $\mc{E}$ $\text{binSearch}(\bar{\tau}_{k_*}, \frac{\delta}{\log_2(T)})$ finds $\bar{\lambda}_{k_*}$ such that
\begin{align*}
\bar{\tau}_{k_*} \sum_{x \in \X}  [\beta+\bar{\theta}^\t(\bar{x}-x) ] \bar{\lambda}_{k_*,x} & \leq \optt_{k_*} + 4\tol \\
 \mathbb{E}_\eta \left [ \max_{x \in \X} \frac{(\bar{x}-x)^\top \Asb(\bar{\lambda}_{k_*})^{-1/2} \eta}{\beta + \bar{\theta}^\t ( \bar{x}-x)} \right ]  & \leq \sqrt{\bar{\tau}}_{k_*} C \nonumber + 4\tol.
\end{align*}
Algorithm \ref{alg:main_compute} outputs $(\bar{\tau}, \lambda_{\widehat{k}_*})$, which satisfies by Lemma \ref{lem:bin_search} and by construction,
\begin{align}
  \bar{\tau}\sum_{x \in \X}[\beta+ \bar{\theta}^\t(\bar{x}-x)  \bar{\lambda}_{\widehat{k}_*,x}]& = 2\bar{\tau}_{\widehat{k}_*} \sum_{x \in \X}[\beta+ \bar{\theta}^\t(\bar{x}-x)  \bar{\lambda}_{\widehat{k}_*,x}] \\
  & \leq 2\bar{\tau}_{k_*} \sum_{x \in \X} [\beta+\bar{\theta}^\t(\bar{x}-x)  ]\bar{\lambda}_{k_*,x} \nonumber \\
 & \leq 2[\optt_k + 4\tol] \nonumber \\
 & \leq 2 \optt_k + 1 \label{eq:alg_opt_1} 
\end{align}
where in the last line we used $\tol = \frac{(\sqrt{2}-1)C}{4} \leq 1/8$, which bounds the objective value of $(\bar{\tau}, \lambda_{\widehat{k}_*})$.

Next, we show feasiblity of $(\bar{\tau}, \lambda_{\widehat{k}_*})$. Observe that
\begin{align}
 \mathbb{E}_\eta \left [ \max_{x \in \X} \frac{(\bar{x}-x)^\top \Asb(\bar{\lambda}_{k_*})^{-1/2} \eta}{\beta + \bar{\theta}^\t ( \bar{x}-x)} \right ]  & \leq \sqrt{\bar{\tau}}_{\widehat{k}_*} C  + 4\tol  \nonumber  \\
 & \leq \sqrt{2\bar{\tau}}_{\widehat{k}_*} C  \nonumber  \\
 & = \sqrt{\bar{\tau}} C \label{eq:alg_feas_1}
\end{align}
where we used the fact that $\bar{\tau } = 2\bar{\tau}_{\widehat{k}_*}$ and $\tol = \frac{(\sqrt{2}-1)C}{4}  $.

\textbf{Step 2: Relate $\optt_{k}$ to $\optt$.} Next, we show that 
\begin{align*}
\optt_{k_*} \leq 2\optt.
\end{align*}
Define the function
\begin{align*}
f(\lambda, \tau) = & \tau \sum_{x \in \X}[\beta+ \bar{\theta}^\t(\bar{x}-x) ] \lambda_x  \\
& \text{ s.t. } \mathbb{E}_\eta \left [ \max_{x \in \X} \frac{(\bar{x}-x)^\top \Asb(\lambda)^{-1/2} \eta}{\beta + \bar{\theta}^\t ( \bar{x}-x)} \right ]  \leq \sqrt{\tau} C \nonumber.
\end{align*}
Recall that we let $(\tau_*,\lambda_*)$ attain the optimal value in the optimization problem \eqref{eq:opt_problem}. Let $k_*$ such that $\bar{\tau}_{k_*} \in [\tau_*, 2 \tau_*]$. Note that 
\begin{align*}
 \mathbb{E}_\eta \left [ \max_{x \in \X} \frac{(\bar{x}-x)^\top \Asb(\lambda_*)^{-1/2} \eta}{\beta + \bar{\theta}^\t ( \bar{x}-x)} \right ]  \leq \sqrt{\bar{\tau}_{k_*}} C
\end{align*}
Thus,
\begin{align}
\optt_{k_*} = f(\bar{\lambda}_k, \bar{\tau}_k) \leq f(\lambda_*, \bar{\tau}_k) \leq 2f(\lambda_*,\tau_*) = 2\optt, \label{eq:relate_opt_1}
\end{align}
where we used the fact that
\begin{align*}
\mathbb{E}_\eta \left [ \max_{x \in \X} \frac{(\bar{x}-x)^\top \Asb(\lambda_*)^{-1/2} \eta}{\beta + \bar{\theta}^\t ( \bar{x}-x)} \right ]  \leq \sqrt{\tau_*} C \leq \sqrt{\bar{\tau}_{k_*} } C.
\end{align*}
This proves the claim.

\textbf{Step 3: Relate $\optt$ to $\opt$.} Next, we show that
\begin{align*}
\optt \leq 2 \opt + T \psi d \Delmax.
\end{align*}

Define
\begin{align*}
\check{\lambda}_i = \begin{cases} \frac{1}{d} : i \in [d] \\
0 : i \not \in [d] 
\end{cases}.
\end{align*}
and
\begin{align*}
\tilde{\lambda} & = \psi d \check{\lambda} + (1- \psi d) \lambda^* \\
\tilde{\tau} & = 2 \tau_*.
\end{align*}
By the hypothesis, we have that $\psi d \leq \frac{1}{4}$ and, thus, $\tilde{\lambda}$ is a convex combination of $\check{\lambda}$ and $\lambda^*$. 

Next, we show that $(\tilde{\lambda}, \tilde{\tau})$ are a feasible solution to \eqref{eq:main_opt_thm} and show that it is approximately optimal.  Note that
\begin{align*}
\Asb(\tilde{\lambda}) \geq (1- \psi d) \Asb(\lambda^*),
\end{align*} 
which implies that
\begin{align*}
\frac{1}{1-\psi d} \Asb(\lambda^*)^{-1} \geq \Asb(\tilde{\lambda})^{-1}.
\end{align*}
Then, by Sudakov-Fernique, we have that
\begin{align*}
 \mathbb{E}_\eta \left [ \max_{x \in \X} \frac{(\bar{x}-x)^\top \Asb(\tilde{\lambda})^{-1/2} \eta}{\beta + \bar{\theta}^\t ( \bar{x}-x)} \right ]  & \leq [1-\psi d]^{-1/2} \mathbb{E}_\eta \left [ \max_{x \in \X} \frac{(\bar{x}-x)^\top \Asb(\lambda^*)^{-1/2} \eta}{\beta + \bar{\theta}^\t ( \bar{x}-x)} \right ] \\
 & \leq  \sqrt{2} \mathbb{E}_\eta \left [ \max_{x \in \X} \frac{(\bar{x}-x)^\top \Asb(\lambda^*)^{-1/2} \eta}{\beta + \bar{\theta}^\t ( \bar{x}-x)} \right ] \\
  & \leq \sqrt{2 \tau^*} C \\
 & = \sqrt{\tilde{\tau}} C
\end{align*}
showing feasibility $(\tilde{\lambda}, \tilde{\tau})$. Furthermore, we have that
\begin{align}
\bar{\tau} \sum_{x \in \X} \tilde{\lambda}_x [\bar{\theta}^\t (\bar{x}-x)+\beta] & \leq 2 \opt + \tilde{\tau} \psi d \sum_{x \in \X} \tilde{\lambda}_x[\bar{\theta}^\t (\bar{x}-x) +\beta] \nonumber \\
& \leq 2 \opt + T \psi d 2\Delmax \nonumber \\
& \leq 2 \opt + 1 \label{eq:relate_opt_2}
\end{align}
where in the last line we used $\psi = \min(\frac{1}{4 d \Delmax T}, \frac{1}{4 d})$.

\textbf{Step 4: Putting it together.} Putting together \eqref{eq:alg_feas_1}, \eqref{eq:alg_opt_1}, \eqref{eq:relate_opt_1}, and \eqref{eq:relate_opt_2}, we have that Algorithm \ref{alg:main_compute} returns $(\bar{\tau},\bar{\lambda})$ such that $\bar{\lambda} \in \triangle$, $\bar{\tau} \leq 2 T$, and 
\begin{align*}
\bar{\tau} \sum_{x \in \X} [\beta+\bar{\theta}^\t(\bar{x}-x)]  \bar{\lambda}_x & \leq 4\opt + 2 \\
\mathbb{E}_\eta \left [ \max_{x \in \X} \frac{(\bar{x}-x)^\top \Asb(\bar{\lambda})^{-1/2} \eta}{\beta + \bar{\theta}^\t ( \bar{x}-x)} \right ]  & \leq \sqrt{\bar{\tau}} C \nonumber.
\end{align*}

\end{proof}

\subsection{Miscellaneous Optimization Lemmas}

\begin{lemma}
\label{lem:estimate_sup}
Let $\lambda \in \trianglem$. With probability at least $1-\delta - \frac{\xi}{2^d}$, Algorithm \ref{alg:est_exp_sup} returns $\widehat{\mu}$ such that
 \begin{align*}
 |\widehat{\mu}-\mathbb{E}_\eta \left [ \max_{x \in \X} \frac{(\bar{x}-x)^\top \Asb(\lambda)^{-1/2} \eta}{\beta + \bar{\theta}^\t ( \bar{x}-x)} \right ]| \leq \tol
\end{align*}  
and the number of linear maximization oracle calls is bounded above by
\begin{align*}
O(\log(1/\delta) \frac{d}{\beta^2 \psi \textsc{tol}^2}[d  + \log(\frac{d\Delmax}{\beta \xi})]).
\end{align*}
\end{lemma}

\begin{proof}
We first show that
\begin{align*}
\max_{x \in \X} \frac{(\bar{x}-x)^\top \Asb(\lambda)^{-1/2} \eta}{\beta + \bar{\theta}^\t ( \bar{x}-x)}  \in \sg(c \frac{d}{\beta^2 \psi}).
\end{align*}
 Note that 
\begin{align*}
|\max_{x \in \X} \frac{(\bar{x}-x)^\top \Asb(\lambda)^{-1/2} \eta}{\beta + \bar{\theta}^\t ( \bar{x}-x)}| \leq \frac{1}{\beta \psi^{1/2}} \sum_{i=1}^d |\eta_i|
\end{align*}
and 
\begin{align*}
\frac{1}{\beta \psi^{1/2}} \sum_{i=1}^d |\eta_i| \in \sg( c \frac{d}{\beta^2 \psi}).
\end{align*} 
The estimation results by applying a standard subGaussian tail bound. The bound on the number of oracle calls follows since Algorithm \ref{alg:compute_max} is applied $O( \log(1/\delta) \frac{d}{\beta^2 \psi \textsc{tol}^2})$ times and by Lemma \ref{lem:find_argmax} and a union bound.
\end{proof}

The following Lemma shows that the binary search procedure in Algorithm \ref{alg:compute_max} is efficient with very high probability and it follows immediately from the proof of Lemma 2 of \cite{katz2020empirical}.

\begin{lemma}
\label{lem:find_argmax}
Draw $\eta \sim N(0,I)$ and consider the optimization problem
\begin{align*}
\tilde{x} = \argmax_{x \in \X} \frac{(\bar{x}-x)^\top \Asb(\lambda)^{-1/2} \eta}{\beta + \bar{\theta}^\t ( \bar{x}-x)}.
\end{align*}
With probability at least $1-\frac{2\xi}{2^d} $, Algorithm \ref{alg:compute_max} returns $\tilde{x}$ using at most $O(d  +  \log(\frac{d\Delmax}{\beta \xi}))$ oracle calls.
\end{lemma}

Next, we describe a result on the multiplicative weights update algorithm that follows immediately from Corollary 4 in \cite{arora2012multiplicative}. Consider the experts problem. The set of events is denoted by $P$. Suppose there are $m$ experts. At each round $t$, the agent picks an expert $i \in [m]$ and the adversary picks an outcome $j^t \in P$ and the agent obtains reward $M(i,j^t)$. The multiplicative weights update algorithm mains a distribution $D^t$ over the experts and chooses an expert randomly from $D^t$ (see \cite{arora2012multiplicative} for details on how this distribution is chosen). The adversary may have knowledge of the $D^t$ when choosing $j^t$. The following provides a lower bound on the expected reward obtained by the multiplicative weights update algorithm.
\begin{theorem}
\label{thm:mult_weights}
Let $\xi>0$ denote an error parameter. Suppose there are $m$ experts and $|M(i,j)| \leq \rho$. If the multiplicative weights algorithm sets the learning rate as $\epsilon = \min( \frac{\xi}{4 \rho}, \frac{1}{2})$, after $T = \frac{16 \rho^2 \ln(m)}{\xi^2}$, then the multiplicative weights algorithm achieves the following bound on its average expected reward: for any expert $i$, 
\begin{align*}
\frac{\sum_t M(i, j^t)}{T} \leq \xi + \frac{\sum_t M(D^t,j^t)}{T}.
\end{align*}
\end{theorem}

\subsection{Convergence Lemmas}

The objective in semi-feedback is convex (by a similar argument to the proof in \cite{katz2020empirical}).
\begin{proposition}
\label{prop:sb_convex}
Fix $V \subset \R^d$. 
\begin{align*}
f(\lambda) =  \E_{\eta \sim N(0,I)}[ \max_{v \in V} v^\t \Asb(\lambda)^{-1/2} \eta]
\end{align*}
is convex.
\end{proposition}

\begin{proof}
Fix $\lambda, \kappa \in \triangle^{|\X|}$ and $\alpha \in [0,1]$. By matrix convexity,
\begin{align*}
\diag(\frac{1}{\sum_{x \in \X} \alpha\lambda_{x,i} +(1-\alpha)\kappa_{x,i}})^{1/2} \preceq \alpha \diag(\frac{1}{\sum_{x \in \X} \lambda_{x,i} })^{1/2} +(1-\alpha)\diag(\frac{1}{\sum_{x \in \X} \kappa_{x,i} })^{1/2} .
\end{align*}
Furthermore, since the above matrices are diagonal,
\begin{align*}
\diag(\frac{1}{\sum_{x \in \X} \alpha\lambda_{x,i} +(1-\alpha)\kappa_{x,i}}) \preceq [\alpha \diag(\frac{1}{\sum_{x \in \X} \lambda_{x,i} })^{1/2} +(1-\alpha)\diag(\frac{1}{\sum_{x \in \X} \kappa_{x,i} })^{1/2}]^2 .
\end{align*}
Then, by Sudakov-Fernique inequality (Theorem 7.2.11 in \cite{vershynin2018high}),
\begin{align*}
f(\alpha\lambda +(1-\alpha)\kappa) & = \E_{\eta \sim N(0,\diag(\frac{1}{\sum_{x \in \X} \alpha\lambda_{x,i} +(1-\alpha)\kappa_{x,i}}))} \sup_{v \in V} v^\t \eta \\
& \leq \E_{\eta \sim N(0,[\alpha \diag(\frac{1}{\sum_{x \in \X} \lambda_{x,i} })^{1/2} +(1-\alpha)\diag(\frac{1}{\sum_{x \in \X} \kappa_{x,i} })^{1/2}]^2)} \sup_{v \in V} z^\t \eta \\
& = \E_{\eta \sim N(0,I)} \sup_{v \in V} v^\t [\alpha \diag(\frac{1}{\sum_{x \in \X} \lambda_{x,i} })^{1/2} +(1-\alpha)\diag(\frac{1}{\sum_{x \in \X} \kappa_{x,i} })^{1/2}]\eta \\
& \leq  \alpha \E_{\eta \sim N(0,I)} \sup_{v \in V} v^\t  \diag(\frac{1}{\sum_{x \in \X} \lambda_{x,i} })^{1/2}  \eta \\
& +  (1-\alpha)\E_{\eta \sim N(0,I)} \sup_{v \in V} v^\t \diag(\frac{1}{\sum_{x \in \X} \kappa_{x,i} })^{1/2}\eta \\
& = \alpha f(\lambda ) + (1-\alpha)f(\kappa) 
\end{align*}

\end{proof}

Next, we turn to analyzing stochastic Frank-Wolfe. Although a convergence result for stochastic frank wolfe is provided in \cite{hazan2016variance}, our setup is slightly different, so we include a convergence analysis for our setting for the sake of completeness. The proof is quite similar to the proof in \cite{hazan2016variance}.

\begin{algorithm}[H] 
\begin{algorithmic}[1]
\State \textbf{Input:} $f : \R^m \times \R^d \longrightarrow \R$, constraint set $\Omega \subset \R^m$, $(p_r)_{r } \in \N^\infty$, $(q_r)_{r } \in [0,1]^\infty$. 
\State Initialize $w_1 \in \Omega$ 
\For{$r=1,2, \ldots$}
 	 \State Draw $\eta_1, \ldots, \eta_{p_r} \sim N(0,I)$
	  \State Compute
	  \begin{align*}
	\tilde{\nabla}_r = \frac{1}{p_r} \sum_{j=1}^{p_r} \nabla f(w_r; \eta_j)
	\end{align*}  
	 \State Compute
	 \begin{align*}
	 v_r = \argmin_{v \in \Omega} \tilde{\nabla}_r^\t v
	 \end{align*}
\State
	 \begin{align*}
	 w_{r+1} \longleftarrow q_r v_r +(1-q_r) w_r
	 \end{align*}
\EndFor
\end{algorithmic}
\caption{Generic Stochastic Frank-Wolfe}
\label{alg:gen_sfw}
\end{algorithm}

\begin{proposition}
\label{prop:sfw_gen}
Let $f: \R^m \times \R^d \longrightarrow \R$ and $\Omega \subset \R^m$.	Define $f(x) = \E_{\eta \sim N(0,I)} f(x;\eta)$ and define
\begin{align*}
w^* = \argmin_{w \in \Omega} \E_\eta f(x;\eta).
\end{align*} 
Suppose that $\sup_{w,w^\prime \in \Omega} \norm{w-w^\prime} \leq D$. Suppose that $f$ is convex,  $\norm{\nabla f(x) - \nabla f(y)}_* \leq L \norm{x-y}$, and $p_r$ in Algorithm \ref{alg:gen_sfw} is chosen such that with probability at least $1-\delta/r^2$
\begin{align*}
\norm{\wt{\nabla}_r - \nabla f(w_{r-1})}_* \leq \frac{LD q_r}{2}
\end{align*}
where $q_r = \frac{2}{k+1}$. Then, with probability at least $1-c\delta$,
\begin{align*}
f(w_r) - f(w_*) \leq \frac{4 LD^2}{r+2}.
\end{align*}
\end{proposition}

\begin{proof}
The proof follows closely the analysis of SFW in \cite{hazan2016variance} but uses smoothness wrt $\norm{\cdot}_*$. We have that
\begin{align}
f(w_r) & \leq f(w_{r-1}) + \nabla f(w_{r-1})^\t (w_r - w_{r-1}) + \frac{L}{2} \norm{w_r - w_{r-1}}_1^2 \label{eq:sfw_smoothness} \\
& =  f(w_{r-1}) + q_r \nabla f(w_{r-1})^\t (v_r - w_{r-1}) + \frac{Lq_r^2}{2} \norm{v_r - w_{r-1}}_1^2 \nonumber \\
& \leq f(w_{r-1}) + q_r \wt{\nabla}_r^\t (v_r - w_{r-1}) + q_r(\nabla f(w_{r-1}) - \wt{\nabla}_r)^\t (v_r - w_{r-1}) + \frac{L D^2 q_r^2}{2} \nonumber \\
& \leq f(w_{r-1}) + q_r \wt{\nabla}_r^\t (w_* - w_{r-1}) + q_r(\nabla f(w_{r-1}) - \wt{\nabla}_r)^\t (v_r - w_{r-1}) + \frac{L D^2 q_r^2}{2} \label{eq:sfw_optim}\\
& = f(w_{r-1}) + q_r \nabla f(w_{r-1})^\t (w_* - w_{r-1}) + q_r(\nabla f(w_{r-1}) - \wt{\nabla}_r)^\t (v_r - w_*) + \frac{L D^2 q_r^2}{2} \nonumber \\
& \leq f(w_{r-1}) + q_r \nabla f(w_{r-1})^\t (w_* - w_{r-1}) + q_r \norm{\nabla f(w_{r-1}) - \wt{\nabla}_r}_* D + \frac{L D^2 q_r^2}{2} -\label{eq:sfw_dual}
\end{align}
where line \eqref{eq:sfw_smoothness} uses smoothness (Lemma \ref{lem:smoothness}), line \eqref{eq:sfw_optim} uses the optimality of $v_r$, and line \eqref{eq:sfw_dual} uses the definition of the dual norm. Now, define the event 
\begin{align*}
\mc{E}_r = \{\norm{\wt{\nabla}_r - \nabla f(w_{r-1})}_* \leq \frac{LD q_r}{2}\}.
\mc{E} = \cap_r \mc{E}_r
\end{align*}
By hypothesis,  $p_r$ is chosen such that with probability at least $1-\delta/r^2$, $ \norm{\wt{\nabla}_r - \nabla f(w_{r-1})}_* \leq \frac{LD q_r}{2}$. Therefore, we have that
\begin{align*}
\Pr(\mc{E}) & = \prod_{r=1}^\infty \Pr(\mc{E}_r | \cap_{s=1}^{r-1} \mc{E}_s) \geq \prod_{r=1}^\infty (1-\frac{\delta}{ r^2}) = \frac{\sin(\pi \delta)}{\pi+\delta} \geq 1 - \delta.
\end{align*}
Now, suppose $\mc{E}$ occurs. Then, we have that for all $r \in \N$,
\begin{align*}
f(w_r) - f(w_{*}) \leq  (1 - q_r) [f(w_{r-1}) - f(w_{*})] + L D^2 q_r^2.
\end{align*}
The proof is concluded by simple induction.
\end{proof}

The following Lemma shows that Algorithm \ref{alg:sfw_sb} is an instantiation of stochastic Frank-Wolfe over $\trianglem$.

\begin{lemma}
Fix $v \in \R^m$. 
Let
\begin{align*}
I := \argmin_{i \in [m]} v_i.
\end{align*}
Define 
\begin{align*}
\bar{\lambda}_i & = \begin{cases}
\begin{cases}
0 & i \not \in [d] \\
\psi & i \in [d] \setminus \{ I \} \\
1- (d-1) \psi & i = I
\end{cases}, \qquad \qquad I \in [d] \\
\begin{cases}
\psi &  \hspace{1.2cm}  i \in [d] \\
1- d \psi  & \hspace{1.2cm} i = I
\end{cases}, \qquad \qquad  \qquad I \not \in [d]
\end{cases}
\end{align*}
Then, $\bar{\lambda} \in \argmin_{\lambda \in \trianglem} v^\t \lambda$.
\end{lemma}

\begin{proof}
This follows by a straightforward case by case analysis.
\end{proof}

The following is standard smoothness Lemma from convex optimization.

\begin{lemma}\label{lem:smoothness}
Let $f: \R^m \longrightarrow \R$ satisfy $\norm{\nabla f(x) - \nabla f(y)}_* \leq L \norm{x-y}$. Then,
\begin{align*}
f(x) - f(y) - \nabla f(y)^\t(-y) \leq \frac{L}{2} \norm{x-y}^2.
\end{align*}
\end{lemma}

\begin{proof}
This is standard (see \cite{bubeck2014convex}).
\end{proof}

\subsection{Differentiability Lemmas}

In this section, we show that $\L(\kappa_1, \kappa_2 ; \lambda)$ is twice-differentiable wrt $\lambda$. We set $\kappa_1, \kappa_2 = 1$ for simplicity and write $\L( \lambda)$ instead of $\L(\kappa_1, \kappa_2 ; \tau; \lambda)$ for the sake of brevity. The following Lemma shows that $\L(\kappa_1, \kappa_2 ; \tau; \lambda)$ is differentiable wrt $\lambda$.

\begin{lemma}
\label{lem:differentiable}
Fix $i \in [m]$, and $\lambda \in {\trianglem}$. Fix $\eta \in \R^d$ such there exists a neighborhood of $\eta$ such that 
\begin{align*}
\tilde{x} = \argmax_{x \in \X} \frac{\sum_{i \in \bar{x} \Delta x} \frac{\eta_i}{\sum_{x^\prime : i \in x^\prime} \lambda_{x^\prime} }}{\beta+ \bar{\theta}^\t(\bar{x}-x)} .
\end{align*}
Then,
\begin{align*}
\frac{\partial \L(  \lambda ; \eta)}{\partial \lambda_i } & = \tau\bar{\theta}^\t(\bar{x}- x_i) -\frac{1}{2} \frac{1}{[\beta + \bar{\theta}^\t(\bar{x}-\tilde{x}) ]}  \sum_{k \in (\bar{x} \Delta \tilde{x}) \cap x_i } \frac{\eta_k}{(\sum_{j : k \in x_j} \lambda_j )^{3/2}}   
\end{align*}
Furthermore, $ \L( \lambda)$ is differentiable at every $\lambda \in {\trianglem}_\psi$ and
\begin{align*}
\frac{\partial \L( \lambda)}{\partial \lambda_i } & =\E_{ \eta \sim N(0,I)} [\frac{\partial \L( \lambda ; \eta)}{\partial \lambda_i } \one\{B_\lambda\} ] 
\end{align*}
where 
\begin{align*}
B_\lambda = \{\eta :  | \argmax_{x \in \X} \frac{\sum_{i \in \bar{x} \Delta x} \frac{\eta_i}{\sum_{x^\prime : i \in x^\prime} \lambda_{x^\prime} }}{\beta+ \bar{\theta}^\t(\bar{x}-x)}| =1 \}.
\end{align*}
\end{lemma}

\begin{proof}
The calculation of $\frac{\partial \L( \lambda ; \eta)}{\partial \lambda_i } $ follows by the chain rule. 

Fix $\lambda  \in {\trianglem}$. Since $\lambda \in {\trianglem}$, we have that $\Asb(\lambda)^{-1/2}$ is full rank.

\textbf{Step 1:} First, we show that $\L(\lambda;\eta)$ is Lipschitz with an absolutely integrable Lipschitz constant. Define
\begin{align*}
\J(\lambda;\eta;x) & = \tau \sum_{x \in \X}  \bar{\theta}^\t(\bar{x}-x)  \lambda_x + ( \frac{x^\top \Asb(\lambda)^{-1/2} \eta}{\beta + \bar{\theta}^\t ( \bar{x}-x)}  - \sqrt{\tau} C).
\end{align*}
and note that
\begin{align*}
|\frac{\partial \J(\lambda;\eta;x) }{ \partial \lambda_i} |& = | \bar{\theta}^\t(\bar{x}- x_i) -\frac{1}{2} \frac{1}{[\beta + \bar{\theta}^\t(\bar{x}-x) ]}  \sum_{k \in (\bar{x} \Delta x) \cap x_i } \frac{\eta_k}{(\sum_{j : k \in x_j} \lambda_j )^{3/2}}| \\
& \leq  | \bar{\theta}^\t(\bar{x}- x_i) | + \frac{1}{2} \frac{1}{[\beta + \bar{\theta}^\t(\bar{x}-x) ]}  \sum_{k \in (\bar{x} \Delta x) \cap x_i } \frac{|\eta_k|}{\psi^{3/2}}\\
& < | \bar{\theta}^\t(\bar{x}- x_i) | + \frac{1}{2} \frac{1}{\beta  }  \sum_{k \in (\bar{x} \Delta x) \cap x_i } \frac{|\eta_k|}{\psi^{3/2}} := C_\eta 
\end{align*}
 Let $\lambda,\lambda^\prime \in {\trianglem}$.  Thus, by the mean value theorem, we have that for all $x \in \X$,
\begin{align*}
|\J(\lambda;\eta;x) - \J(\lambda;\eta;x)| \leq C_\eta \norm{\lambda - \lambda^\prime}_1
\end{align*}
Since $\L(\lambda;\eta) := \max_{x \in \X} \J(\lambda;\eta;x) $ and the maximum of $C_\eta$-Lipschitz functions is $C_\eta$-Lipschitz, we have that
\begin{align*}
|\L(\lambda;\eta) - \L(\lambda;\eta)| \leq C_\eta \norm{\lambda - \lambda^\prime}_1
\end{align*}

\textbf{Step 2:} Now, we show that the partial derivatives exist. Define the event
\begin{align*}
B_\lambda = \{\eta :  |\argmax_{x \in \X} \frac{(\bar{x}-x)^\top \Asb(\lambda)^{-1/2} \eta}{\beta+ \bar{\theta}^\t(\bar{x}-x)}| =1 \},
\end{align*}
Since $\Asb(\lambda)^{-1/2}$ is full rank and each 
\begin{align*}
\frac{x}{\beta+ \bar{\theta}^\t(\bar{x}-x)}
\end{align*}
is distinct, if $\eta \sim N(0,I)$, then with probability $1$ $B_\lambda$ holds and $ \L(   \lambda; \eta )$ is differentiable at $\lambda$.

Since $\L(\lambda;\eta)$ is $C_\eta$-Lipschitz (because $\lambda \in \trianglem$, we have
\begin{align*}
|\frac{\L( \lambda + e_i h; \eta ) -L(  \lambda ; \eta )  }{h}| \leq C_\eta.
\end{align*}
Since in addition $\E C_\eta < \infty$, by the dominated convergence theorem,
\begin{align*}
 \lim_{h\longrightarrow 0} \E [\frac{\L( \lambda + e_i h; \eta ) -L(  \lambda ; \eta )  }{h} ] & = \lim_{h\longrightarrow 0} \E[\frac{\L(  \lambda + e_i h; \eta ) -L(  \lambda ; \eta )  }{h} \one\{B_\lambda\} ]\\
 & =  \E[\lim_{h\longrightarrow 0} \frac{\L(  \lambda + e_i h; \eta ) -L(  \lambda ; \eta )  }{h} \one\{B_\lambda\} ]\\
& = \E [\nabla \L(  \lambda; \eta )^\t e_i  \one\{B_\lambda\} ] \\
\end{align*}
where the last equality follows since on $B_\lambda$ and $\lambda \in \trianglem$, $\frac{\partial \L(\lambda ; \eta)}{\partial \lambda_i}$ exists. Thus, the partial derivative $\frac{\partial \L( \lambda)}{\partial \lambda_i } $ exists at every point $\lambda \in {\trianglem}$ and
\begin{align*}
\frac{\partial \L( \lambda)}{\partial \lambda_i } & =\E [\nabla \L(\lambda; \eta )^\t e_i  \one\{B_\lambda\} ] 
\end{align*}
\textbf{Step 3:} We claim that the partial derivative is continuous in $\lambda \in {\trianglem}$, which would show that that $\L( \lambda)$ is differentiable at every $\lambda \in {\trianglem}$ \cite{munkres2018analysis}. Let $\lambda^{(n)}$ be a sequence in ${\trianglem}$ such that $\lambda^{(n)} \longrightarrow \lambda$. Note that since $\lambda^{(n)} \in {\trianglem}$, we have that
\begin{align*}
\nabla \L(  \lambda^{(n)}; \eta )^\t e_i  \one\{B_{\lambda^{(n)}}\}  & =  | \bar{\theta}^\t(\bar{x}- x_i)| +c \frac{1}{[\beta + \bar{\theta}^\t(\bar{x}-\tilde{x}) ]}  \sum_{k \in (\bar{x} \Delta \tilde{x}) \cap x_i }  \frac{|\eta_k| }{\psi^{3/2}}
\end{align*}
for an appropriate universal constant $c > 0$, which has finite expectation. Further, since $\lambda^{(n)} \in {\trianglem}$, the calculation showing that $\L(\lambda;\eta)$ is Lipschitz in $\lambda$ implies that $\Asb(\lambda)^{-1/2}$ is Lipschitz in $\lambda$, so $\Asb(\lambda^{(n)})^{-1/2}$ can be made arbitrarily close to $\Asb(\lambda)^{-1/2}$. If $|\argmax_{x \in \X} \frac{(\bar{x}-x)^\top \Asb(\lambda)^{-1/2} \eta}{\beta+ \bar{\theta}^\t(\bar{x}-x)}| =1$, this implies that:
$$ \frac{(\bar{x}-x_\eta)^\top \Asb(\lambda)^{-1/2} \eta}{\beta+ \bar{\theta}^\t(\bar{x}-x_\eta)} \geq \frac{(\bar{x}-x')^\top \Asb(\lambda)^{-1/2} \eta}{\beta+ \bar{\theta}^\t(\bar{x}-x')} + \epsilon_\eta $$
for some $\epsilon_\eta > 0$, $x_\eta$ the unique value the argmax is attained at, and $x' \neq x_\eta$. As we can make $\Asb(\lambda^{(n)})^{-1/2}$ arbitrarily close to $\Asb(\lambda)^{-1/2}$, it follows that for large enough $n$, we can guarantee:
$$ \frac{(\bar{x}-x_\eta)^\top \Asb(\lambda^{(n)})^{-1/2} \eta}{\beta+ \bar{\theta}^\t(\bar{x}-x_\eta)} \geq \frac{(\bar{x}-x')^\top \Asb(\lambda^{(n)})^{-1/2} \eta}{\beta+ \bar{\theta}^\t(\bar{x}-x')} + \epsilon_\eta/2 $$
so the maximizer will be unique. As this is true for all $\eta \in B_\lambda$, it follows that $\lim_{n \rightarrow \infty} B_{\lambda^{(n)}} \subseteq B_\lambda$. An identical argument implies $B_\lambda \subseteq \lim_{n \rightarrow \infty} B_{\lambda^{(n)}}  $, so $\lim_{n \rightarrow \infty} B_{\lambda^{(n)}}  = B_\lambda$. Then, by the dominated convergence theorem,
\begin{align*}
\lim_{n \longrightarrow \infty	} \E [\nabla \L(  \lambda^{(n)}; \eta )^\t e_i   \one\{B_{\lambda^{(n)}}\}  ] & =  \E [\lim_{n \longrightarrow \infty	} \nabla \L(  \lambda^{(n)}; \eta )^\t e_i   \one\{B_{\lambda^{(n)}}\}  ] \\
& = \E [ \nabla \L( \lambda; \eta )^\t e_i  \one\{B_\lambda\} ] 
\end{align*}
where in the last line we used the continuity of $\nabla \L( \lambda; \eta )^\t e_i  \one\{B_\lambda\}$ in $\lambda$ on $\trianglem$ for a fixed $\eta$. Thus, the partial derivatives are continuous, proving differentiability at every $\lambda \in \trianglem$.

\end{proof}

The following Lemma shows that $\L(\kappa_1, \kappa_2 ; \tau; \lambda)$ is twice-differentiable wrt $\lambda$.

\begin{lemma}
\label{lem:differentiable_2}
$\L(\lambda)$ is twice-differentiable at every $\lambda \in \trianglem$ and
\begin{align*}
\frac{\partial^2 \L(\lambda)}{\partial \lambda_i \partial \lambda_j} = \E[\frac{\partial^2 \L(\lambda;\eta)}{\partial \lambda_i \partial \lambda_j} \one\{B_\lambda \}]
\end{align*}
where 
\begin{align*}
B_\lambda = \{\eta :  |\argmax_{x \in \X} \frac{(\bar{x}-x)^\top \Asb(\lambda)^{-1/2} \eta}{\beta+ \bar{\theta}^\t(\bar{x}-x)}| =1 \}.
\end{align*}
\end{lemma}

\begin{proof}
\textbf{Step 0: Setup.} From Lemma \ref{lem:differentiable}, $\L(\lambda)$ is differentiable at every $\lambda \in {\trianglem}_\psi$. Therefore, it suffices to show that $\nabla \L(\lambda)$ is differentiable at every $\lambda \in {\trianglem}_\psi$. It suffices to show that the 2nd order partial derivatives exist and are continuous. For the sake of abbreviation, define $g(\lambda) := \frac{\partial \L( \lambda )}{\partial \lambda_j }$ and $g(\lambda;\eta) := \frac{\partial \L(\lambda ;\eta)}{\partial \lambda_j }$. Note that we have that
\begin{align}
\frac{\partial g(\lambda;\eta)}{\partial \lambda_i} \one\{B_\lambda\} & = \one\{B_\lambda\} \frac{3}{4} \frac{1}{(\beta+ \bar{\theta}^\t (\bar{x}-\tilde{x})} \sum_{k \in (\bar{x} \Delta \tilde{x}) \cap x_i \cap x_j} \frac{\eta_k}{(\sum_{l : k \in x_l} \lambda_l)^{5/2}} \\
& \text{     where } \tilde{x} = \argmax_{x \in \X} \frac{\sum_{i \in \bar{x} \Delta x} \frac{\eta_i}{\sum_{x^\prime : i \in x^\prime} \lambda_{x^\prime} }}{\beta+ \bar{\theta}^\t(\bar{x}-x)} .
. \label{eq:derivative_2}
\end{align}

To begin, we show that the 2nd order partial derivatives exist using a truncation argument. Let $\varphi > 0$. Fix $\lambda \in \trianglem$. Define
\begin{align*}
q(x;\eta) =  \frac{\sum_{i \in \bar{x} \Delta x} \frac{\eta_i}{\sum_{x^\prime : i \in x^\prime} \lambda_{x^\prime} }}{\beta+ \bar{\theta}^\t(\bar{x}-x)} 
\end{align*} 
Define
\begin{align*}
B_{\varphi}& = \{\eta : \tilde{x} = \argmax_{x \in \X} q(x;\eta), \, \forall x^\prime \neq \tilde{x} \quad \frac{q(x^\prime;\eta)}{\norm{\eta}_2} < \frac{q(\tilde{x};\eta)}{\norm{\eta}_2} - \varphi\}.
\end{align*}
Note that
\begin{align*}
\lim_{\varphi \longrightarrow 0} B_{\varphi} = B_\lambda.
\end{align*}

\textbf{Step 1.} First, we show that
\begin{align}
\lim_{h \rightarrow 0} \E \left [ (\frac{g(\lambda + h e_i, \eta) - g(\lambda,\eta)}{h}) \one\{B_{\varphi}\} \right ] & =  \E \left [ \frac{\partial g(\lambda, \eta)}{\partial \lambda_i}  \one\{B_{\varphi}\} \right ] \label{eq:twice_diff_1}
\end{align}

Define
\begin{align*}
V_x & = \frac{\bar{x}-x}{\beta + \bar{\theta}^\t (\bar{x}-x)}
\end{align*}
for $x \in \X$. Note that since for any fixed $x \in \X$, $\Asb(\lambda)^{-1/2}V_x$ is Lipschitz in $\lambda$ on $\trianglem$, there exists $L_{\psi}$ depending on $\psi, \beta, \bar{x}$ such that for all $x \in \X$
\begin{align*}
\norm{[\Asb(\lambda)^{-1/2} - \Asb(\lambda + h e_i)^{-1/2}] V_x}_2 \leq L_\psi h.
\end{align*}
Let $h_{min} = \frac{\varphi}{4L_\psi}$. Let $h \in [0,h_{min}]$. Let $\eta \in \R^d$ such that it satisfies $B_\varphi$ and let $\tilde{x} = \argmax_{x \in \X} q(x;\eta)$. Let $x \in \X \setminus \{\tilde{x}\}$. Then, 
\begin{align*}
\frac{\varphi}{4} + \frac{v_{\tilde{x}}^\t A(\lambda + he_i)^{-1/2} \eta }{\norm{\eta}_2} & \geq \frac{v_{\tilde{x}}^\t A(\lambda )^{-1/2} \eta }{\norm{\eta}_2} \\
& \geq \varphi + \frac{v_{x}^\t A(\lambda )^{-1/2} \eta }{\norm{\eta}_2} \\
& \geq \frac{3\varphi}{4} + \frac{v_{x}^\t A(\lambda + h e_i )^{-1/2} \eta }{\norm{\eta}_2} 
\end{align*}
which implies that $\tilde{x} = \argmax_{x \in \X} v_x^\t A(\lambda + he_i)^{-1/2} \eta $. Thus, on $B_{\varphi}$, for all $h \in [0,h_{min}] $, $ \argmax_{x \in \X} v_x^\t A(\lambda + he_i)^{-1/2} \eta $ is the same and hence $g(\lambda + h e_i, \eta) = V_{\tilde{x}} A(\lambda+h e_i)^{-1/2} \eta$ for all $h \in (0,h_{min})$ and is thus differentiable for all $h \in (0,h_{min})$. Thus, by the mean value theorem, we have that
\begin{align*}
(\frac{g(\lambda + h e_i, \eta) - g(\lambda,\eta)}{h}) \one\{B_{\varphi}\}  = \frac{\partial g(\lambda+h^\prime e_i ;\eta)}{\partial \lambda_i} \one\{B_{\varphi}\}
\end{align*} 
for some $h^\prime \in (0,h]$. Inspection of \eqref{eq:derivative_2} shows that using $\lambda \in \trianglem$
\begin{align*}
\E[|\frac{\partial g(\lambda;\eta)}{\partial \lambda_i}|] < \infty .
\end{align*}
Thus, we may apply the dominating convergence theorem to obtain
\begin{align*}
\lim_{h \rightarrow 0} \E \left [ (\frac{g(\lambda + h e_i, \eta) - g(\lambda,\eta)}{h}) \one\{B_{\varphi}\} \right ] & = \E  \left [ \lim_{h \rightarrow 0}  (\frac{g(\lambda + h e_i, \eta) - g(\lambda,\eta)}{h}) \one\{B_{\varphi}\} \right ] \\ 
& = \E \left [ \frac{\partial g(\lambda, \eta)}{\partial \lambda_i}  \one\{B_{\varphi}\} \right ] 
\end{align*}

\textbf{Step 2.} Now, we show that
\begin{align}
\lim_{\varphi \longrightarrow 0} \E \left [ \frac{\partial g(\lambda, \eta)}{\partial \lambda_i}  \one\{B_{\varphi}\} \right ] = \E \left [ \frac{\partial g(\lambda,\eta)}{\partial \lambda_i}  \one\{B_\lambda\} \right ]. \label{eq:twice_diff_2}
\end{align}
Define
\begin{align*}
Z(\eta) & = \frac{3}{4} \frac{1}{(\beta+ \bar{\theta}^\t (\bar{x}-\tilde{x})} \sum_{k \in (\bar{x} \Delta \tilde{x}) \cap x_i \cap x_j} \frac{|\eta_k|}{(\sum_{l : k \in x_l} \lambda_l)^{5/2}} \text{     where } \tilde{x} = \argmax_{x \in \X} \frac{\sum_{i \in \bar{x} \Delta x} \frac{\eta_i}{\sum_{x^\prime : i \in x^\prime} \lambda_{x^\prime} }}{\beta+ \bar{\theta}^\t(\bar{x}-x)} .
\end{align*}
Note that for every $\varphi > 0$
\begin{align*}
|\frac{\partial g(\lambda,\eta)}{\partial \lambda_i}  \one\{B_{\varphi}\} | \leq Z(\eta)
\end{align*}
and $\E Z(\eta) < \infty$. Therefore, by the dominating convergence theorem,
\begin{align*}
\lim_{\varphi \longrightarrow 0} \E \left [ \frac{\partial g(\lambda,\eta)}{\partial \lambda_i}  \one\{B_{\varphi}\} \right ] =  \E \left [ \lim_{\varphi \longrightarrow 0} \frac{\partial g(\lambda,\eta)}{\partial \lambda_i}  \one\{B_{\varphi}\} \right ] = \E \left [ \frac{\partial g(\lambda,\eta)}{\partial \lambda_i}  \one\{B_\lambda\} \right ].
\end{align*}

\textbf{Step 3.} Now, we show that
\begin{align}
\lim_{\varphi \longrightarrow 0}  \lim_{h \rightarrow 0} \E \left [ (\frac{g(\lambda + h e_i, \eta) - g(\lambda,\eta)}{h}) \one\{B_{\varphi}\} \right ] = \lim_{h \rightarrow 0} \E \left [ (\frac{g(\lambda + h e_i, \eta) - g(\lambda,\eta)}{h}) \one\{B_\lambda\} \right ] \label{eq:twice_diff_3}
\end{align}
By step 1, for every $\varphi > 0$, 
\begin{align*}
\lim_{h \rightarrow 0} \E \left [ (\frac{g(\lambda + h e_i, \eta) - g(\lambda,\eta)}{h}) \one\{B_{\varphi}\} \right ] & =  \E \left [ \frac{\partial g(\lambda,\eta)}{\partial \lambda_i}  \one\{B_{\varphi}\} \right ]  \leq \E \left [ |\frac{\partial 
g(\lambda,\eta)}{\partial \lambda_i}  \one\{B_\lambda\}| \right ]  \leq C
\end{align*}
for some constant $C > 0$. Therefore, by the bounded convergence theorem for limits, we have that
\begin{align*}
\lim_{\varphi \longrightarrow 0}  \lim_{h \rightarrow 0} \E \left [ (\frac{g(\lambda + h e_i, \eta) - g(\lambda,\eta)}{h}) \one\{B_{\varphi}\} \right ] = \lim_{h \rightarrow 0} \lim_{\varphi \longrightarrow 0}  \E \left [ (\frac{g(\lambda + h e_i, \eta) - g(\lambda,\eta)}{h}) \one\{B_{\varphi}\} \right ]
\end{align*}
More formally, consider some sequence $\varphi_m, h_n$ such that $\varphi_m \rightarrow 0$ as $m \rightarrow \infty$ and $h_n \rightarrow 0$ as $n \rightarrow \infty$. Let $a_{mn} = \E \left [ (\frac{g(\lambda + h_n e_i, \eta) - g(\lambda,\eta)}{h_n}) \one\{B_{\varphi_m}\} \right ]$. If $\lim_{m \rightarrow \infty} \lim_{n \rightarrow \infty} a_{mn} = \lim_{n \rightarrow \infty} \lim_{m \rightarrow \infty} a_{mn}$ then the result is proven. Let $c_{mn} = a_{mn} - a_{m,n-1}$ and $c_{m0} = 0$. Note that for finite $m$, $c_{mn}$ is uniformly bounded for all $n$. Then the Bounded Convergence Theorem applied to the counting measure gives that:
$$ \lim_{m \rightarrow \infty} \sum_{n=0}^\infty c_{mn} = \sum_{n=0}^\infty \lim_{m \rightarrow \infty} c_{mn}$$
However, $\sum_{n=0}^\infty c_{mn} = \lim_{N \rightarrow \infty} \sum_{n=0}^N c_{mn}$, so the above implies:
$$ \lim_{m \rightarrow \infty}  \lim_{N \rightarrow \infty} \sum_{n=0}^N c_{mn} =  \lim_{N \rightarrow \infty}  \lim_{m \rightarrow \infty} \sum_{n=0}^N c_{mn}$$
By construction, we have $\sum_{n=0}^N c_{mn} = a_{mN}$, which proves the result.

Fix $h > 0$. Define
\begin{align*}
Y(h) = |\frac{g(\lambda + h e_i, \eta) - g(\lambda,\eta)}{h}) |\one\{B_\lambda\}.
\end{align*}
Note that for every $\varphi > 0$
\begin{align*}
|(\frac{g(\lambda + h e_i, \eta) - g(\lambda,\eta)}{h}) \one\{B_{\varphi}\}| \leq Y(h)
\end{align*}
and $E Y(h) < \infty$. Thus, by the dominating convergence theorem,
\begin{align*}
\lim_{\varphi \longrightarrow 0}  \E \left [ (\frac{g(\lambda + h e_i, \eta) - g(\lambda,\eta)}{h}) \one\{B_{\varphi}\} \right ] & =  \E \left [ \lim_{\varphi \longrightarrow 0} (\frac{g(\lambda + h e_i, \eta) - g(\lambda,\eta)}{h}) \one\{B_{\varphi}\} \right ] \\
& = \E \left [ (\frac{g(\lambda + h e_i, \eta) - g(\lambda,\eta)}{h}) \one\{B_\lambda\} \right ].
\end{align*}
This completes the step. 

\textbf{Step 4.} Putting together \eqref{eq:twice_diff_1}, \eqref{eq:twice_diff_2}, and \eqref{eq:twice_diff_3}, we have shown that
\begin{align*}
\lim_{h \longrightarrow 0} \E \left [ [(\frac{g(\lambda + h e_i, \eta) - g(\lambda,\eta)}{h}) \one\{B_\lambda\} \right ] & = \lim_{\varphi \longrightarrow 0} \lim_{h \longrightarrow 0} \E \left [ (\frac{g(\lambda + h e_i, \eta) - g(\lambda,\eta)}{h}) \one\{B_{\varphi}\} \right ] \\
 & =  \lim_{\varphi \longrightarrow 0} \E \left [ \frac{\partial g(\lambda, \eta)}{\partial \lambda_i}  \one\{B_{\varphi}\} \right ] \\
 & = \E \left [ \frac{\partial g(\lambda, \eta)}{\partial \lambda_i}  \one\{B_\lambda\} \right ]
\end{align*}

Thus, we have that that the second order partial derivatives exist and derived an expression for them. Showing that the second order partial derivatives are continuous proceeds as in the proof of Lemma \ref{lem:differentiable} (apply the dominating convergence theorem). 

\end{proof}

\section{Rounding}\label{sec:rounding}
\begin{theorem}[Caratheodory's Theorem]
For any point $y$ in the convex hull of a set $\mathcal{P} \subseteq \R^d$, $y$ can be written as a convex combination of at most $d+1$ points in $\mathcal{P}$. 
\end{theorem}
\begin{proof}
This is a standard result in convex geometry, see for instance \cite{eggleston1958convexity}.
\end{proof}

\begin{lemma}\label{lem:cara_bandit}
Given any $\lambda \in \triangle_{\X}$, in the bandit setting, there exists a distribution $\lambda' \in \triangle_{\X}$ that is $(d^2 + d + 1)$-sparse and:
$$ \Ab(\lambda) = \Ab(\lambda'), \quad \sum_{x\in \X} \lambda_x x = \sum_{x \in \X} \lambda_x' x$$
In the semi-bandit setting, when $\X \subseteq \{0,1\}^d$,  there exists a distribution $\lambda' \in \triangle_{\X}$ that is $(d + 1)$-sparse and:
$$ \Asb(\lambda) = \Asb(\lambda'), \quad \sum_{x\in \X} \lambda_x x = \sum_{x \in \X} \lambda_x' x$$
\end{lemma}
\begin{proof}
This is a direct corollary of Caratheodory's Theorem. Take $\lambda \in \triangle_{|\X|}$ and let $z_\lambda \in \R^{d + d^2}$, which we define as:
$$ z_\lambda = \left [\sum_{x \in \X} \lambda_x x; \text{vec} \left (\sum_{x \in \X} \lambda_x x x^\top \right ) \right ] $$
Define the set:
$$ \cV := \left \{ \left [ x; \text{vec}(x x ^\top) \right ] : x \in \X \right \} \subseteq \R^{d + d^2} $$
For any $\lambda$, we see that $z_\lambda$ lies in the convex hull of $\cV$. Caratheodory's Theorem then immediately implies the result in the bandit case, since $\text{vec} \left (\sum_{x \in \X} \lambda_x x x^\top \right )$ uniquely determines $\Ab(\lambda)$.

In the semi-bandit case, we note that the diagonal of $\Asb(\lambda)$ is equal to $\sum_{x \in \X} \lambda_x x$. Thus, we only need to consider a $d$-dimensional space, so Caratheodory implies we can find a $d+1$ sparse distribution.
\end{proof}

\begin{proof}[Proof of Lemma \ref{lem:rounding}]
Given some allocation $\tau$, let $\lambda$ the corresponding distribution, and $\bar{\tau} = \sum_{x \in \X} \tau_x$ (so $\tau = \bar{\tau} \lambda$). 

Since we only care about the sparsity of $\lambda$, consider $\bar{\tau}$ fixed. Then, given a solution $\lambda$ to (\ref{eq:vl_inefficient}) or (\ref{eq:vl}), the value of the constraint and objective the solution achieves achieves are fully specified by $\Af(\lambda)$ and $\sum_{x \in \X} \lambda_x x$. To see the latter, note that $\sum_{x \in \X} (\epsilon + \Delta_x) \lambda_x = \epsilon + \sum_{x \in \X} \theta^\top(\xst - x) \lambda_x = \epsilon + \theta^\top \xst + \theta^\top \sum_{x \in \X} \lambda_x x$. Lemma \ref{lem:cara_bandit} then implies that there exists a distribution $\lambda$ that is $(d^2 + d + 1)$-sparse in the bandit case and $(d+1)$-sparse in the semi-bandit case that achieves the same value of the constraint and objective of (\ref{eq:vl_inefficient}) or (\ref{eq:vl}).

To see the second part of the result, note that if we run the procedure of Theorem \ref{thm:comp_complex}, we will run stochastic Frank Wolfe for a polynomial number of steps, each increasing the support of our distribution by at most 1, so we will obtain an approximate solution that has at most $n = \poly(d,\Delmin,T,1/\delta)$ non-zero entries. By Theorem 6 in \cite{maalouf2019fast}, it then follows that we can compute the $(d+1)$-sparse distribution achieving the same value of the constraint and objective in time $\calO(nd)$. 
\end{proof}

\section{Gaussian Width Results}\label{sec:gw}
\newcommand{\diam}{\mathrm{diam}}
\newcommand{\adj}{\mathrm{adj}}

\begin{proposition}\label{prop:semi_kf}
\begin{equation}\label{eq:semi_kf}
\inf_{\lambda \in \triangle_\X} \max_{x \in \X} \| x \|_{\Asb(\lambda)^{-1}}^2 = d 
\end{equation}
\end{proposition}
\begin{proof}
This proof closely mirrors the proof of Theorem 21.1 of \cite{lattimore2020bandit}.

Let:
$$ f(\lambda) = \log \det \Asb(\lambda) $$
Noting that:
$$ \frac{d}{dt} \det (A(t))  = \tr \left ( \adj(A(t)) \frac{d}{dt} A(t) \right ) $$
and $A^{-1} = \adj(A)^\top / \det(A)$ \cite{lattimore2020bandit}, we can compute the gradient of $f(\lambda)$ as:
$$ \frac{d}{d \lambda_x} f(\lambda) = \frac{1}{\det \Asb(\lambda)} \tr \left ( \adj(\Asb(\lambda)) \diag(x x^\top) \right ) = \tr \left ( \Asb(\lambda)^{-1} \diag(x x^\top) \right ) $$
Since $\Asb(\lambda)$ is diagonal, we have:
$$ \tr \left ( \Asb(\lambda)^{-1} \diag(x x^\top) \right ) = \sum_{i=1}^d \frac{x_i^2}{[\Asb(\lambda)]_i} = x^\top \Asb(\lambda)^{-1}  x = \| x \|_{\Asb(\lambda)^{-1}}^2 $$
Note also that, by the identity above, for any $\lambda$:
\begin{align*}
\sum_{x \in \X} \lambda_x \| x \|_{\Asb(\lambda)^{-1}}^2 & =  \sum_{x \in \X} \lambda_x \tr \left (\Asb(\lambda)^{-1} \diag(x x^\top)  \right ) \\
& = \tr \left ( \Asb(\lambda)^{-1} \left ( \sum_{x \in \X} \lambda_x \diag(xx^\top) \right ) \right ) \\
& = \tr \left ( \Asb(\lambda)^{-1} \Asb(\lambda) \right ) \\
& = \tr (I) \\
& = d
\end{align*}
Then, since $\log \det X$ is concave and $\Asb(\lambda)$ is linear in $\lambda$, it follows that $f(\lambda)$ is concave. Applying standard first-order optimality conditions and denoting $\lambda^*$ the solution to (\ref{eq:semi_kf}), we have, for any $\lambda$:
\begin{align*}
0 & \geq \langle f(\lambda^*), \lambda - \lambda^* \rangle \\
& = \sum_{x \in \X} \lambda_x \| x \|_{\Asb(\lambda^*)^{-1}}^2 - \sum_{x \in \X} \lambda_x^* \| x \|_{\Asb(\lambda^*)^{-1}}^2 \\
& =\sum_{x \in \X} \lambda_x \| x \|_{\Asb(\lambda^*)^{-1}}^2 - d
\end{align*}
Choosing $\lambda$ to be the distribution putting all its mass on $x$, we have:
$$ d \geq \| x \|_{\Asb(\lambda^*)^{-1}}^2  $$
To see the equality, note that the above implies:
$$ d = \sum_{x \in \X} \lambda_x^* \| x \|_{\Asb(\lambda^*)^{-1}}^2 \leq \max_{x \in \X} \| x \|_{\Asb(\lambda^*)^{-1}}^2 \leq d $$
\end{proof}

\begin{proof}[Proof of Proposition \ref{prop:gw_topk_product}]
Let $S = \{x \in \X : \Delta_x \leq \epsilon \}$ for some fixed $\epsilon > 0$. Therefore, $x^* \in S$. Define
\begin{align*}
S_1 & = \{ (x,x^*_{m+1:n+m}): x \in \{0,1\}^m  \text{ s.t. there exists } x^\prime \in S \text{ s.t. } \Pi_{[m]} x^\prime = x \} \\
S_2 & = \{(x^*_{1:m}, x) : x \in \{0,1\}^{n} \text{ s.t. there exists } x^\prime \in S \text{ s.t. } \Pi_{[n+m]\setminus [m]} x^\prime = x \} 
\end{align*}
where $\Pi_A$ is the coordinate wise projection onto the coordinates $A \subset \N$. 
Then, using the fact that $\E[(x^*)^\t A(\lambda)^{-1/2} \eta ] = 0$, we have that
\begin{align*}
\min_{\lambda \in \triangle^{|S|}} \E[\sup_{x \in S} x^\t A(\lambda)^{-1/2} \eta]^2 & \leq \min_{\lambda \in \triangle^{|S|}} \E[\sup_{x_1 \in S_1} \sum_{i=1}^m x_{1,i} [A(\lambda)^{-1/2} \eta]_i+\sup_{x_2 \in S_2} \sum_{i=m+1}^{n+m} x_{2,i} [A(\lambda)^{-1/2} \eta]_i ]^2 \\
& = \min_{\lambda \in \triangle^{|S|}} \E[\sup_{x_1 \in S_1} \sum_{i=1}^m x_{1,i} [A(\lambda)^{-1/2} \eta]_i+\sup_{x_2 \in S_2} \sum_{i=m+1}^{n+m} x_{2,i} [A(\lambda)^{-1/2} \eta]_i  \\
& + \sum_{i=1}^{n+m} x^*_i [A(\lambda)^{-1/2} \eta]_i ]^2 \\
& = \min_{\lambda \in \triangle^{|S|}} \E[\sup_{x_1 \in S_1} x^\t_1 A(\lambda)^{-1/2} \eta+\sup_{x_2 \in S_2}  x_{2}^\t A(\lambda)^{-1/2} \eta ]^2 \\
& \leq  \min_{\lambda \in \triangle^{|S|}} c[\E[\sup_{x_1 \in S_1} x_1^\t A(\lambda) \eta ]^2+\E[\sup_{x_2 \in S_2} x_2^\t A(\lambda) \eta ]^2] \\
& \leq \min_{\lambda \in \triangle^{|S|}} c^\prime [k \log(m) \max_{x_1 \in S_1} \norm{x_1}_{A(\lambda)^{-1}}^2+\ell \log(n) \max_{x_2 \in S_2} \norm{x_2}_{A(\lambda)^{-1}}^2] \\
& \leq  c^{\prime \prime} [k \log(m) \min_{\lambda \in \triangle^{|S|}} \max_{x_1 \in S_1} \norm{x_1}_{A(\lambda)^{-1}}^2+ \ell \log(n) \min_{\lambda \in \triangle^{|S|}}  \max_{x_2 \in S_2} \norm{x_2}_{A(\lambda)^{-1}}^2] \\
\end{align*}
We begin by bounding the first term. Notice that $S_1 \subset S$ since $S = \{x \in \X : \Delta_x \leq \epsilon \}$ for some fixed $\epsilon > 0$ and thus if $x \in \{0,1\}^m \text{ s.t. there exists } x^\prime \in S \text{ s.t. } \Pi_{[m]} x^\prime = x$, then $(x,x^*_{m+1:n+m}) \in S$. Furthermore, the span of the vectors in $S_1$ has dimension at most $m+1$ since for any $x_1 \in S_1$, for all $i \geq m+1$, we have that
\begin{align*}
[x_1 - (\vec{0}_{1:m}, x^*_{m+1:n+m}) ]_i = 0 .
\end{align*}
Thus, by the Kiefer-Wolfowitz Theorem \cite{lattimore2020bandit}:
\begin{align*}
 \min_{\lambda \in \triangle^{|S|}} \max_{x_1 \in S_1} \norm{x_1}_{A(\lambda)^{-1}}^2 \leq m+1.
\end{align*}
and:
\begin{align*}
\min_{\lambda \in \triangle^{|S|}}  \max_{x_2 \in S_2} \norm{x_2}_{A(\lambda)^{-1}}^2 \leq n.
\end{align*}
Therefore, 
\begin{align*}
\min_{\lambda \in \triangle^{|S|}} \E[\sup_{x \in S} x^\t A(\lambda)^{-1/2} \eta]^2 \leq c [k \log(m) m+ \ell \log(n) n].
\end{align*}

To lower bound $|\X|$, note that:
$$ |\X| = \binom{m}{k} \binom{n}{\ell} \geq \left ( \frac{m}{k} \right )^k \left ( \frac{n}{\ell} \right )^\ell $$ 
For the second conclusion we set $\ell = \calO(1)$, $k = \sqrt{m}$, and $n = m^{3/2}$ and apply our regret bound. 

For the regret bound of competing algorithms, LinUCB will scale as $\tilO(d\sqrt{T}) = \tilO(m^{3/2} \sqrt{T})$. Given the above lower bound on $|\X|$, the regret of action elimination will scale as $\tilO(m\sqrt{T})$. In the semi-bandit setting, \cite{kveton2015tight} obtain a regret bound of $\tilO(m\sqrt{T})$ and, ignoring logarithmic terms, \cite{degenne2016combinatorial} obtain the same bound. Other existing works \citep{combes2015combinatorial,perrault2020statistical} do not state minimax bounds but, using the standard analysis to obtain a minimax bound from a gap-dependent bound, their regret will also scale as $\tilO(m\sqrt{T})$. Note that in this comparison we have ignored $\log(T)$ terms and have taken the dominate term to be the term with leading $m$ dependence that hits the $\sqrt{T}$. 
\end{proof}

\begin{proof}[Proof of Proposition \ref{prop:gw_bound_union}]
$ \Exp_\eta[\max_{x \in \X} x^\top A(\lambda)^{-1/2} \eta]$ is the Gaussian width of the set $\{ A(\lambda)^{-1/2} x \ : \ x \in \X \}$. By Proposition 7.5.2 of \cite{vershynin2018high}:
$$ \Exp_\eta[\max_{x \in \X} x^\top A(\lambda)^{-1/2} \eta] \leq c \sqrt{d} \diam(\{ A(\lambda)^{-1/2} x \ : \ x \in \X \}) $$
and:
$$ \diam(\{ A(\lambda)^{-1/2} x \ : \ x \in \X \}) = \max_{x_1, x_2 \in \X} \| A(\lambda)^{-1/2} (x_1 - x_2) \| \leq 2 \max_{x \in \X} \| x \|_{A(\lambda)^{-1}} $$
Taking the infimum over $\lambda \in \triangle_{\X}$, in the bandit feedback case Kiefer-Wolfowitz gives $\inf_{\lambda \in \triangle_{\X}}  \max_{x \in \X} \| x \|_{A(\lambda)^{-1}} \leq \sqrt{d}$, and in the semi-bandit case, Proposition \ref{prop:semi_kf} gives the same result. Since $\X$ was chosen arbitrarily, it follows that $\gamworst(\X) \leq d^2$.

For the second bound, Exercise 7.5.10 of \cite{vershynin2018high} gives that:
$$ \Exp_\eta[\max_{x \in \X} x^\top A(\lambda)^{-1/2} \eta] \leq c \sqrt{\log |\X|} \diam(\{ A(\lambda)^{-1/2} x \ : \ x \in \X \}) $$
from which the result follows immediately. 
\end{proof}

\begin{proof}[Proof of Proposition \ref{prop:gw_bound_comb}]
If $\X \subseteq \{0,1\}^d$ and $k = \max_{x \in \X} \| x \|_1$, then $\X$ at most contains all subsets of size $k$ and less so:
$$ |\X| \leq \sum_{j=1}^k \binom{d}{j} \leq c \sum_{j=1}^k (d/j)^j \leq c \sum_{j=1}^k d^j = c \frac{d(d^k - 1)}{d - 1} \leq c d^k$$
Thus, Proposition \ref{prop:gw_bound_union} gives:
$$ \gamma^* \leq c d k \log d $$
\end{proof}

\begin{proof}[Proof of Proposition \ref{prop:gw_topk_plus1}]
Consider the Top-$k$ problem in the semi-bandit feedback regime, but augment the action set by adding the vector of all 1s to it. In this case, then, we can either query a subset of size $k$, or we can query every point at once. Assume that $\theta_i \geq 0$ for all $i$. Note that by our assumption on $\theta_i$, $\mathbf{1}$ will always be in the action set regardless of how we are filtering on the gaps. If we put all our mass on $\mathbf{1}$, we will have that $\Asb(\lambda) = I$. Thus:
\begin{align*}
\gamworst(\Asb) & = \sup_{\epsilon > 0} \inf_{\lambda \in \triangle_{\X_\epsilon}} \mathbb{E}_\eta [ \sup_{x \in \X_\epsilon} x^\top A(\lambda)^{-1/2} \eta]^2 \\
& \leq \mathbb{E}_\eta [ \sup_{x \in \X_\epsilon} x^\top \eta]^2 \\
& \leq \E_{\eta}[\max_{x \in \X} |x^\top \eta|]^2 \\
& \leq c\left ( \E_{\eta}[\max_{x \in \X \backslash \mathbf{1} } |x^\top \eta|]^2 + \E_{\eta}[|\mathbf{1}^\top \eta|]^2 \right ) \\
& \leq  c\left ( \E_{\eta}[\max_{x \in \X \backslash \mathbf{1} } |x^\top \eta|]^2 + d \right ) \\
& \leq c \left ( k^2 \E_\eta[\max_{z : \| z \|_1 \leq 1} |z^\top \eta| ]^2 + d \right ) \\
& \leq c (k^2 \log d + d)
\end{align*}
where the last inequality follows since the gaussian complexity is within a constant of the Gaussian width when the set contains 0, by Exercise 7.6.9 of \cite{vershynin2018high}. The result then follows by choosing $k = \sqrt{d}$.
\end{proof}

\begin{theorem}[Tsirelson-Ibragimov-Sudakov Inequality \cite{cirel1976norms}]\label{thm:tis_full}
Let $\calS \subseteq \R^d$ be bounded. Let $(V_s)_{s \in \calS}$ be a Gaussian process such that $\Exp[V_s] = 0$ for all $s \in \calS$. Define $\sigma^2 = \sup_{s \in \calS} \Exp[V_s^2]$. Then, for all $u > 0$:
$$ \P [ | \sup_{s \in \calS} V_s - \Exp \sup_{x \in \calS}| \geq u] \leq 2 \exp \left ( \frac{-u^2}{2\sigma^2} \right ) $$
\end{theorem}
\begin{proof}[Proof of Proposition \ref{prop:tis}]
The proof in the bandit setting is identical to the proof given in \cite{katz2020empirical} and we therefore omit it.

In the semibandit setting, we have that:
$$ \hat{\theta}_i = \theta_i + \frac{1}{T_i} \sum_{t=1}^T x_{t,i} \eta_{t,i} $$
so $\Exp[\hat{\theta}_i] = \theta_i$ and:
\begin{align*}
\Exp[(\hat{\theta}_i - \theta_i)^2] = \frac{1}{T_i^2} \sum_{t=1}^T x_{t,i} = \frac{1}{T_i}
\end{align*}
Furthermore, since the noise is uncorrelated between coordinates, we have $\Exp[(\hat{\theta}_i - \theta_i)(\hat{\theta}_j - \theta_j)] = 0$. Since $x_t \in \{0,1\}^d$, it follows then that:
$$ \hat{\theta} \overset{\mathrm{distribution}}{=} \thetast + \wt{A}^{-1/2} \eta $$
for $\eta \sim \cN(0,I)$. Now consider the Gaussian process $V_x :=  x^\top (\hat{\theta} - \thetast) = x^\top \wt{A}^{-1/2} \eta$ for $x \in \X$. Noting that $\Exp[V_x^2] = x^\top \wt{A}^{-1} x \leq \max_{x \in \X} \| x \|_{\wt{A}^{-1}}^2$, we can then apply Theorem \ref{thm:tis_full} to this process, which gives the result. 
\end{proof}

\section{Lower Bound for Semi-Bandit Feedback and Optimistic Strategies}

A policy $\pi$ is \emph{consistent} if for all $\theta$ and $p > 0$, $R^{\pi}_\theta(T) = o(T^p)$.  Let $T_x$ denote the number of times that $x \in \X$ is pulled and $T_i$ the number of times that $i \in [d]$ is pulled. 

\begin{theorem}
\label{thm:lower_bound}
Let $\pi$ be a consistent policy such that $T_i \geq 1$ for all $i \in [d]$ with probability $1$, $\theta \in \R^d$ such that there is a unique optimal arm in $\X$. Let $G_T = \E[\sum_{t=1}^T \diag(x_t x_t^\t )]$ where $x_t$ is chosen at round $t \in [T]$. Then, 
\begin{align*}
\limsup_{T \longrightarrow \infty} \log(T) \norm{x}_{G_T^{-1}}^2 \leq \frac{\Delta_x^2}{2}
\end{align*}
for all $x \in \X$. Furthermore,
\begin{align*}
\limsup_{T \longrightarrow \infty} \frac{R^{\pi}_\theta(T)}{\log(T)} \geq c(\X, \theta)
\end{align*}
where 
\begin{align*}
c(\X, \theta) := \min_{ \tau \in [0,\infty)^{|\X|}}&  \sum_{x \in \X} \tau_x \Delta_x \\
& \text{s.t. } \sum_{i \in x} \frac{1}{\sum_{x^\prime : i \in x^\prime} \tau_{x^\prime} } \leq \frac{\Delta_x^2}{2} \quad \forall x \in \X \setminus \{x_* \}.
\end{align*}
\end{theorem}

\begin{proof}

We use a similar argument to the proof of Theorem 1 in \cite{lattimore2017end}. We construct an alternative instance $\theta^\prime$ to obtain an asymptotic lower bound. Let $P^\prime$ denote the probability measure of the associated instance (which we will specify shortly). We note that the Divergence Lemma (Lemma 15.1 \cite{lattimore2020bandit}) is easily adapted to the semi-bandit feedback setting. Thus, by a standard argument that applies the Divergence Lemma and the Bretagnolle–Huber inequality (Theorem 14.2 in \cite{lattimore2020bandit}), we have that
\begin{align}
\frac{1}{2} \norm{\theta - \theta^\prime}_{G_T}^2 \geq \log(\frac{1}{2 \P(E) + 2 \P^\prime(E^c)}) \label{eq:lb_div_pink}
\end{align}
for any event $E$. Define
\begin{align*}
\theta^\prime & = \theta + \frac{G_T^{-1} [x-x_*](\Delta_x + \epsilon)}{\norm{x-x_*}_{G_T^{-1}}^2}.
\end{align*}
Note that
\begin{align*}
(x-x_*)^\t \theta^\prime = \epsilon > 0.
\end{align*}
Let $R_T^\prime$ denote the regret of $\pi$ on the alternative instance $\theta^\prime$. Choose $E = \{T_{x_*} \leq \frac{T}{2} \}$. We have that
\begin{align*}
R_T = \sum_x \E[T_x] \Delta_x \geq \Delmin \frac{T}{2} \P(T_{x_*} \leq T/2).
\end{align*}
Furthermore, 
\begin{align*}
R^\prime_T =\sum_x \E[T_x] \Delta_x^\prime \geq \frac{\epsilon T}{2} \P^\prime(T_{x_*} \geq T/2).
\end{align*}
Thus, assuming that $\epsilon \leq \Delmin$, we have that
\begin{align}
\frac{R_T+R^\prime_T}{\epsilon T} \geq \P(E) +  \P^\prime(E^c). \label{eq:prob_small}
\end{align}
Then, inequalities \eqref{eq:lb_div_pink} and \eqref{eq:prob_small} imply that
\begin{align*}
\frac{(\Delta_x + \epsilon)^2}{2\norm{x-x_*}_{G_T^{-1}}^2 } \geq \log(\frac{\epsilon T }{2[R_T + R^\prime_T]}).
\end{align*}
Dividing both sides by $\log(T)$, we have that
\begin{align*}
\frac{(\Delta_x + \epsilon)^2}{2\norm{x-x_*}_{G_T^{-1}}^2 } \geq 1 - \frac{\log(1/2\epsilon)}{\log(T)} - \frac{\log(2 R_T - R_T^\prime)}{\log(T)}.
\end{align*}
Consistency of the policy $\pi$ implies that
\begin{align*}
\liminf_{T \longrightarrow \infty} \frac{(\Delta_x + \epsilon)^2}{2\norm{x-x_*}_{G_T^{-1}}^2 \log(T)} \geq 1.
\end{align*}
Rearranging, we have that 
\begin{align*}
\frac{(\Delta_x + \epsilon)^2}{2} \geq \limsup_{T \longrightarrow \infty} \norm{x-x_*}_{G_T^{-1}}^2 \log(T).
\end{align*}
This establishes the first claim in the lower bound. The second claim follows by a similar argument to the argument in Corollary 2 of \cite{lattimore2017end}.

\end{proof}

\begin{proof}[Proof of Proposition \ref{prop:opt_counterexample_semibandit}]
\textbf{Proof of lower bound for optimism:} Define the following problem instance
\begin{align*}
\theta_i & = \begin{cases}
1  & i = 1 \\
1-\epsilon & i \in \{2,\ldots, m\} \\
-1 + \epsilon & i \in \{m+1, \ldots, 2m -1 \} \\
-1 & i \in \{2m, \ldots, 2m + \sqrt{m}  \}
\end{cases}
\end{align*}
with $\X = \{\{1\}, \ldots, \{m\}, [2m+ \sqrt{m}] \}$. Let $x^{(i)} = \{i\}$ for $i \leq m$ and $x^{(m+1)} = [2m+ \sqrt{m}]$. Note that $\Delta_i = \epsilon$ if $i \leq m$ and $\Delta_{m+1} = \sqrt{m}+1$. Then, the optimization problem in Theorem \ref{thm:lower_bound} becomes
\begin{align*}
\min_{ \tau \in [0,\infty)^{|\X|}}&  \sum_{i \leq m} \tau_i \epsilon + \tau_{m+1} (\sqrt{m}+1)\\
 \text{s.t. }& \frac{1}{\tau_i + \tau_{m+1}} \leq \epsilon^2/2 \quad \forall i \in \{2,\ldots, m\}  \\
& \sum_{i\in [m]} \frac{1}{\tau_i + \tau_{m+1}} + \frac{m+\sqrt{m}}{\tau_{m+1}} \leq \frac{(\sqrt{m}+1)^2}{2}
\end{align*}
Consider the solution is $\tau_{m+1} = \frac{4}{\epsilon^2}$ and $\tau_i = 0$ otherwise. This attains a value of
\begin{align*}
O(\frac{\sqrt{m}}{\epsilon^2}).
\end{align*}

Now, consider the performance of the generic optimistic algorithm. Let $T_i$ denote the number of times that arm $i$ is chosen. Define the event
\begin{align*}
\mc{E} & = \{ |x^\t (\widehat{\theta}_t -\theta)| \leq \ucb(x,\{x_s\}_{s \in [t-1]})  \forall x \in \X, \, \,  \forall t \in [T] \}.
\end{align*}
Suppose $\mc{E}$ holds. Now, suppose that $T_{m+1} = 4 \alpha \log(T)$. Then,
\begin{align*}
[x^{(m+1)}]^\t \widehat{\theta}_t + \ucb(x^{(m+1)},\{x_s\}_{s=1}^{t-1}) & \leq [x^{(m+1)}]^\t \theta + 2 \ucb(x^{(m+1)},\{x_s\}_{s=1}^{t-1}) \\
& \leq -\sqrt{m} + 2\sqrt{\alpha \norm{x}_{(\sum_{s=1}^{t-1} x_s x_s^\t )^{-1} }^2 \log(T)} \\
&  \leq 0.
\end{align*}
On the other hand, on $\mc{E}$, we have that $[x^{(1)}]^\t (\widehat{\theta}_t + \ucb(x^{(1)},\{x_s\}_{s=1}^{t-1})) \geq 1$ and hence $x^{(m)}$ is pulled at $4 \alpha \log(T)$ times. Since $\P(\mc{E}^c) \leq \frac{1}{T}$, we have that 
\begin{align}
\E[T_{m+1} ] \leq 4 \alpha \log(T) + 1 \label{eq:pull_upper_bound}
\end{align} 

Recall that $G_T = \E[\sum_{t=1}^T \diag(x_t x_t^\t)]$. By Theorem \ref{thm:lower_bound}, we have that
\begin{align*}
\limsup_{T \longrightarrow \infty} \log(T) \norm{x^{(1)} - x^{(i)}}_{G_T^{-1}}^2 \leq \epsilon^2/2
\end{align*}
for all $i$, which together with \eqref{eq:pull_upper_bound} implies that
\begin{align*}
\E[T_i]/\log(T) = \Omega(1/\epsilon^2)
\end{align*}
for all $i \in \{2,\ldots, m\}$. Thus,
\begin{align*}
\limsup_{T \longrightarrow \infty} \frac{R^{optimistic}_\theta(T)}{\log(T)} & = \Omega(m/\epsilon ).
\end{align*}

\textbf{Proof of upper bound for Algorithm \ref{alg:gw_ae_comp}:} 
From the proof of Theorem \ref{thm:efficient_regret_bound}, we know that, for all $\ell$ simultaneously, with probability at least $1-\delta$: 
 \begin{align*}
 \calR_\ell \leq & \min_{\tau}  \ \sum_{x \in \X} 2(\epsilon_\ell + \hat{\Delta}_x) \tau_x  \\
& \text{ s.t. } \mathbb{E}_\eta \left [ \max_{x \in \X} \frac{(x_\ell - x)^\top \Asb(\tau)^{-1/2} \eta}{\epsilon_\ell + \hat{\Delta}_x} \right ] \leq \frac{1}{128 (1 + \sqrt{\pi \log(2\ell^3/\delta)})} 
\end{align*}  
and a $\tau$ satisfying:
$$ \mathbb{E}_\eta \left [ \max_{x \in \X} \frac{(x_\ell - x)^\top \Asb(\tau)^{-1/2} \eta}{\epsilon_\ell + \Delta_x} \right ] \leq \frac{1}{512 (1 + \sqrt{\pi \log(2\ell^3/\delta)})} $$
is also feasible for the problem above. Note that if we put all our mass on $\mathbf{1}$ we will have $\Asb(\tau) = \tau I$, so a feasible solution to the above problem requires that:
$$\left ( 512 (1 + \sqrt{\pi \log(2\ell^3/\delta)}) \mathbb{E}_\eta \left [ \max_{x \in \X} \frac{(x_\ell - x)^\top \eta}{\epsilon_\ell + \Delta_x} \right ] \right )^2 \leq \tau$$
we can upper bound:
\begin{align*}
\mathbb{E} \left [ \max_{x \in \X} \frac{(x_\ell - x)^\top \eta}{\epsilon_\ell + \Delta_x} \right ] & =  \mathbb{E} \left [ \max \left \{ \frac{x_\ell^\top \eta}{\epsilon_\ell + \epsilon} +  \frac{\max_{i=1,\ldots,m} -\eta_i}{\epsilon_\ell + \epsilon}, \frac{(x_\ell - \mathbf{1})^\top \eta}{\epsilon_\ell + \sqrt{m}} \right \} \right ] \\
& \leq \frac{1}{\epsilon_\ell + \epsilon} \Exp[|x_\ell^\top \eta|] + \frac{1}{\epsilon_\ell + \epsilon} \Exp[\max_{i=1,\ldots,m} | \eta_i|] + \frac{1}{\epsilon_\ell + \sqrt{m}} \Exp[|(x_\ell - \mathbf{1})^\top \eta|] 
\end{align*}
Since $x_\ell$ is a candidate for the best arm at round $\ell$, on the good event we must have that $\Delta_{x_\ell} \leq c \epsilon_\ell$. In particular, then, we will either have that $\| x_\ell \|_1 = 1$, or $\epsilon_\ell = O(\sqrt{m})$, so regardless of $\ell$, $\frac{1}{\epsilon_\ell + \epsilon} \Exp[|x_\ell^\top \eta|] \leq c/\epsilon_\ell$. By \cite{vershynin2018high}, since each $\eta_i$ has unit variance, we'll have $\Exp[\max_{i=1,\ldots,m} | \eta_i|] \leq c \sqrt{\log(m)}$. Finally, noting that $x_\ell - \mathbf{1}$ has at most $c(m + \sqrt{m})$ non-zero entries, $(x_\ell - \mathbf{1})^\top \eta$ has variance bounded as $c(m + \sqrt{m})$, so $\Exp[|(x_\ell - \mathbf{1})^\top \eta|] \leq \calO(\sqrt{m})$. We conclude that:
$$ \mathbb{E} \left [ \max_{x \in \X} \frac{(x_\ell - x)^\top \eta}{\epsilon_\ell + \Delta_x} \right ] \leq \calO \left (\frac{\sqrt{\log m}}{\epsilon_\ell} \right ) $$
It follows that:
$$ \tau \geq \calO \left ( \frac{\log(\ell^3/\delta) \log m}{\epsilon_\ell^2} \right ) $$
is sufficient. Since this is a feasible solution, we'll then have that:
$$ \calR_\ell \le \sum_{x\in\X} 2(\epsilon_\ell + \hat{\Delta}_x) \tau_{\ell,x}^* \leq  \calO \left ( (\epsilon_\ell + \sqrt{m} ) \frac{\log(\ell^3/\delta) \log m}{\epsilon_\ell^2} \right ) \leq  \calO \left ( \sqrt{m} \frac{\log(\ell^3/\delta) \log m}{\epsilon_\ell^2} \right ) $$
where the last inequality holds since $\sqrt{m} = \Delmax$. Ignoring $\log$ factors that do not involve $\delta$, and noting that there are at most $\log(\sqrt{m}/\epsilon)$ rounds, the total regret is bounded as:
$$\calO \left (  \sum_{\ell=1}^{\log(\sqrt{m}/\epsilon)}   \frac{\sqrt{m} \log(1/\delta)}{\epsilon_\ell^2} \right ) \leq \calO \left (   \frac{\sqrt{m} \log(1/\delta)}{m} 4^{\log(\sqrt{m}/\epsilon)} \right ) = \calO \left ( \frac{\sqrt{m} \log(1/\delta)}{\epsilon^2} \right )$$
Choosing $\delta = 1/T$ completes the proof.

\end{proof}

\textbf{Failure of Thompson Sampling for semi-bandit feedback:} We now provide a sketch as to why Thompson sampling fails on the instance in Proposition \ref{prop:opt_counterexample_semibandit}. Intuitively, Thompson Sampling is optimistic in a randomized fashion, so we would expect it to fail in the same way as optimistic algorithms. Slightly more formally, consider a typical version of Thompson sampling where at each round $t$, $\wt{\theta}_t \sim N(\widehat{\theta}_t, (\sum_{s=1}^{t-1} \diag(x_s x_s^\t))^{-1}))$ where $x_s$ is the arm chosen at time $s$ and $x_t = \argmax_{x \in \X} x^\t \wt{\theta}_t$. Note that with high probability, we will have that:
$$ | x^\top \wt{\theta}_t - x^\top \thetast  | \leq  \sqrt{\alpha \| x \|_{(\sum_{s=1}^{t-1} \diag(x_s x_s^\top))^{-1}}^2 \log(T)} $$
so we will essentially only pull an arm when$\sqrt{\alpha \| x \|_{(\sum_{s=1}^{t-1} \diag(x_s x_s^\top))^{-1}}^2 \log(T)} > \Delta_x$. In the case of $\mathbf{1}$, we will have:
$$ \| x \|_{(\sum_{s=1}^{t-1} \diag(x_s x_s^\top))^{-1}}^2 \approx \frac{\sqrt{m}}{T_{m+1}}$$
where $T_{m+1}$ are the total pulls of $\mathbf{1}$. Since $\Delta_{m+1} = \sqrt{m}$, the above inequality reduces to:
$$ \sqrt{\frac{\alpha \sqrt{m} \log(T)}{T_{m+1}}} > \sqrt{m} \implies \frac{\log(T)}{\sqrt{m}} > T_{m+1} $$
so arm $\mathbf{1}$ will only be pulled a logarithmic number of times in $T$, which, as with optimism, is not sufficient to achieve optimal regret.

\section{Additional Experimental Results}\label{sec:add_exp}
\begin{figure}[H]
     \centering
     \begin{subfigure}[b]{0.32\textwidth}
         \centering
        \includegraphics[width=\linewidth]{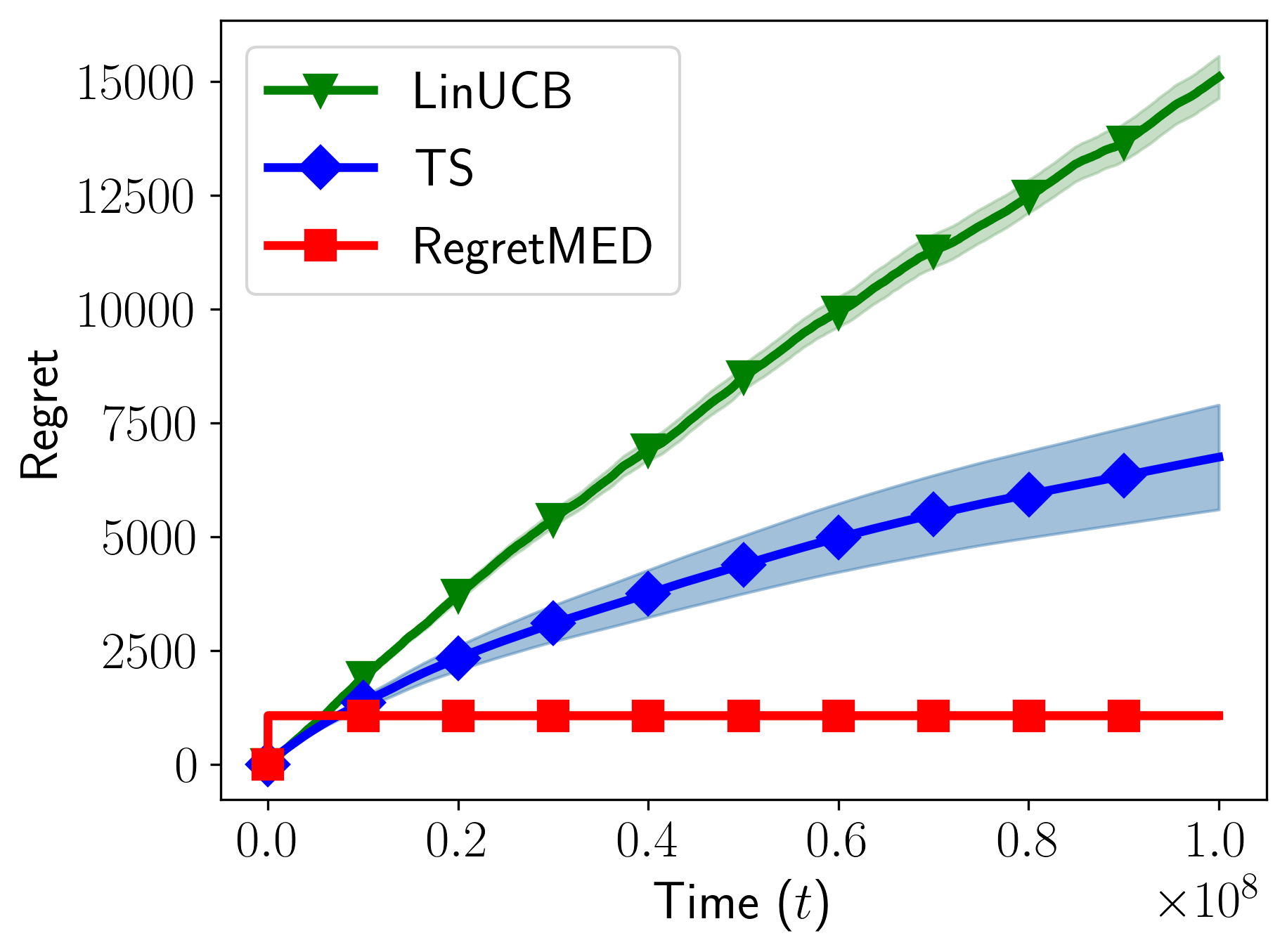}
  \caption{$\epsilon = 0.0005$}
     \end{subfigure}
     \hfill
     \begin{subfigure}[b]{0.32\textwidth}
         \centering
         \includegraphics[width=\linewidth]{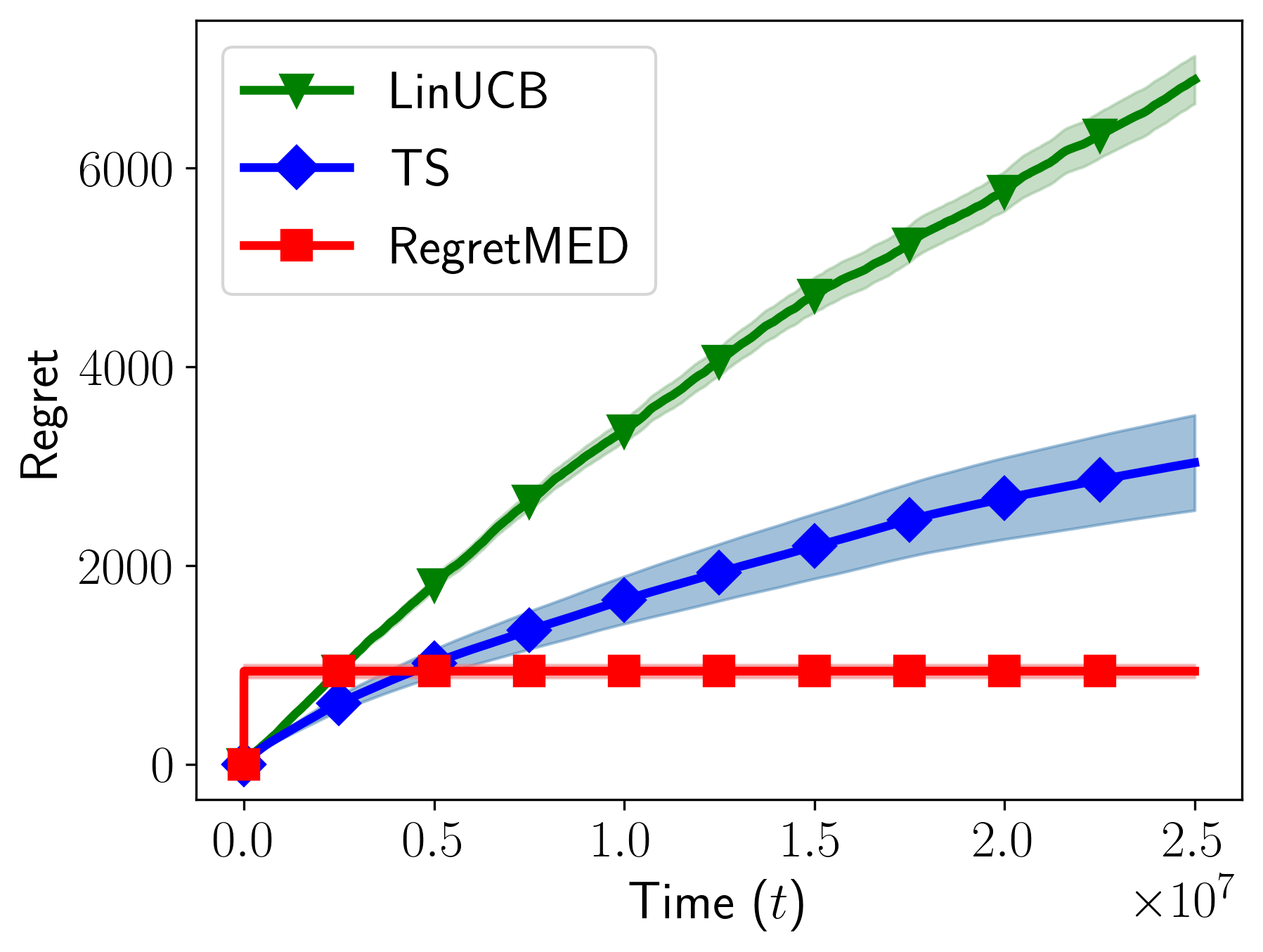}
  \caption{$\epsilon = 0.001$}
     \end{subfigure}
     \hfill
     \begin{subfigure}[b]{0.32\textwidth}
         \centering
           \includegraphics[width=\linewidth]{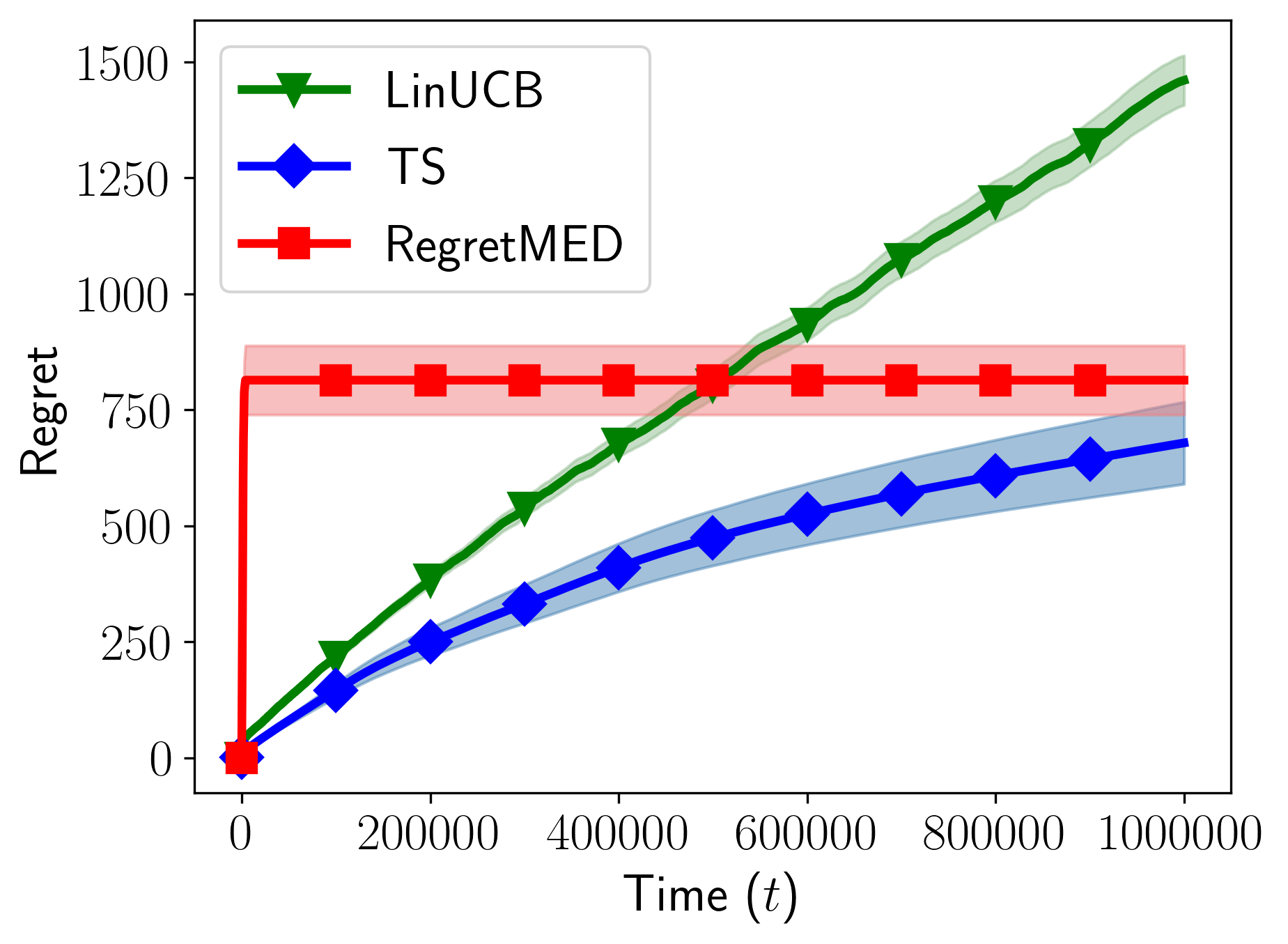}
  	\caption{$\epsilon = 0.005$}
     \end{subfigure}
        \caption{Regret against time plots for data points in Figure \ref{fig:soare}}
\end{figure}

\rev{We remark that, when running RegretMED, we do not use the exact constants specified in the algorithm. These constants are likely somewhat loose due to looseness in our analysis. In addition, we do not run the computationally efficient procedure derived formally but instead found that a much simpler heuristic---running stochastic Frank-Wolfe on the Lagrangian relaxation---works well in practice. We also do not use the precise value of $\Delmax$, and instead use an upper bound that can be computed using only knowledge of the arms.}

\rev{The algorithms we compare against do not contain significant hyperparameters, and we choose reasonable values for the parameters they do require. In particular, for LinUCB, we use the regularization $\lambda = 1$.}

\end{document}